\documentclass[onefignum,onetabnum]{siamonline220329}
\usepackage{lipsum}
\usepackage{algorithm}
\usepackage{algorithmic}
\usepackage{booktabs}
\usepackage{amsfonts}
\usepackage{setspace}
\usepackage{multirow}
\usepackage{amsmath}
\usepackage{subfigure}
\usepackage{graphicx}
\usepackage{epstopdf}
\usepackage{enumitem}
\newcommand{\tabincell}[2]{\begin{tabular}{@{}#1@{}}#2\end{tabular}}

\setlist[enumerate]{leftmargin=.5in}
\setlist[itemize]{leftmargin=.5in}

\newsiamremark{remark}{Remark}
\newsiamremark{hypothesis}{Hypothesis}
\crefname{hypothesis}{Hypothesis}{Hypotheses}
\newsiamthm{claim}{Claim}

\title{NeurTV: Total Variation on the Neural Domain
\thanks{This research is supported by the National Key R\&D Program of China (2022YFA1004100), the NSFC (No. 124B2029, 12371456, 12171072, 62131005, 12226004, 62272375, 32125009), the Major Key Project of PCL under Grant PCL2024A06, the Sichuan Science and Technology Program (No. 2024NSFJQ0038, 2023ZYD0007, 2024NSFSC0038), and the National Key Research and Development Program of China (No. 2020YFA0714001).}}

\author{Yisi Luo\thanks{School of Mathematics and Statistics, Xi'an Jiaotong University, Xi'an, Shaanxi, China.}
\and Xile Zhao\thanks{Corresponding author. School of Mathematical Sciences, University of
Electronic Science and Technology of China, Chengdu, Sichuan, China.}
\and Kai Ye\thanks{School of Automation Science and Engineering, Faculty of Electronic and Information Engineering, Xi'an Jiaotong University, Xi'an, Shaanxi, China.}
\and Deyu Meng\thanks{Corresponding author. School of Mathematics and Statistics, Xi'an Jiaotong University,  Xi'an, Shaanxi, China; Pengcheng Laboratory, Shenzhen, China; Macao Institute of Systems Engineering, Macau University of Science and Technology, Taipa, Macao.}}

\usepackage{amsopn}

\begin{document}
	
	\maketitle
	
	\begin{abstract}
Recently, we have witnessed the success of total variation (TV) for many imaging applications. However, traditional TV is defined on the original pixel domain, which limits its potential. In this work, we suggest a new TV regularization defined on the neural domain. Concretely, the discrete data is implicitly and continuously represented by a deep neural network (DNN), and we use the derivatives of DNN outputs w.r.t. input coordinates to capture local correlations of data. As compared with classical TV on the original domain, the proposed TV on the neural domain (termed NeurTV) enjoys the following advantages. First, NeurTV is free of discretization error induced by the discrete difference operator. Second, NeurTV is not limited to meshgrid but is suitable for both meshgrid and non-meshgrid data. Third, NeurTV can more exactly capture local correlations across data for any direction and any order of derivatives attributed to the implicit and continuous nature of neural domain. We theoretically reinterpret NeurTV under the variational approximation framework, which allows us to build the connection between NeurTV and classical TV and inspires us to develop variants (e.g., space-variant NeurTV). Extensive numerical experiments with meshgrid data (e.g., color and hyperspectral images) and non-meshgrid data (e.g., point clouds and spatial transcriptomics) showcase the effectiveness of the proposed methods.
\end{abstract}
\section{Introduction} Total variation (TV) regularization \cite{Rudin,SIAM_TV} has been a very successful method for a wide range of imaging applications \cite{SIAM_SVTV,SIAM_TTV,SIIMS_Tai,SIIMS_TV_2024}. The main aim of TV is to capture local correlations of data to reconstruct signals with degradations, e.g., signals with noise \cite{SIIMS_Elad} or missing pixels \cite{TIP_Inpainting}. {Here, the local correlations of data refer to the fact that objects that are close to each other are more likely to be similar or have some kind of spatial relationship compared to objects that are farther apart. This comes from the first law of geography introduced by Waldo R. Tobler in 1969 \cite{geo1}. For instance, the adjacent pixels of an image or the adjacent points of a point cloud tend to have similar color values, which naturally results in visually smooth local patterns. To illustrate such local correlations, we plot the histograms of the local differences of an image in Fig. \ref{fig_hist}. {Here, the local differences refer to the discrete gradients of the image by calculating the differences between adjacent pixel values}. It can be observed that the local differences are concentrated around zero, which validates the strong local correlations of the image. Such local correlations of data actually have been widely employed to alleviate the ill-posedness of many inverse imaging problems.} As symbolic work along this line, vast TV-based models have been proposed to address image processing problems, such as image denoising \cite{SIAM_MultiNoise}, deblurring \cite{TIP_09_TV}, deconvolution \cite{TIP_deconv}, unmixing \cite{TGRS_Unmixing}, retinex \cite{TV_Retinex}, etc. Among these models, some representative methods beyond the classical TV configuration were developed, such as the higher-order total generalized variation \cite{TGV}, directional total variation \cite{DTV}, weighted total variation \cite{PR_WTV,TGRS_S2S} and so on \cite{SIAM_higherorder,SIAM_Zhao_14,SIAM_TTV,SIAM_SVTV,SIIMS_TV_2024,TVTLS}. These methods successfully address some limitations of classical TV such as staircase effects \cite{NA_21}, loss of contrast \cite{SIAM_TV}, and isotropy property \cite{DTV}. Nowadays, TV has also been widely utilized in many science fields, e.g., gene analysis \cite{SR}, low-dose computed tomography \cite{SR_2}, and seismic inversion \cite{seismic}, serving as a fundamental tool for many tasks in real life.\par 
Classical TV regularizations are defined on the original pixel domain, which may limit their potential in two aspects. First, classical TV is limited to meshgrid data such as images and is not suitable for emerging non-meshgrid data, e.g., point clouds and spatial transcriptomics \cite{NMI_ST}, since it is difficult to define the discrete difference operator for non-meshgrid data. Second, the classical TV solely utilizes vertical and horizontal differences, which can not sufficiently capture local correlations of data existed in more comprehensive directions and orders of derivatives. A remedy to enhance the flexibility of TV is to consider improved discretization schemes, e.g., by using directional differences \cite{SPIC}, second-order differences \cite{TGV,NA_21}, and learning adaptive discretizations from paired datasets \cite{SIAM_discrete}. However, these improved discretization methods rely on carefully designed discrete approximation operators, and hence their characterizations for directional and high-order derivatives are generally quite complex.\par 
\begin{figure}[t]
	\vspace{-0.2cm}
		
	\scriptsize
	\setlength{\tabcolsep}{0.9pt}
	\begin{center}
		\includegraphics[width=0.70\textwidth]{dxdy.pdf}
		\vspace{-0.2cm}
	\end{center}
	\caption{The histograms of the local differences between adjacent pixels of the image ``Peppers''. \label{fig_hist}}
\vspace{-0.5cm}
\end{figure} 
In this work, we suggest a new TV regularization by using a deep neural network (DNN) to continuously represent data. {Here, the continuous representation refers to using a DNN to represent data by feeding the coordinate of data into the DNN and outputting the corresponding value, {allowing the network to continuously represent discrete data \cite{NeRF,CVPR_LIIF,sine}}. The continuous representation perspective allows us to readily deconstruct and reconstruct the classical TV. Concretely, we use the derivatives of DNN outputs w.r.t. input coordinates to capture local correlations of data, and propose the TV on the neural domain (termed NeurTV). {The proposed NeurTV encodes prior information of data (i.e., the local correlations of data) into the continuous representation model, which can help alleviate the overfitting of classical continuous representation methods (e.g., alleviate overfitting to noise; see for example Fig. \ref{fig_Barbara}) by virtue of the encoded prior information of data.} Besides, our NeurTV regularization can be readily combined with different DNN structures used in existing continuous representation methods; see examples in Sec. \ref{sec_neurtv}.}\par 
As compared with classical TV on the original pixel domain, the proposed NeurTV enjoys the following advantages (as intuitively shown in Fig. \ref{fig_NeurTV}). First, the NeurTV is free of discretization error induced by the discrete difference operator used in classical TV. Second, NeurTV is suitable for both meshgrid and non-meshgrid data attributed to the continuous representation. Third, NeurTV can be more readily extended to capture local correlations across any direction and any order of derivatives due to the differentiable nature of the DNN without designing additional discrete difference operators.\par
To justify the proposed NeurTV, we reinterpret NeurTV from the variational approximation perspective, which allows us to build connections between NeurTV and classical TV regularizations. Moreover, the variational approximation motivates us to develop variants of NeurTV such as space-variant NeurTV by designing spatially varying scale and directional parameters, which further improves the flexibility of NeurTV. Furthermore, NeurTV is flexibly integrated as a plug-and-play module, allowing a wide range of applications. In this work, we consider multiple applications including image denoising, inpainting, hyperspectral image mixed noise removal (on meshgrid), point cloud recovery, and spatial transcriptomics reconstruction (beyond meshgrid), which are validated through a series of numerical experiments. \par 
	\begin{figure}[t]
		\vspace{-0.2cm}
		\scriptsize
		\setlength{\tabcolsep}{0.9pt}
		\begin{center}
			\includegraphics[width=0.95\textwidth]{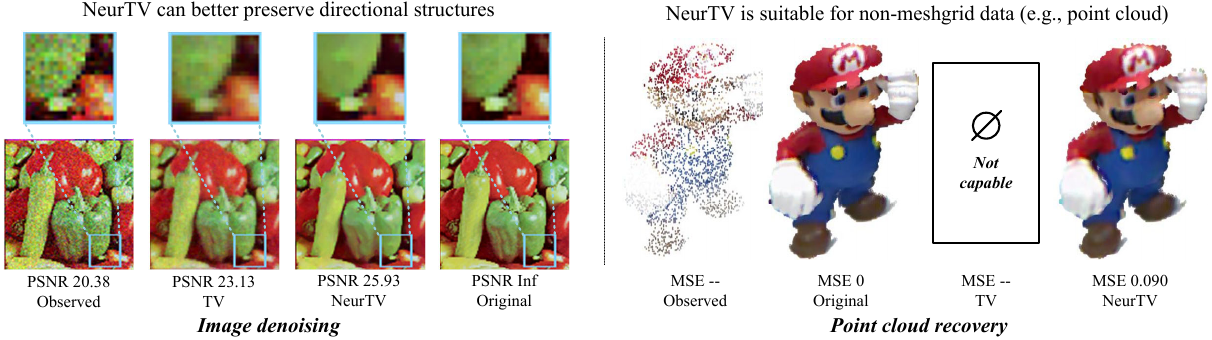}
			\vspace{-0.2cm}
		\end{center}
		\caption{The results of image denoising on ``Peppers'' and the results of point cloud recovery on ``Mario'' by using classical TV and the proposed NeurTV. NeurTV is suitable for both meshgrid and non-meshgrid data (e.g., point cloud), while TV is not suitable for non-meshgrid data. Moreover, NeurTV can better capture directional features by using the directional derivatives of DNN in the continuous domain.\label{fig_NeurTV}\vspace{-0.5cm}}
	\end{figure}
In summary, the contributions of this work are as follows.
\begin{itemize}
\item We propose the NeurTV regularization by using a DNN to continuously and implicitly represent data. Instead of using image differences on the pixel domain, we utilize the derivatives of DNN outputs w.r.t. input coordinates on the continuous neural domain to more comprehensively capture local correlations of data. Attributed to the continuous representation, NeurTV is suitable for both meshgrid and non-meshgrid data. Meanwhile, NeurTV can more exactly capture local correlations of data for any direction and any order of derivatives.
\item We theoretically reinterpret NeurTV from the variational approximation perspective, which allows us to draw connections between NeurTV and classical TV and also motivates us to develop more variants, e.g., arbitrary resolution NeurTV and space-variant NeurTV, which further improve the effectiveness and flexibility of NeurTV.
\item The suggested NeurTV regularization constitutes a basic building block that allows a wide range of applications. Specifically, we consider different applications including image denoising, inpainting, hyperspectral image mixed noise removal (on meshgrid), point cloud recovery, and spatial transcriptomics reconstruction (non-meshgrid). Extensive numerical experiments demonstrate the effectiveness of NeurTV methods. 
	\end{itemize}\par 
	The rest of the paper is organized as follows. Sec. \ref{sec_rela} introduces some related works. Sec. \ref{sec_method} presents the proposed NeurTV regularization and discusses its advantages as compared with classical TV. Moreover, we introduce some variants of NeurTV to more effectively capture local correlations of data. Sec. \ref{sec_exp} conducts numerical experiments using meshgrid and non-meshgrid data to show the effectiveness of NeurTV. Sec. \ref{sec_con} finally concludes the study.
	\section{Related Works}\label{sec_rela}
	\subsection{Interpolation and Rotation-based Methods} In the literature, some methods have attempted to enhance the flexibility of the TV from the perspective of interpolation or rotation \cite{DTV,SIIMS_TDV}. For example, Hosseini \cite{SPIC} considered interpolation to construct TV with four directions. Zhuang et al. \cite{SPL_derain,JCAM} proposed the learned gradient prior by interpolating over a rectangle space to define directional derivatives. The well-known directional TV \cite{DTV} can capture local directional smoothness of data by weighting the vertical and horizontal differences to create directional differences. Another type of directional TV is based on rotations. For instance, Jiang et al. \cite{FastDeRain} and Chang et al. \cite{ACMMM} rotated the image and then performed TV regularization to depict directional information. Zhuang et al. \cite{UConNet} rotated the feature maps of a deep neural network and then performed the TV to implicitly model directional derivatives. The main aim of these methods is to characterize directional derivative information by using elaborately designed interpolation or rotation techniques. As compared, our NeurTV can more accurately and concisely capture local correlations of data for arbitrary directions by using derivatives of DNN in the continuous space, which does not need to perform interpolation, rotation, or other operators to calculate directional derivatives.
	\subsection{Basis Function-based Methods} Another type of methods that related to this work are the basis function-based regularization methods \cite{sp,spl_smooth,TSP_CP,STD}. The pioneer works \cite{sp,spl_smooth} utilized basis functions (e.g., Gaussian basis \cite{sp}) to reveal the implicit smoothness between adjacent elements of the matrix. Other basis functions such as Fourier
	series \cite{TSP_CP} or Chebyshev expansions \cite{SIAM_cheb} were also utilized to represent matrices/tensors under function representations. The main aim of these basis function-based methods is to represent discrete meshgrid data via basis functions with adaptively learned coefficients to implicitly keep the local smoothness of the representation brought from the smoothness of basis functions. As compared with these methods, our NeurTV may hold stronger representation abilities brought from the DNN to characterize more complex structures of data. Also, NeurTV can be readily extended to capture local correlations for any direction and any order of derivatives. Such explicit directional and higher-order derivatives-based local correlations, however, are relatively hard to be excavated by basis function-based methods.
	{\subsection{Deep Learning-based Total Variation Methods} Recently, some deep learning-based TV regularizations have been proposed by leveraging the powerful representation abilities of deep neural networks to capture local features of data. For instance, Batard et al. \cite{DIP-VBTV} proposed a vector bundle total variation regularized deep image prior method for image denoising, which assumes that the restored image can be generated through a neural network in an unsupervised manner. Kobler et al. \cite{TPAMI_TDV} proposed a deep learning method for image denoising by introducing a data-driven total deep variation regularizer. In this regularizer, a convolutional neural network is used to extract local features after a supervised training process. Liu et al. \cite{JMIV_Liu} proposed a novel soft threshold dynamics framework which can integrate many spatial regularizations of the classical variational models (such as TV regularization) into the deep convolutional neural network for image segmentation. Grossmann et al. \cite{NIPS_Spectral} proposed a neural network approximation of a non-linear spectral total variation decomposition model for spectral decomposition of images. All the previous methods are working on the discrete meshgrid. Comparatively, the proposed NeurTV is designed based on the continuous representation using a deep neural network, which enjoys the following advantages. First, our NeurTV is free of discretization error induced by the discrete difference operator. Second, NeurTV is suitable for both meshgrid and non-meshgrid data, while existing deep learning-based methods are solely suitable for meshgrid data. Third, the proposed NeurTV can be more readily extended to capture local correlations for any direction and any order of derivatives due to the differentiable nature of the neural network.
\subsection{Graph and Mesh-based Methods} The graph and mesh-based regularization methods have been widely considered to handle unstructured representations beyond regular meshgrid. The flexibility of graph and mesh-based regularizations makes them suitable for unstructured data such as point clouds. For example, Zeng et al. \cite{TIP_point} proposed a patch-based graph Laplacian regularization for point cloud denoising. Dinesh et al. \cite{TIP_point_2} proposed a feature graph Laplacian regularization for point cloud denoising. Liu et al. \cite{mesh} proposed a mesh total generalized variation for 3D data denoising on triangular meshes. {Graph-based methods could handle both meshgrid data and non-meshgrid data (such as point clouds) effectively, and one can also define regularization on graphs (such as the graph Laplacian regularization). In this work, we propose an alternative regularization based on the continuous representation using deep neural network, which can also handle both meshgrid data and non-meshgrid data. It would be interesting to combine graph-based methods and the NeurTV regularization to further boost the performance, which we will investigate in future research.}}
\section{The Proposed Methods}\label{sec_method} In this section, we first introduce a general data recovery model by using the DNN to continuously represent data in Sec. \ref{sec_model_tv}. In Sec. \ref{sec_neurtv}, we introduce the proposed NeurTV and discuss its advantages. In Sec. \ref{sec_relation}, we reinterpret NeurTV from the variational approximation perspective, which allows us to draw connections between NeurTV and classical TV and also motivates us to develop variants of NeurTV.
	\subsection{Notations} We use $x,{\bf x},{\bf X},{\cal X}$ to denote scalar, vector, matrix, and tensor, respectively. The $i$-th element of a vector $\bf x$ is denoted by ${\bf x}_{(i)}$. The $(i_1,i_2,\cdots,i_N)$-th element of an $N$-th order tensor ${\cal X}$ is denoted by ${\cal X}_{(i_1,i_2,\cdots,i_N)}$. The notation ${\bf X}_{(i,:)}$ refers to the $i$-th row of the matrix $\bf X$, and it is similar for tensors. The discrete difference operator of a tensor ${\cal X}\in{\mathbb R}^{n_1\times n_2\times\cdots\times n_N}$ along its $d$-th dimension ($d=1,2,\cdots,N$) is denoted by ${\rm D}_d{\cal X} := {\cal X}_{(:,\cdots,1:n_d-1,\cdots,:)}-{\cal X}_{(:,\cdots,2:n_d,\cdots,:)}$, which returns a tensor of size ${n_1\times\cdots\times (n_d-1)\times\cdots\times n_N}$. The $\ell_1$-norm of a tensor $\cal X$ is denoted by $\lVert{\cal X}\rVert_{\ell_1}:=\sum_{i_1,i_2,\cdots,i_N}|{\cal X}_{(i_1,i_2,\cdots,i_N)}|$. Given a differentiable multivariate function $f({\bf x}):{\mathbb R}^N\rightarrow {\mathbb R}$, the $\frac{\partial f({\bf x})}{\partial {\bf x}_{(d)}}$ ($d=1,2,\cdots,N$) denotes the partial derivative of $f(\cdot)$ w.r.t. its $d$-th input variable ${\bf x}_{(d)}$. 
	\subsection{Data Recovery Model Based on Continuous Representation}\label{sec_model_tv}
	TV is a widely-used regularization for many imaging applications \cite{SIAM_TV,SIAM_SVTV,SIIMS_TV_2024}. The traditional TV-based data recovery model \cite{SIAM_TV,TV_Retinex,SIIMS_TV_2024} can be generally formulated as
	\begin{equation}\small\label{model_classical}
		\min_{{\bf X}\in{\mathbb R}^{n_1\times n_2}}\;{\rm Fidelity}({\bf O},{\bf X})+\lambda\;{\rm TV}({\bf X}),\;\;\;{\rm TV}({\bf X}):=\lVert{\rm D}_1{\bf X}\rVert_{\ell_1}+\lVert{\rm D}_2{\bf X}\rVert_{\ell_1},
	\end{equation}\noindent
where ${\bf O}\in{\mathbb R}^{n_1\times n_2}$ denotes the observed data. The TV is defined by performing discrete differences between adjacent pixels of $\bf X$. As aforementioned, classical TV is not suitable for non-meshgrid data such as point clouds, and may not be flexible and accurate enough to capture local correlations of data for any direction and any order of derivatives.\par 
	To address these shortcomings, we propose a new regularization method, termed NeurTV regularization, based on the continuous representation using DNN. Our NeurTV regularization is suitable for both meshgrid and non-meshgrid data, and can be extended to more comprehensively capture local correlations of data for any direction and any order of derivatives more accurately. We first introduce the general data recovery model based on the continuous representation. Specifically, the discrete data can be seen as a point set sampled on a multivariate function defined in a continuous domain. Taking the image data as an example and denoting ${\bf O}\in{\mathbb R}^{n_1\times n_2}$ as the observed image. We assume that the $(i,j)$-th element of $\bf O$ is the function value of a function on the coordinate $(i,j)$, i.e., there exists a function $f:\Omega\rightarrow{\mathbb R}$ where $\Omega \subset{\mathbb R}^2$ such that ${\bf O}_{(i,j)}=f(i,j)$ holds for any $i,j$ on meshgrid. Then we can rationally reconstruct the signal by optimizing the function $f(\cdot)$ instead of the discrete matrix. Achieve an effective continuous representation needs to select a function $f(\cdot)$ with strong representation abilities. Inspired by recent continuous representation methods in deep learning \cite{NeRF,CVPR_LIIF,sine,NIPS_PE}, we propose to use a DNN to parameterize the function $f(\cdot)$. A direct way to learn such continuous representation of the data $\bf O$ is to optimize the following objective with a DNN $f_\Theta:{\mathbb R}^2\rightarrow{\mathbb R}$ (where $\Theta$ denotes the DNN weights): 
	\begin{equation}\small\label{model_1}
		\min_{\Theta}\sum_{(i,j)\in{\rm meshgrid}}({\bf O}_{(i,j)}-f_\Theta(i,j))^2+\lambda\;\Psi(\Theta),
	\end{equation}\noindent
	where the first term is the fidelity loss to keep the consistency between the observed data $\bf O$ and the reconstructed signal, and $\Psi(\Theta)$ denotes a regularization term conditioned on DNN parameters $\Theta$. The $\lambda$ is a trade-off parameter. As compared with traditional TV model \eqref{model_1}, we optimize the underlying continuous representation $f_\Theta(\cdot)$ instead of the discrete signal $\bf X$. This allows us to deconstruct and reconstruct classical TV from the continuous representation perspective, offering potential advantages in terms of applicability and flexibility. 
	\par 
For images, the regularization $\Psi(\Theta)$ can be simply set as the traditional discrete difference-based TV on meshgrid:
	\begin{equation}\small
		\Psi(\Theta)=\sum_{(i,j)\in{\rm meshgrid}} |f_\Theta(i+1,j)-f_\Theta(i,j)|+|f_\Theta(i,j+1)-f_\Theta(i,j)|.
	\end{equation}\noindent
	However, for non-meshgrid data (e.g., point cloud), the observed points are not arranged on meshgrid, which makes discrete difference-based TV hard to apply. \par 
	To make the data recovery model \eqref{model_1} applicable for both meshgrid and non-meshgrid data, we reformulate it as follows. First, the observed $N$-dimensional data can be stored in an $n$-by-$(N+1)$ matrix ${\bf O}\in{\mathbb R}^{n\times(N+1)}$, where $n$ denotes the number of observed points and $N$ denotes the dimension number (e.g., $N=2$ for gray images). Each row ${\bf O}_{(i, :)}$ denotes the $i$-th observed point, where the first $N$ elements of ${\bf O}_{(i, :)}$ (i.e., ${\bf O}_{(i,1:N)}$) contain the coordinates and the last element of ${\bf O}_{(i, :)}$ (i.e., ${\bf O}_{(i,N+1)}$) is the value at this coordinate. Using this notation, both meshgrid and non-meshgrid data can be uniformly represented and \eqref{model_1} can be formulated in a more general form by using an $N$-dimensional DNN $f_\Theta:{\mathbb R}^N\rightarrow{\mathbb R}$:
	\begin{equation}\small\label{model_2}
		\min_{\Theta}\sum_{i=1}^n({\bf O}_{(i,N+1)}-f_\Theta({\bf O}_{(i,1:N)}))^2+\lambda\;\Psi(\Theta).
	\end{equation}\noindent
In this new continuous representation model \eqref{model_2}, the classical discrete difference-based TV can no longer be directly applied. Hence, it needs to develop a new regularization that can capture local correlations of both meshgrid and non-meshgrid data under model \eqref{model_2}\footnote{We remark that the data recovery models \eqref{model_1} \& \eqref{model_2} can be readily tackled by standard gradient descent method by optimizing the parameters of DNN, and we consistently employ the Adam optimizer \cite{Adam} to tackle the data recovery models presented in this paper.}.\par 
\subsection{Continuous Representation-based NeurTV}\label{sec_neurtv}
This section introduces the NeurTV regularization, which penalizes the derivatives of DNN outputs w.r.t. input coordinates to more sufficiently keep local smoothness of the continuous representation. Specifically, we introduce the following basic formulation of the NeurTV regularization.
{\begin{definition}[NeurTV]\label{def_NeurTV}
Given a function $f_\Theta({\bf x}):{\Omega\rightarrow{\mathbb R}}$ that is differentiable w.r.t. the input $\bf x$ and a point set $\Gamma\subset{\Omega}$, where $\Omega\subset{\mathbb R}^N$, the NeurTV regularization conditioned on $\Theta$ is defined as
		\begin{equation}\small\label{NeurTV}
			\Psi_{\rm NeurTV}(\Theta) := \sum_{d=1}^N\sum_{{\bf x}\in\Gamma}\left| \frac{\partial f_\Theta({\bf x})}{\partial {\bf x}_{(d)}}\right|,
		\end{equation}\noindent
which is based on the derivative of function output w.r.t. the input coordinate, i.e., $\frac{\partial f_\Theta({\bf x})}{\partial {\bf x}_{(d)}}$, where ${\bf x}_{(d)}$ denotes the $d$-th element of the input vector $\bf x$. If not otherwise specified, all the NeurTV regularizations appeared in this work refer to the above derivative-based NeurTV. \par 
On the other hand, we can also define the difference-based NeurTV. Let $N=2$ and $\Gamma$ be a meshgrid point set in the two-dimensional space (for simplicity we consider $\Gamma=\{(i,j)|i=0,1,\cdots,n_1,\;j=0,1,\cdots,n_2\}$), the difference-based NeurTV regularization is defined as 
		\begin{equation}\small\label{Diff-NeurTV}
			\Psi_{\rm Diff\mbox{-}NeurTV}(\Theta) := \sum_{i=0}^{n_1-1}\sum_{j=0}^{n_2}|f_\Theta(i+1,j)-f_\Theta(i,j)|+\sum_{j=0}^{n_2-1}\sum_{i=0}^{n_1}|f_\Theta(i,j+1)-f_\Theta(i,j)|,
		\end{equation}\noindent
which is based on the difference operators between adjacent points of the meshgrid.
	\end{definition}\par
We remark that the difference-based NeurTV has stronger requirements to be defined, i.e., it needs a meshgrid point set to define the difference-based regularization. Comparatively, the NeurTV based on derivatives can be easily applied to both meshgrid and non-meshgrid data. Hence, in the main experiments, we have consistently used the derivative-based NeurTV if not otherwise specified. The difference-based NeurTV is only used in the ablation study (i.e., Sec. \ref{sec_ablation}). Meanwhile, we remark that in Definition \ref{def_NeurTV}, only the local derivatives w.r.t. the input $\bf x$ are required to define the NeurTV, which is a mild assumption. In practice, we need to use the gradient descent-based algorithm to estimate the DNN parameters $\Theta$. Hence, it is necessary that the DNN $f_\Theta(\cdot)$ should be also differentiable w.r.t. the parameters $\Theta$ in practice. In all main experiments, we have used the sine activation function, which is differentiable w.r.t. both the input and the DNN parameters. To test the influence of different DNN structures (see Figs. \ref{fig_inr_img}-\ref{fig_inr_result}), we have used the ReLU activation function only in this ablation study. When using the ReLU activation function, we actually have used its subgradient for calculation, which is a standard protocol in deep learning.\par}
By plugging the derivative-based NeurTV into the continuous representation-based data recovery model \eqref{model_2}, we obtain the following model by giving the observed data $\bf O$:
	\begin{equation}\small\label{model_3}
		\min_{\Theta}\sum_{i=1}^n({\bf O}_{(i,N+1)}-f_\Theta({\bf O}_{(i,1:N)}))^2+\lambda\sum_{d=1}^N\sum_{{\bf x}\in\Gamma}\left|\frac{\partial f_\Theta({\bf x})}{\partial {\bf x}_{(d)}}\right|.
	\end{equation}\noindent\noindent
Here, the recovered data can be obtained by querying $f_\Theta(\cdot)$ on discrete coordinates. Our NeurTV-based data recovery model \eqref{model_3} preserves the local smoothness of the recovered data by penalizing the local derivatives of DNN outputs w.r.t. inputs. Besides, NeurTV has three advantages as compared with classical difference-based TV, as detailed below.\par 
First, NeurTV is free of discretization error induced by the discrete difference operator. Especially, the classical TV has discretization error when approximating the derivative using the difference operator. As compared, the proposed NeurTV directly utilizes the derivatives of DNN, which gets rid of the discretization process. Second, the NeurTV-based model \eqref{model_3} is not limited to meshgrid, but is suitable for both meshgrid and non-meshgrid data processing, e.g., image recovery and point cloud recovery. If the observed data $\bf O$ represents meshgrid data, then the sampling set $\Gamma$ could be set as meshgrid coordinates. While if the observed data $\bf O$ represents non-meshgrid data, then the sampling set $\Gamma$ could be set as non-meshgrid coordinates. Third, our NeurTV can be more readily and exactly extended to capture local correlations of data for any direction and any order of derivatives. Specifically, classical TV-based methods need to develop more complicated discrete difference \cite{TGV,SIIMS_TDV,SIAM_higherorder}, interpolation \cite{DTV,SPL_derain}, or rotation operators \cite{ACMMM,FastDeRain} to define directional or higher-order TV. As compared, NeurTV can directly make use of directional and higher-order derivatives of DNN outputs w.r.t. inputs to more easily and concisely design directional or higher-order derivatives-based NeurTV. We will introduce the corresponding extensions of our NeurTV to directional and higher-order NeurTV regularizations in Sec. \ref{sec_extend}. \par
{Here, we further discuss the historical revolution of TV family methods. Initially, the TV was defined over a continuous domain \cite{TV}, which benefits the solution of many variational problems such as finding minimal surfaces \cite{TV}. To develop more efficient algorithms, researchers have transferred the TV-constrained problem into discrete optimizations by using the difference operator to approximate the derivative \cite{SIAM_NA_1990,Rudin,JAMS_12}. In this work, we return to the continuous domain and leverage the powerful representation abilities of neural networks to {continuously represent data}, which allows us to deconstruct and reconstruct TV by using the function derivatives of DNN outputs w.r.t. inputs. The proposed NeurTV does not need numerical discretization and thus is free of the discretization error.}\par
{At the end of this section, we clarify the detailed structure of the DNN $f_\Theta(\cdot)$ used in this work. The proposed NeurTV regularization is a basic building block that can be combined with different DNN architectures to capture local correlations of data. Here, we take three representative DNN structures as examples, i.e., the sine function-based \cite{sine}, the positional encoding (PE)-based \cite{NIPS_PE}, and the tensor factorization (TF)-based \cite{TPAMI_Luo} DNNs, which have been shown to be effective for continuous representation of data. A general illustration of these DNNs is shown in Fig. \ref{fig_INR}. The detailed structures of these DNNs are as below.\par 
The sine function-based DNN \cite{sine} leverages periodic sine activation functions for continuous data representation. The sine function-based DNN is ideally suited for representing complex natural signals and their derivatives \cite{sine} due to its powerful representation abilities and differentiability. The detailed structure of the sine function-based DNN is formulated as
	\begin{equation}\small
		\begin{split}
			{\rm (Sine\;function\mbox{-}based\;DNN)\;}f_\Theta({\bf x}) = {\bf W}_{K}\sin({\bf W}_{K-1}\cdots\sin({\bf W}_1{\bf x})),
		\end{split}
	\end{equation}\noindent
where {${\Theta:=\{{\bf W}_k\}_{k=1}^K}$} denote the weight matrices of the fully connected layers. {Since the sine activation function is used, the sine function-based DNN $f_\Theta({\bf x})$ is a continuous function and is differentiable w.r.t. the input coordinates ${\bf x}$.} Hence we can directly deploy the NeurTV regularization by using the derivatives of DNN outputs $f_\Theta({\bf x})$ w.r.t. input coordinates.\par
The positional encoding-based DNN \cite{NIPS_PE} is a type of DNN that passes the input coordinates through a Fourier feature mapping, which can help the DNN better represent high-frequency details of the data. The detailed structure of the PE-based DNN is formulated as
	\begin{equation}\small
		\begin{split}
			{\rm (PE\mbox{-}based\;DNN)\;}f_\Theta({\bf x}) = {\bf W}_{K}{\rm ReLU}({\bf W}_{K-1}\cdots{\rm ReLU}({\bf W}_1\;{\rm PE}({\bf x}))),
		\end{split}
	\end{equation}\noindent
where $\{{\bf W}_k\}_{k=1}^K$ denote the weight matrices of the fully connected layers and ${\rm PE}({\bf x}):= [a_1\cos(2\pi {\bf b}_1^T{\bf x}),a_1\sin(2\pi{\bf b}_1^T{\bf x}),\cdots,a_m\cos(2\pi{\bf b}_m^T{\bf x}),a_m\sin(2\pi{\bf b}_m^T{\bf x})]^T$ denotes the positional encoding layer \cite{NIPS_PE}. Here, the ReLU is used as the activation function and hence the PE-based DNN outputs $f_\Theta({\bf x})$ is not differentiable w.r.t. the input coordinate $\bf x$. In this work, we have employed the subgradient of ReLU to calculate the derivatives of DNN outputs w.r.t. inputs for the PE-based DNN to deploy the proposed NeurTV.\par
	\begin{figure}[!t]
	\vspace{-0.2cm}
	\scriptsize
	\setlength{\tabcolsep}{0.9pt}
	\begin{center}
		\includegraphics[width=0.9\textwidth]{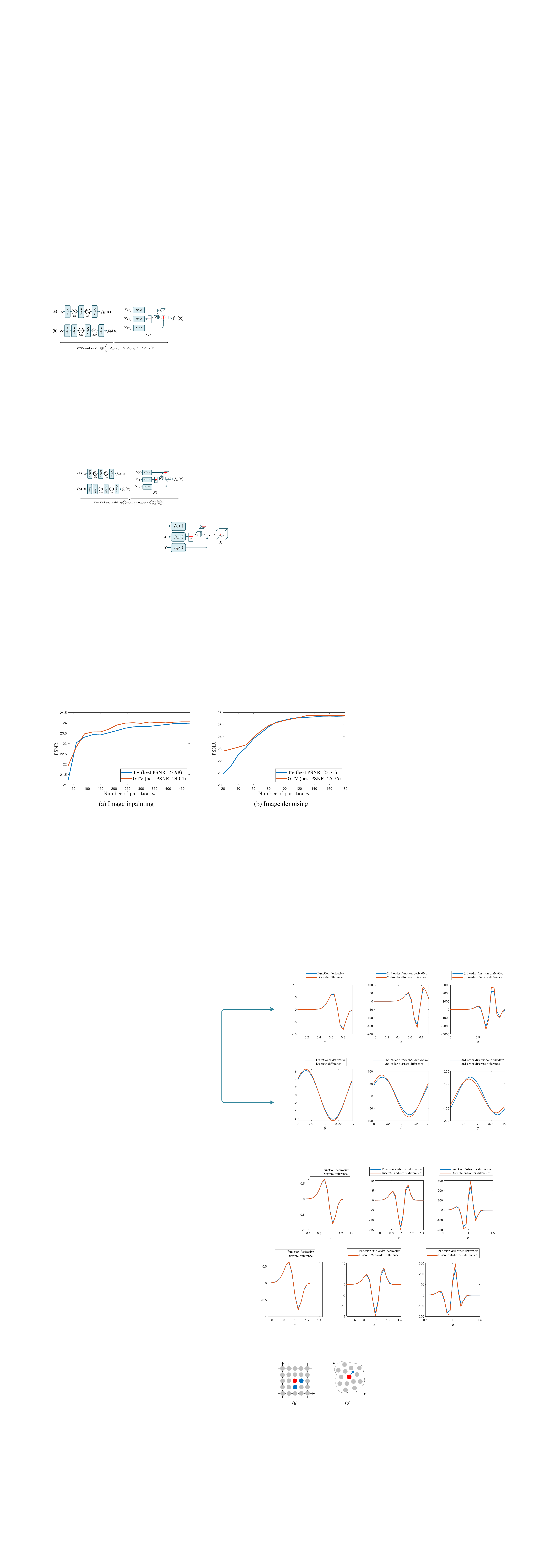}
		\vspace{-0.4cm}
	\end{center}
	\caption{Our NeurTV regularization is a basic building block that can be combined with different DNNs $f_\Theta(\cdot)$. For instance we consider (a) the sine function-based DNN \cite{sine}, (b) the positional encoding (PE)-based DNN \cite{NIPS_PE}, and (c) the tensor factorization (TF)-based DNN \cite{TPAMI_Luo} to test the proposed NeurTV method.\label{fig_INR}\vspace{-0.2cm}}
\end{figure}
The tensor factorization-based DNN \cite{TPAMI_Luo} is a type of DNN based on the tensor Tucker factorization \cite{1927}. It decomposes a continuous multivariant function into a core tensor and multiple factor functions, and then uses fully connected DNNs to parameterize the factor functions. The resulting tensor factorization-based DNN takes the coordinates as inputs and outputs the corresponding value, implicitly representing the data while encoding the low-rankness. The detailed structure of the tensor factorization-based DNN is formulated as
\begin{equation}\small
{\rm (Tensor\;factorization\mbox{-}based\;DNN)\;}f_\Theta({\bf x}) = {\cal C}\times_1f_{\Theta_1}({\bf x}_{(1)})\times_2\cdots\times_Nf_{\Theta_N}({\bf x}_{(N)}),
\end{equation}\noindent
		\begin{figure}[t]
	\scriptsize
	\setlength{\tabcolsep}{0.9pt}
	\begin{center}
		\begin{tabular}{cccc}
			\includegraphics[width=0.145\textwidth]{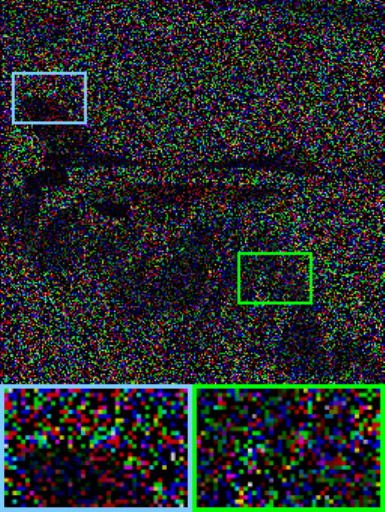}
			&\includegraphics[width=0.145\textwidth]{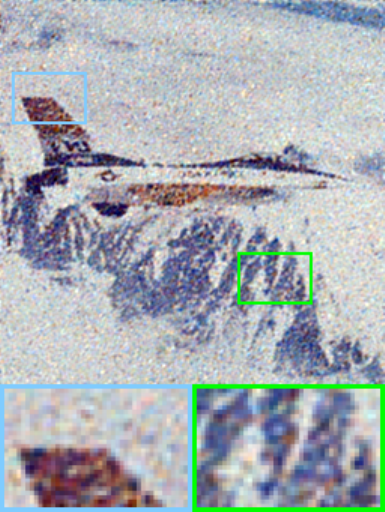}
			&\includegraphics[width=0.145\textwidth]{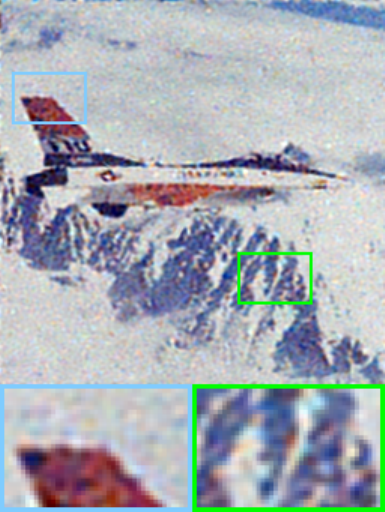}
			&\includegraphics[width=0.145\textwidth]{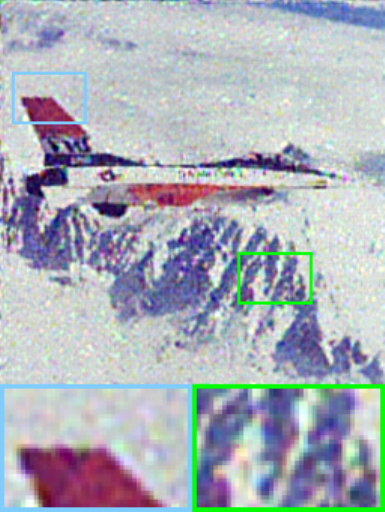}\\	
			PSNR 3.65 &
			PSNR 21.32 &
			PSNR 23.41  &
			PSNR 22.97  \\
			Observed&Sine&PE&TF\\
			\includegraphics[width=0.145\textwidth]{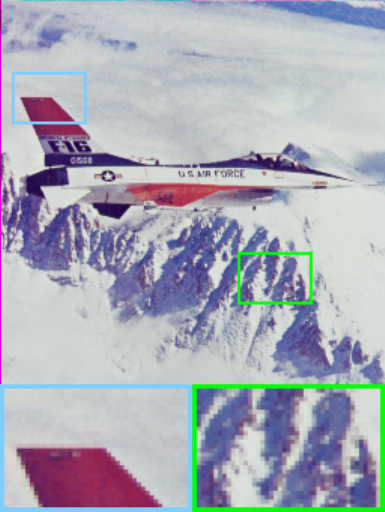}
			&\includegraphics[width=0.145\textwidth]{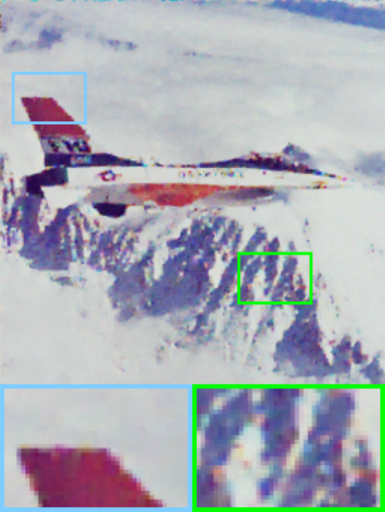}
			&\includegraphics[width=0.145\textwidth]{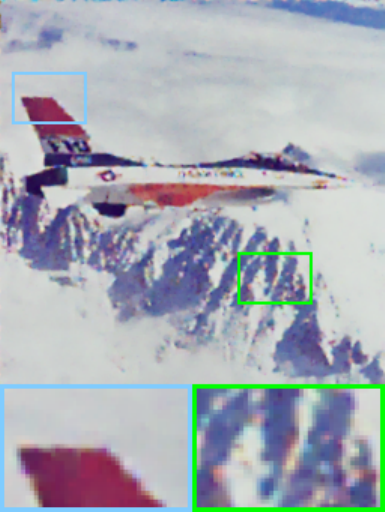}
			&\includegraphics[width=0.145\textwidth]{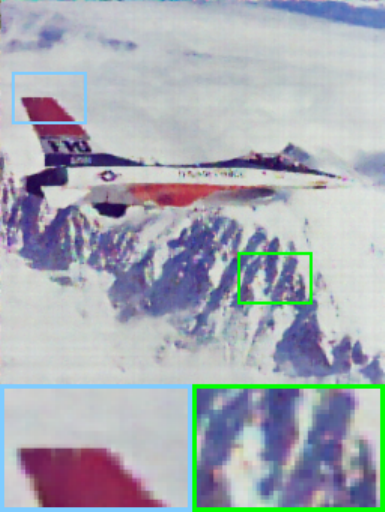}\\	
			PSNR Inf&
			PSNR 25.51&
			PSNR 25.79 &
			PSNR 26.13 \\
			\vspace{0.3cm}
			Original&Sine+${\rm NeurTV}$&PE+${\rm NeurTV}$&TF+${\rm NeurTV}$\\
		\end{tabular}
		\vspace{-0.4cm}
	\end{center}
	\caption{The results of image inpainting using different DNNs (sine function-based \cite{sine}, PE-based \cite{NIPS_PE}, and TF-based \cite{TPAMI_Luo} DNNs) with and without the NeurTV regularization \eqref{WGV}. \label{fig_inr_img}}
	\vspace{-0.2cm}
\end{figure} 
\begin{figure}[t]
	\scriptsize
	\setlength{\tabcolsep}{0.9pt}
	\begin{center}
		\begin{tabular}{cccc}
			\vspace{-0.4cm}
			\includegraphics[width=0.155\textwidth]{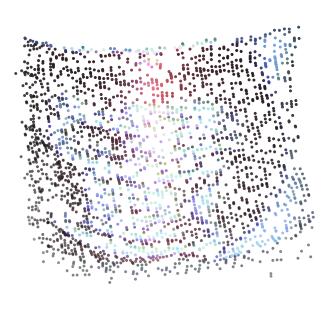}			&\includegraphics[width=0.155\textwidth]{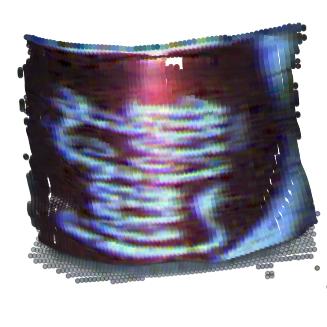}
			&\includegraphics[width=0.155\textwidth]{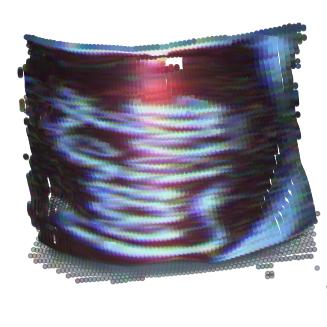}
			&\includegraphics[width=0.155\textwidth]{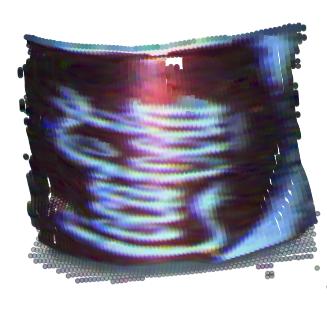}\\
			MSE \--\-- &MSE 0.173& MSE 0.173&MSE 0.150\\
			Observed&Sine&PE&TF\\
			\vspace{-0.4cm}
			\includegraphics[width=0.155\textwidth]{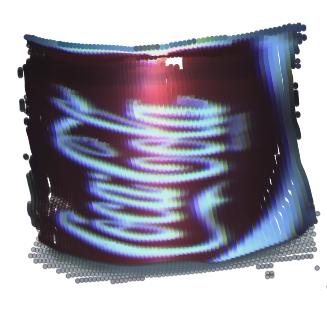}
			&\includegraphics[width=0.155\textwidth]{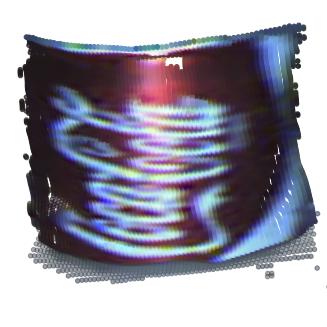}
			&\includegraphics[width=0.155\textwidth]{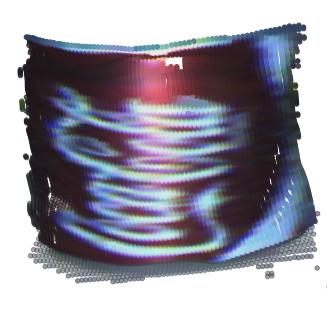}
			&\includegraphics[width=0.155\textwidth]{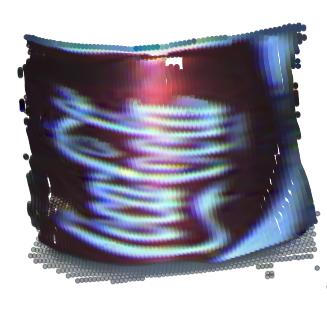}\\
			MSE 0 &MSE 0.136 & MSE 0.125&MSE 0.118\\
			\vspace{0.3cm}
			Original&Sine+${\rm NeurTV}$&PE+${\rm NeurTV}$&TF+${\rm NeurTV}$\\
		\end{tabular}
		\vspace{-0.4cm}
	\end{center}
	\caption{The results of point cloud recovery using different DNNs (sine function-based \cite{sine}, PE-based \cite{NIPS_PE}, and TF-based \cite{TPAMI_Luo} DNNs) with and without the NeurTV regularization \eqref{WGV}. \label{fig_inr_result}}
	\vspace{-0.2cm}
\end{figure}
where $\times_d$ ($d=1,2,\cdots,N$) denotes the mode-$d$ tensor-matrix product, ${\cal C}$ is the core tensor, and $\{f_{\Theta_d}\}_{d=1}^N$ are factor functions parameterized by fully connected DNNs, i.e., $f_{\Theta_d}({\bf x}_{(d)}) = {\bf W}_{d,K}\sin({\bf W}_{d,K-1}\cdots\sin({\bf W}_{d,1}{\bf x}_{(d)}))$ ($d=1,2,\cdots,N$), where {${\Theta}_d:=\{{\bf W}_{d,k}\}_{k=1}^K$} denote the weight matrices of the $d$-th factor DNN, and ${\Theta}:=\{{\cal C},{\Theta}_1,\cdots,{\Theta}_N\}$ are network parameters. {Since the sine activation function is used, the tensor factorization-based DNN $f_\Theta({\bf x})$ is a continuous function and is differentiable w.r.t. each input variable, i.e., ${\bf x}_{(1)},{\bf x}_{(2)},\cdots,{\bf x}_{(N)}$.} Hence we can directly deploy the NeurTV regularization for the tensor factorization-based DNN by taking the derivatives of DNN outputs $f_\Theta({\bf x})$ w.r.t. input coordinates.
\par 
Our NeurTV regularization can be easily combined with these DNNs by using the NeurTV-based data recovery model, e.g., model \eqref{model_3}. Here, we take image inpainting (meshgrid data) and point cloud recovery (non-meshgrid data) as examples by using the proposed NeurTV with different DNNs. The results are shown in Figs. \ref{fig_inr_img}-\ref{fig_inr_result}. It can be observed that the NeurTV regularization consistently improves the performances by using different DNN architectures, demonstrating its effectiveness and good compatibility. Since the tensor factorization-based DNN with NeurTV could attain better performances over the other two DNNs with NeurTV, we have adopted the tensor factorization-based DNN as the backbone to test the proposed NeurTV in all experiments conducted in this work.}
	\subsection{Extensions of NeurTV}\label{sec_extend}
	Our NeurTV can be easily extended to capture local correlations for any order of derivatives and any direction. In this section, we give some examples of the higher-order and directional derivatives-based NeurTV regularizations.
	\begin{figure}[t]
	\tiny
	\setlength{\tabcolsep}{0.95pt}
	\begin{center}
		\begin{tabular}{cccccc}
			\includegraphics[width=0.14\textwidth]{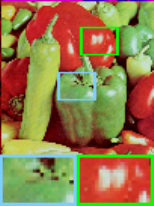}
			&\includegraphics[width=0.14\textwidth]{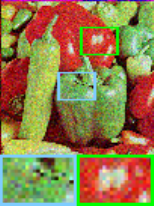}
			&\includegraphics[width=0.14\textwidth]{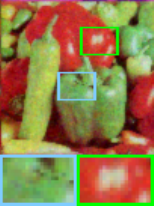}
			&\includegraphics[width=0.14\textwidth]{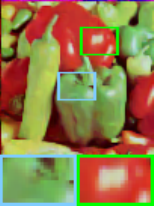}
			&\includegraphics[width=0.14\textwidth]{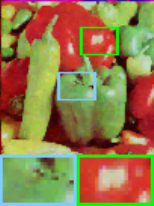}
			&\includegraphics[width=0.14\textwidth]{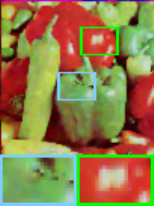}\\
			PSNR Inf &PSNR 20.38 &
			PSNR 23.13 &
			PSNR 25.47 &
			PSNR 25.43 &
			PSNR 25.66 \\
			Noisy&Original&TV&HDTV \cite{SIIMS_TDV}&NeurTV&2nd-order NeurTV\\
		\end{tabular}
		\vspace{-0.2cm}
	\end{center}
	\caption{The results of image denoising with Gaussian noise deviation 0.1 using the classical TV, the higher-order directional TV (HDTV), the NeurTV, and the proposed second-order NeurTV \eqref{2ndNeurTV}. \label{fig_htv}}
	\vspace{-0.2cm}
\end{figure}
	\subsubsection{Higher-order NeurTV}\label{sec_HNeurTV}
Our NeurTV can be readily extended to higher-order NeurTV regularization. Specifically, we can directly utilize higher-order derivatives of DNN outputs w.r.t. inputs, which avoids the discrete approximation process used in classical higher-order TV-based methods \cite{TGV,SIAM_higherorder,SIIMS_TDV,NA_21}. The higher-order NeurTV could ease the staircase effect \cite{TGV,NA_21} induced by first-order methods. As a basic example, we consider the following second-order NeurTV regularization by taking the second-order derivatives of DNN.
\begin{definition}
	For a two-dimensional function $f_\Theta({\bf x}):\Omega\rightarrow{\mathbb R}$ that is differentiable w.r.t. the input $\bf x$, where $\Omega\subset{\mathbb R}^{2}$, the second-order NeurTV of $f_\Theta(\cdot)$ is defined as
	\begin{equation}\small\label{2ndNeurTV}
		\Psi_{{\rm \tiny 2nd\mbox{-}NeurTV}}(\Theta) := \sum_{{\bf x}\in\Gamma} \left(\sum_{i=1}^2\left|\frac{\partial f_\Theta({\bf x})}{\partial{\bf x}_{(i)}}\right|+\kappa\sum_{i,j=1}^2\left|\frac{\partial^2 f_\Theta({\bf x})}{\partial{\bf x}_{(i)}\partial{\bf x}_{(j)}}\right|\right),
	\end{equation}\noindent\noindent
	where $\Gamma\subset{\Omega}$ is a point set and $\kappa$ is a trade-off constant.
\end{definition}\par
The proposed second-order NeurTV regularization uses a weighted hybrid of the first-order and the second-order derivatives of the DNN $f_\Theta(\cdot)$ to convey local correlations underlying data. The use of second-order derivatives is expected to alleviate the staircase effect. To validate the effectiveness of our second-order NeurTV regularization, we show the image denoising results by using the classical TV, the higher-order TV-based method \cite{SIIMS_TDV}, the NeurTV, and the second-order NeurTV regularizations in Fig. \ref{fig_htv}. It can be observed that the second-order NeurTV \eqref{2ndNeurTV} obtains a better result than NeurTV. Especially, the second-order NeurTV evidently alleviates the staircase effect of the first-order NeurTV method. 
	\subsubsection{Directional NeurTV}\label{sec_DNeurTV}
Analogously, NeurTV can be easily extended to directional NeurTV to capture directional local correlations. Traditional TV methods depend on image rotation \cite{FastDeRain} or interpolation \cite{SIIMS_TDV} to define the directional TV. As compared, our NeurTV can be more easily extended to directional NeurTV by utilizing the directional derivative of DNN outputs w.r.t. inputs. Here, we take the two-dimensional case as an example.
	\begin{lemma}\label{lemma_DTV}
		Given a function $f({\bf x}):\Omega\rightarrow{\mathbb R}$ that is differentiable w.r.t. the input $\bf x$, where $\Omega \subset{\mathbb R}^2$, the directional derivative along the direction ${\bf e}=(\cos\theta,\sin\theta)$ where $\theta\in[0,2\pi)$ is defined as
		\begin{equation}\small
			\nabla_{\bf e}f({\bf x}) := \lim_{t\rightarrow0^+}({f({\bf x}+(t\cos\theta,t\sin\theta))-f({\bf x})})\;{t^{-1}}.
		\end{equation}\noindent
		Such directional derivative of $f(\cdot)$ is equivalent to $\nabla_{\bf e}f({\bf x}) = \frac{\partial f({\bf x})}{\partial {\bf x}_{(1)}}\cos\theta+\frac{\partial f({\bf x})}{\partial {\bf x}_{(2)}}\sin\theta$.
	\end{lemma}
	\begin{figure}[t]
	\tiny
	\setlength{\tabcolsep}{0.95pt}
	\begin{center}
		\begin{tabular}{cccccc}
			\includegraphics[width=0.14\textwidth]{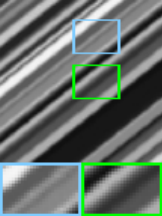}
			&\includegraphics[width=0.14\textwidth]{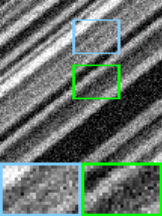}
			&\includegraphics[width=0.14\textwidth]{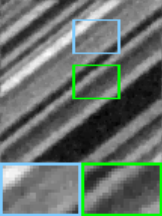}
			&\includegraphics[width=0.14\textwidth]{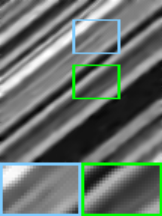}
			&\includegraphics[width=0.14\textwidth]{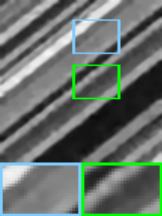}
			&\includegraphics[width=0.14\textwidth]{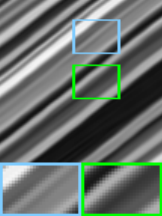}\\
			PSNR Inf&PSNR 20.04 &
			PSNR 25.11 &
			PSNR 30.75 &
			PSNR 27.39 &
			PSNR 31.89 \\
			Noisy&Original&TV&HDTV \cite{SIIMS_TDV}&NeurTV&NeurDTV\\
		\end{tabular}
		\vspace{-0.2cm}
	\end{center}
	\caption{The results of image denoising with Gaussian noise deviation 0.1 using the classical TV, the higher-order directional TV (HDTV) \cite{SIIMS_TDV}, the NeurTV, and the proposed directional NeurTV \eqref{DNeurTV} (NeurDTV) with $\theta = -\frac{3\pi}{10}$. \label{fig_dtv}}
	\vspace{-0.2cm}
\end{figure}
	\begin{proof}
Direct calculate yields
		\begin{equation}\small
			\begin{split}
				\nabla_{\bf e}f({\bf x}) &= \lim_{t\rightarrow0^+}({f({\bf x}+(t\cos\theta,t\sin\theta))-f({\bf x})})\;{t^{-1}}\\ 
				&=\lim_{t\rightarrow0^+}({f({\bf x}+(t\cos\theta,t\sin\theta))-f({\bf x}+(0,t\sin\theta))})\;{t^{-1}}+ \lim_{t\rightarrow0^+}({f({\bf x}+(0,t\sin\theta))-f({\bf x})})\;{t^{-1}}.
				\end{split}\end{equation}\noindent
By using the {intermediate value theorem}, we further have
\begin{equation}\small\begin{split}
\nabla_{\bf e}f({\bf x})=\lim_{t\rightarrow0^+}\left({\frac{\partial f({\bf\xi})}{\partial{\bf\xi}_{(1)}}\;t\cos\theta}\right)\;{t^{-1}}+\lim_{t\rightarrow0^+}\left({\frac{\partial f({\bf\eta})}{\partial{\bf\eta}_{(2)}}\;t\sin\theta}\right)\;{t^{-1}},
			\end{split}
		\end{equation}\noindent
where $\xi$ and $\eta$ satisfy $\lim_{t\rightarrow0^+}\xi = \lim_{t\rightarrow0^+}\eta = {\bf x}$. This then concludes that $\nabla_{\bf e}f({\bf x}) = \frac{\partial f({\bf x})}{\partial {\bf x}_{(1)}}\cos\theta+\frac{\partial f({\bf x})}{\partial {\bf x}_{(2)}}\sin\theta$.
	\end{proof}\par  
	\begin{definition}[Directional NeurTV]
		For a two-dimensional function $f_\Theta({\bf x}):\Omega\rightarrow{\mathbb R}$ that is differentiable w.r.t. the input $\bf x$, where $\Omega\subset{\mathbb R}^{2}$, the directional NeurTV (NeurDTV) of $f_\Theta(\cdot)$ along the direction $\theta\in[0,2\pi)$ is defined as
		\begin{equation}\small\label{DNeurTV}
			\Psi_{{\rm NeurDTV}_{\theta}}(\Theta) := \sum_{{\bf x}\in\Gamma}\left|\frac{\partial f_\Theta({\bf x})}{\partial{\bf x}_{(1)}}\cos\theta+\frac{\partial f_\Theta({\bf x})}{\partial{\bf x}_{(2)}}\sin\theta\right|,
		\end{equation}\noindent
where $\Gamma\subset{\Omega}$ is a discrete point set.
	\end{definition}\par
	According to Lemma \ref{lemma_DTV}, the directional function derivative can be analytically performed along any direction of the two-dimensional continuous space. We can then naturally introduce the directional function derivative of DNN into NeurTV as shown in \eqref{DNeurTV}. Such directional NeurTV can more concisely and effectively characterize directional structures of data. In Fig. \ref{fig_dtv}, we show the image denoising results by using the classical TV, the directional TV-based method \cite{SIIMS_TDV}, the NeurTV, and the proposed NeurDTV regularization, which shows the superiority of our NeurDTV for capturing directional local correlations of data.
	\begin{figure}[t]
		\tiny
		\setlength{\tabcolsep}{0.95pt}
		\begin{center}
			\begin{tabular}{cccccc}
				\includegraphics[width=0.14\textwidth]{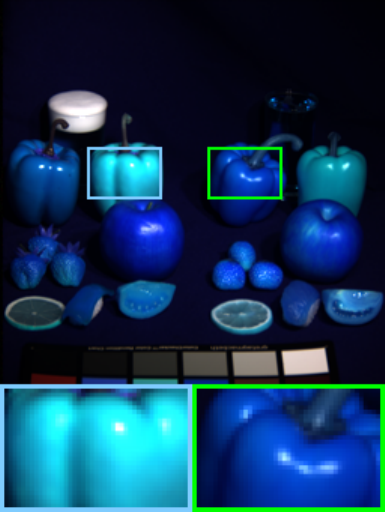}
				&\includegraphics[width=0.14\textwidth]{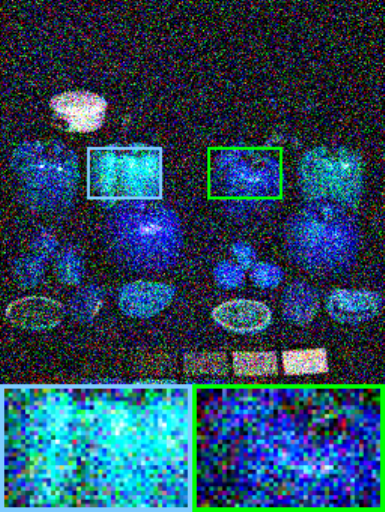}
				&\includegraphics[width=0.14\textwidth]{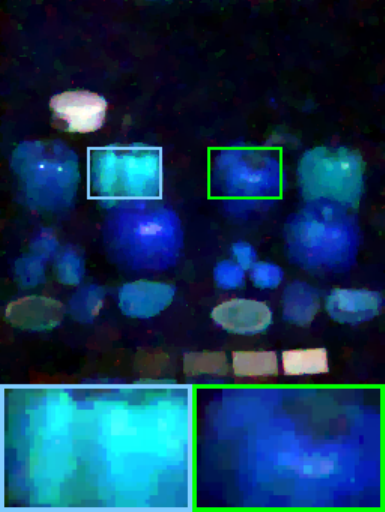}
				&\includegraphics[width=0.14\textwidth]{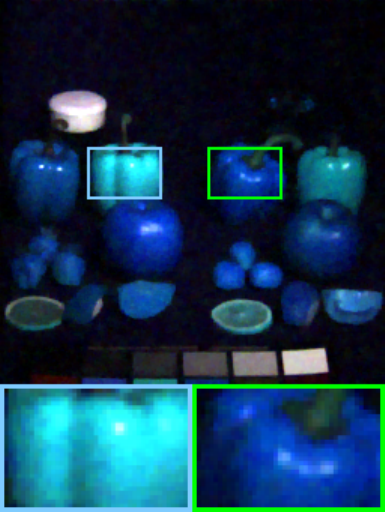}
				&\includegraphics[width=0.14\textwidth]{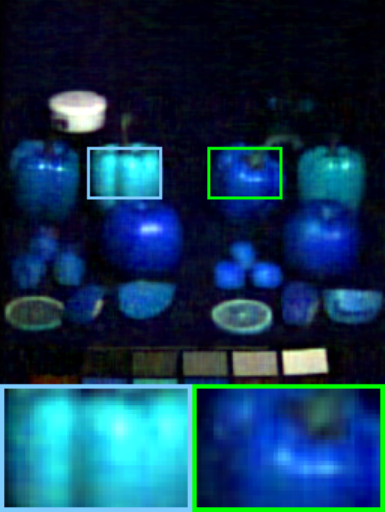}
				&\includegraphics[width=0.14\textwidth]{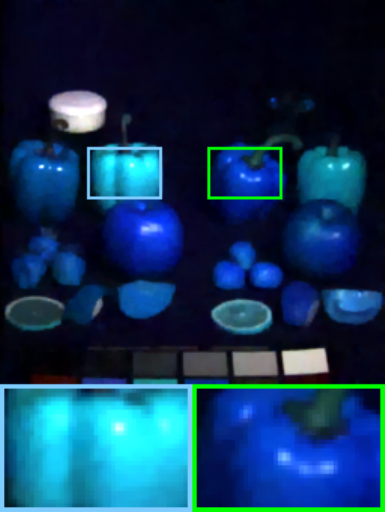}\\
				PSNR Inf&PSNR 21.89 &
				PSNR 35.43 &
				PSNR 38.77  &
				PSNR 37.83 &
				PSNR 41.18\\
				Original&Noisy&TV&LRTDTV \cite{LRTDTV}&NeurTV&NeurSSTV\\
			\end{tabular}
			\vspace{-0.2cm}
		\end{center}
		\caption{The results of hyperspectral image denoising with Gaussian noise deviation 0.2 using the classical TV, the SSTV-based method LRTDTV \cite{LRTDTV}, the NeurTV, and the proposed NeurSSTV \eqref{SSNeurTV}.  \label{HSI}}
		\vspace{-0.2cm}
	\end{figure}
	\subsubsection{Spatial-spectral NeurTV}
	Then, we introduce a new higher-order NeurTV regularization for multi-dimensional data. Specifically, we take the hyperspectral image (HSI) \cite{LRTDTV,LRTF-DFR} as an example. The HSI contains intrinsic spatial-spectral local smooth structures \cite{SSTV}. To represent the spatial-spectral local correlations of HSI, we propose the following second-order spatial-spectral NeurTV (termed NeurSSTV) regularization.
	\begin{definition}[Spatial-spectral NeurTV]
		Given a three-dimensional function $f_\Theta({\bf x}):\Omega\rightarrow{\mathbb R}$ (with the first two dimensions correspond to the spatial dimensions and the last dimension corresponds to the spectral dimension) that is differentiable w.r.t. the input $\bf x$ and a point set $\Gamma\subset{\Omega}$, where $\Omega\subset{\mathbb R}^{3}$, the NeurSSTV of $f_\Theta(\cdot)$ is defined as
		\begin{equation}\small\label{SSNeurTV}
			\Psi_{\rm \tiny NeurSSTV}(\Theta) := \sum_{{\bf x}\in\Gamma} \left(\left|\frac{\partial^2 f_\Theta({\bf x})}{\partial{\bf x}_{(1)}\partial{\bf x}_{(3)}}\right|+\left|\frac{\partial^2 f_\Theta({\bf x})}{\partial{\bf x}_{(2)}\partial{\bf x}_{(3)}}\right|\right),
		\end{equation}\noindent\noindent
		where $\frac{\partial}{\partial {\bf x}_{(1)}},\frac{\partial}{\partial {\bf x}_{(2)}}$ are two spatial partial derivative operators and $\frac{\partial}{\partial {\bf x}_{(3)}}$ is the partial derivative operator along the spectral dimension of the HSI.
	\end{definition}\par
	The proposed NeurSSTV is defined by viewing the HSI as a three-dimensional differentiable function $f_\Theta(\cdot)$. Different from the classical SSTV methods \cite{SSTV,LRTDTV}, our NeurSSTV directly penalizes the higher-order derivatives of DNN. A numerical validation of our NeurSSTV as compared with the classical SSTV-based method \cite{LRTDTV} is shown in Fig. \ref{HSI}. It is obvious to observe that NeurSSTV obtains a better recovered result over the traditional SSTV \cite{LRTDTV}, which validates the effectiveness of our NeurSSTV for multi-dimensional HSI recovery.
\subsection{Justification of NeurTV from Variational Approximation}\label{sec_relation}
To further justify our NeurTV, we reinterpret NeurTV from the variational approximation perspective, which allows us to draw connections between NeurTV and classical TV and motivates us to develop effective variants (e.g., space-variant NeurTV). {Specifically, based on the definition and theory of classical functional TV (i.e., Lemma \ref{lemma_TV_DF}) in the literature \cite{TV}, our contributions here are the following theoretical results, which are different from the contents proposed in previous work. First, we have theoretically presented the connections between NeurTV and classical TV (i.e., Sec. \ref{sec_connect}). Second, we have theoretically presented the variational approximation error analysis to reveal the rationality of the arbitrary resolution NeurTV (i.e., Sec. \ref{sec_ar}). Third, we have presented theoretical analysis that inspires us to develop space-variant NeurTV with second-order derivative-based parameter selection (i.e., Sec. \ref{sec_wtv}).}
\subsubsection{Connections Between NeurTV and Classical TV}\label{sec_connect}
First, we introduce the classical functional TV \cite{TV} in Lemma \ref{lemma_TV_DF}, which is a variational regularization that penalizes the derivatives of function.
{\begin{lemma}[Total variation of function \cite{TV}]\label{lemma_TV_DF}
Given a function $f_\Theta({\bf x}):\Omega\rightarrow{\mathbb R}$ parameterized by $\Theta$, where $\Omega\subset{\mathbb R}^N$, the TV of $f_\Theta(\cdot)$ in $\Omega$ is defined as
		\begin{equation}\small\label{FunTV}
			V(f_\Theta;\Omega) := \sup\{\int_\Omega f_\Theta({\bf x})\;{\rm div}\phi({\bf x}) d{\bf x}: \phi\in C_c^\infty(\Omega,{\mathbb R}^N),\lVert\phi({\bf x})\rVert_{\infty}\leq1\},
		\end{equation}\noindent
		where $\rm div$ is the divergence operator and $\phi(\cdot)$ is a differentiable vector field with compact support. If the function $f_\Theta({\bf x}):\Omega\rightarrow{\mathbb R}$ is differentiable w.r.t. the input $\bf x$, then its functional TV \eqref{FunTV} is equal to $
			V(f_\Theta;\Omega) = \sum_{d=1}^N\int_\Omega \left|\frac{\partial f_\Theta({\bf x})}{\partial {\bf x}_{(d)}}\right| d{\bf x}.$
	\end{lemma}\par
We provide a proof of Lemma \ref{lemma_TV_DF} in Sec. \ref{proof_lemma_TV_DF} of the appendix.} It can be seen that the TV of a differentiable function is the sum of integration of partial derivatives of the function. Recall that our NeurTV \eqref{NeurTV} is defined as the sum of the partial derivatives of DNN over a discrete set, and hence we can build the connection between the classical TV \eqref{FunTV} and NeurTV by viewing NeurTV as the numerical integration for approximating the variational regularization \eqref{FunTV}. Here, we consider uniformly sampled partitions in the one-dimensional case to study the variational approximation. The numerical integration in the one-dimensional case has the following form by taking quadratures at right-hand values at each interval:
\begin{equation}\small\label{V_approx}
V(f_\Theta;[a,b])={\int_a^b \left|\frac{df_\Theta({x})}{dx}\right| d{x}}\approx\frac{b-a}{n}\underbrace{\sum_{i=1}^n\left|\frac{df_\Theta(x_i)}{dx_i}\right|}_{\rm NeurTV},
\end{equation}\noindent
where $f_\Theta(x):\Omega\rightarrow{\mathbb R}$ is a differentiable function w.r.t. the input $x$, $\Omega=[a,b]\subset{\mathbb R}$, and $x_0,x_1,\cdots,x_n$ are some uniformly sampled points over $[a,b]$, where $x_i=a+\frac{i(b-a)}{n}$. The following lemma shows that the approximation is exact in the infinite condition (i.e., $n\rightarrow\infty$).
	\begin{lemma}\label{lemma_n}
For a function $f_\Theta(x):[a,b]\rightarrow{\mathbb R}$ that is differentiable w.r.t. the input $x$, we consider uniformly sampled points over $[a,b]$, i.e., $x_i=a+\frac{i(b-a)}{n}$ $(i=0,1,\cdots,n)$. Then 
\begin{equation}\small\label{gtv_equal}
V(f_\Theta;[a,b])=\int_a^b \left|\frac{df_\Theta({x})}{dx}\right| d{x}=\lim_{n\rightarrow\infty}\frac{b-a}{n}{\sum_{i=1}^n\left|\frac{df_\Theta(x_i)}{dx_i}\right|}. 
\end{equation}\noindent
	\end{lemma}
	\begin{proof}
		The first equality of \eqref{gtv_equal} follows from Lemma \ref{lemma_TV_DF} and the second equality of \eqref{gtv_equal} follows from the definition of Riemann integration. 
	\end{proof}\par
It should be noted that the classical functional TV, $V(f_\Theta;[a,b])$, can also be approximated via difference operators without using the function derivatives, as stated below.
{\begin{lemma}[Total variation based on uniform partitions]\label{lemma_uniform}
For a function $f_\Theta(x):[a,b]\rightarrow{\mathbb R}$ that is differentiable w.r.t. the input $x$, we consider the uniform partitions over $[a,b]$, i.e., $x_i = a+\frac{i(b-a)}{n}$ $(i=0,1,\cdots,n)$. Then we have 
\begin{equation}\small\label{uniformTV}
V(f_\Theta;[a,b])=\lim_{n\rightarrow \infty}\sum_{i=1}^n{|f_\Theta(x_i)-f_\Theta(x_{i-1})|}.
\end{equation}\noindent
	\end{lemma}}
{The proof of Lemma \ref{lemma_uniform} can be found in Sec. \ref{proof_lemma_uniform} of the appendix.} From \eqref{gtv_equal} \& \eqref{uniformTV} we can see that both the derivative-based and difference-based NeurTV have tight connections with the classical functional TV (i.e., $V(f_\Theta;[a,b])$), and they are all equivalent when $n\rightarrow\infty$. However, in the finite case (i.e., with finite $n$), the derivative-based NeurTV and difference-based NeurTV are not equivalent, which rationally explains their performance gap in the finite condition (see for example Fig. \ref{fig_n}). 
\subsubsection{Arbitrary Resolution NeurTV Inspired by Variational Approximation}\label{sec_ar}
Inspired by the variational approximation, we develop some variants of NeurTV to further improve its effectiveness. First, we consider the arbitrary resolution NeurTV by using larger number of partitions $n$ in \eqref{V_approx}. Specifically, we show that the truncation error of the numerical integration in \eqref{V_approx} is inversely proportional to the number of partitions $n$. Hence a larger $n$ is expected to attain lower approximation error of the numerical integration.
	\begin{lemma}\label{lemma_error_gv}
		Suppose that $f_\Theta(x):[a,b]\rightarrow{\mathbb R}$ is differentiable w.r.t. the input $x$. The truncation error $R$ of numerical integration by using uniform partitioned quadratures for approximating the functional TV (i.e., approximating $V(f_\Theta;[a,b])$) satisfies
		\begin{equation}\small
			R:=\left|\left(V(f_\Theta;[a,b])-\frac{b-a}{n}{\sum_{i=1}^n\left|\frac{df_\Theta(x_i)}{dx_i}\right|}\right)\right|\leq
			\frac{(b-a)^2}{2n}\left|\frac{df^2_\Theta(\eta)}{d\eta^2}\right|,
		\end{equation}\noindent
		where $\eta\in(a,b)$ and $\frac{df^2_\Theta(\eta)}{d\eta^2}$ denotes the second-order derivative.
	\end{lemma}\par
{The proof of Lemma \ref{lemma_error_gv} can be found in Sec. \ref{proof_lemma_error_gv} of the appendix.} Lemma \ref{lemma_error_gv} shows that a larger number of partitions $n$ leads to potentially lower approximation error of the numerical integration. For difference-based NeurTV, we have an analogous lemma.
	\begin{lemma}\label{lemma_error_tv}
		Given a function $f_\Theta(x):[a,b]\rightarrow{\mathbb R}$ that is differentiable w.r.t. the input $x$, we consider the uniform partitions ${\Gamma}_n := (x_0,x_1,\cdots,x_n)$ of $[a,b]\subset{\mathbb R}$, i.e., $x_i = a+\frac{i(b-a)}{n}$ $(i=0,1,\cdots,n)$. Then we have $
			\sum_{{\Gamma}_n}|f_\Theta(x_{i})-f_\Theta(x_{i-1})|\leq\sum_{{\Gamma}_{2n}}|f_\Theta(x_{i})-f_\Theta(x_{i-1})|\leq V(f_\Theta;[a,b]).$
	\end{lemma}\par 
{The proof of Lemma \ref{lemma_error_tv} can be found in Sec. \ref{proof_lemma_error_tv} of the appendix.} Lemma \ref{lemma_error_tv} shows that using the difference operators with larger number of partitions (i.e., from $n$ to $2n$) leads to lower error for approximating the functional TV, $V(f_\Theta;[a,b])$. Inspired by the variational approximation analysis in Lemmas \ref{lemma_error_gv}-\ref{lemma_error_tv}, we can set a larger number of partitions $n$ to reduce the approximation error for approximating the functional TV $V(f_\Theta;[a,b])$, hence leading to potentially better numerical results. Taking the image recovery as an example, we can sample arbitrary resolution meshgrid beyond the original image resolution to perform the NeurTV based on the continuous representation $f_\Theta(\cdot)$. More specifically, we consider the following two image recovery models by using the arbitrary resolution NeurTV regularization. {First, the arbitrary resolution derivative-based NeurTV model can be formulated as
\begin{equation}\small\label{model_gtv}
	\begin{split}	
	\min_{\Theta}\sum_{(i,j)\in\Gamma}{\rm Fidelity}({\bf O}_{(i,j)}, f_\Theta(i,j))+\lambda\;\Psi_{\rm NeurTV}(\Theta),\;
	\Psi_{\rm NeurTV}(\Theta)={\sum_{{\bf x}\in\Gamma_n}\left| \frac{\partial f_\Theta({\bf x})}{\partial {\bf x}_{(1)}}\right|+\left| \frac{\partial f_\Theta({\bf x})}{\partial {\bf x}_{(2)}}\right|},
	\end{split}
\end{equation}\noindent
where ${\bf O}$ denotes the observed data, ${\Gamma}$ denotes the original image meshgrid points, $\lambda$ is a trade-off parameter, and $\Psi_{\rm NeurTV}(\Theta)$ is the derivative-based NeurTV regularization. Here, $\Gamma_{n}$ denotes a point set that contains points in a denser meshgrid with number of meshgrid partitions $n$. In the derivative-based NeurTV regularization, we construct higher-resolution NeurTV beyond the image resolution by using local derivatives of DNN $f_\Theta(\cdot)$ on denser meshgrid. The higher-resolution regularization aims to reduce the approximation error of the numerical integration based on Lemma \ref{lemma_error_gv}. \par 
Similarly, we consider the following arbitrary resolution difference-based NeurTV model:
\begin{equation}\small\label{model_tv}
	\begin{split}	
	&\min_{\Theta}\sum_{(i,j)\in\Gamma}{\rm Fidelity}({\bf O}_{(i,j)}, f_\Theta(i,j))+\lambda\;\Psi_{\rm Diff\mbox{-}NeurTV}(\Theta),\\&
	\Psi_{\rm Diff\mbox{-}NeurTV}(\Theta)=\sum_{{\bf x}\in\Gamma_n}|f_\Theta({\bf x}_{(1)}+\Delta x,{\bf x}_{(2)})-f_\Theta({\bf x})|+|f_\Theta({\bf x}_{(1)},{\bf x}_{(2)}+\Delta y)-f_\Theta({\bf x})|,
	\end{split}
	\end{equation}\noindent
where $\Psi_{\rm Diff\mbox{-}NeurTV}(\Theta)$ is the difference-based NeurTV regularization. Here, $\Gamma_{n}$ denotes a point set that contains points in a denser meshgrid with number of meshgrid partitions $n$, and $\Delta x$ and $\Delta y$ denote the lengthes of the partitioned interval.} In the difference-based regularization $\Psi_{\rm Diff\mbox{-}NeurTV}(\Theta)$, we construct higher-resolution difference-based NeurTV beyond the original image resolution by using the local differences of the DNN $f_\Theta(\cdot)$. The numerical experiments in Fig. \ref{fig_n} shall show that a higher-resolution NeurTV regularization (i.e., a larger number of partitions $n$) inclines to obtain better performances. This can be rationally explained by the lower variational approximation error induced by the larger number of partitions $n$; see Lemmas \ref{lemma_error_gv}-\ref{lemma_error_tv}. In summary, classical TV-based methods \cite{SIAM_TV,SIIMS_TDV,SIIMS_TV_2024,TVTLS} utilize discrete representation to construct TV with the original image resolution. Differently, we use the DNN $f_\Theta(\cdot)$ to {continuously represent data}, which allows us to construct arbitrary resolution NeurTV to obtain potentially better performances for data recovery.
	\subsubsection{Space-variant NeurTV Inspired by Variational Approximation}\label{sec_wtv}
	Inspired by the variational approximation analysis, we further develop the space-variant NeurTV regularization, which imposes spatially varying scale and directional parameters to better describe local structures. We consider the two-dimensional case as an example.\par 
	\begin{definition}[Space-variant NeurTV]
		Given a two-dimensional function $f_\Theta({\bf x}):\Omega\rightarrow{\mathbb R}$ that is differentiable w.r.t. the input $\bf x$ where $\Omega \subset{\mathbb R}^2$, a point set $\Gamma\subset\Omega$, and some scale parameters $\alpha_{\bf x}$ and directional parameters $a_{\bf x},\theta_{\bf x}$ depending on ${\bf x}\in\Gamma$, the space-variant NeurTV regularization is defined as 
		\begin{equation}\small\label{WGV}
			\Psi_{{\rm NeurTV}_\theta^\alpha}(\Theta)=\sum_{{\bf x}\in\Gamma} \alpha_{\bf x}\left\lVert\begin{pmatrix}
				{a_{\bf x}}&0\\
				0&2-a_{\bf x}\\
			\end{pmatrix}
			\begin{pmatrix}
				\cos\theta_{\bf x}&\sin\theta_{\bf x}\\
				\sin\theta_{\bf x}&-\cos\theta_{\bf x}\\
			\end{pmatrix}
			\begin{pmatrix}
				\frac{\partial f_\Theta({\bf x})}{\partial{\bf x}_{(1)}}\\
				\frac{\partial f_\Theta({\bf x})}{\partial{\bf x}_{(2)}}
			\end{pmatrix}
			\right\rVert_{\ell_1}.
		\end{equation}\noindent
	\end{definition}\par 
\begin{algorithm}[t]
	\begin{spacing}{0.85}
		\renewcommand\arraystretch{1.2}
		\caption[Caption for LOF]{{Space-variant NeurTV regularization-based method for image denoising}}\label{alg}
		\begin{algorithmic}[1]
			\renewcommand{\algorithmicrequire}{\textbf{Input:}} 
			\REQUIRE
			Observed image ${\bf O}\in{\mathbb R}^{n_1\times n_2}$, trade-off parameter $\lambda$, number of partitions $n$;
			\renewcommand{\algorithmicrequire}{\textbf{Initialization:}} 
			\REQUIRE Randomly initialize the DNN parameters $\Theta$ of an untrained DNN $f_\Theta(\cdot):{\mathbb R}^2\rightarrow{\mathbb R}$;
			\STATE Construct a meshgrid coordinate set $\Gamma\subset{\mathbb R}^2$ based on the partition number $n$; 
			\STATE Construct the loss function with the space-variant NeurTV regularization:
			\begin{equation}\small\label{model_alg}
				\begin{small}
					\min_{\Theta}\sum_{(i,j)}({\bf O}_{(i,j)}-f_\Theta(i,j))^2+\lambda\underbrace{\sum_{{\bf x}\in\Gamma} \alpha_{\bf x}\left\lVert\begin{pmatrix}
							{a_{\bf x}}&0\\
							0&2-a_{\bf x}\\
						\end{pmatrix}
						\begin{pmatrix}
							\cos\theta_{\bf x}&\sin\theta_{\bf x}\\
							\sin\theta_{\bf x}&-\cos\theta_{\bf x}\\
						\end{pmatrix}
						\begin{pmatrix}
							\frac{\partial f_\Theta({\bf x})}{\partial{\bf x}_{(1)}}\\
							\frac{\partial f_\Theta({\bf x})}{\partial{\bf x}_{(2)}}
						\end{pmatrix}
						\right\rVert_{\ell_1}}_{{\rm Space\mbox{-}variant\;NeurTV\;regularization}\;\Psi_{{\rm NeurTV}_{\theta}^\alpha}(\Theta)}.
				\end{small}
			\end{equation}\noindent
			\WHILE {not converge}
			\STATE Calculate the DNN output $f_\Theta({\bf x})$ for each input coordinate ${\bf x}\in\Gamma$;
			\STATE Calculate the function derivatives of DNN outputs w.r.t. input coordinates, i.e., $\frac{\partial f_\Theta({\bf x})}{\partial{\bf x}_{(1)}},\frac{\partial f_\Theta({\bf x})}{\partial{\bf x}_{(2)}}$, using the automatic differentiation system;
			\STATE Calculate the unsupervised loss in \eqref{model_alg} with the observed image $\bf O$ and the calculated function derivatives;
			\STATE Update DNN parameters $\Theta$ by Adam based on the computed unsupervised loss;
			\STATE Update scale parameters $\alpha_{\bf x}$ of the space-variant NeurTV via \eqref{WTV_para2};
			\STATE Update directional parameters $a_{\bf x},\theta_{\bf x}$ of the space-variant NeurTV via \eqref{DTV_para};
			\ENDWHILE
			\renewcommand{\algorithmicrequire}{\textbf{Output:}}
			\REQUIRE The recovered image by querying $f_\Theta(\cdot)$ on meshgrid points.
		\end{algorithmic}
	\end{spacing}
\end{algorithm}
The aim of the space-variant NeurTV is to make distinction between smooth and non-smooth regions by assigning different scale parameters $\alpha_{\bf x}$ to different elements $\bf x$. Meanwhile, the space-variant NeurTV is also capable of distinguishing the local directions of the patterns by assigning different direction parameters $a_{\bf x},\theta_{\bf x}$, where $\theta_{\bf x}$ and $a_{\bf x}$ correspond to the direction and the magnitude imposed on the directional derivative with direction $\theta_{\bf x}$, respectively. According to the directional derivative Lemma \ref{lemma_DTV}, the regularization $\Psi_{{\rm NeurTV}_\theta^\alpha}(\Theta)$ penalizes the derivative along the direction $\theta_{\bf x}$ with magnitude $a_{\bf x}$ and penalizes the derivative along the orthogonal direction of $\theta_{\bf x}$ with magnitude $2-a_{\bf x}$. The next step is to design effective strategies to determine the parameters $\alpha_{\bf x},a_{\bf x},\theta_{\bf x}$. Inspired by the weighted total variation methods \cite{PR_WTV,TGRS_S2S}, the scale parameter $\alpha_{\bf x}$ can be simply set as the reciprocal of the local derivative:
	\begin{equation}\small\label{WTV_para1}
		\alpha_{\bf x} = \left(\left|\frac{\partial f_\Theta({\bf x})}{\partial{\bf x}_{(1)}}\right|+\left|\frac{\partial f_\Theta({\bf x})}{\partial{\bf x}_{(2)}}\right|+\epsilon\right)^{-1},
	\end{equation}\noindent\noindent
	where $\epsilon$ is a small constant. The philosophy of this scale parameter selection is to enhance the smoothness over smoother regions while keep the scale parameter smaller in sharper regions with larger derivatives. This can keep the sharpness of the image edges and promote the smoothness of relatively flat regions. In practice, we can update the weight $\alpha_{\bf x}$ based on the learned continuous representation $f_\Theta(\cdot)$ in the last iteration.\par 
We also introduce another scale parameter selection strategy for NeurTV. Specifically, we use the second-order derivative information to determine scale parameters, which was not considered in previous weighted TV-based methods \cite{PR_WTV,TGRS_S2S}. This selection strategy is inspired by the approximation error analysis of Lemma \ref{lemma_error_gv}, where the variational approximation error is shown to be related to the second-order function derivative. We show the following example that if we enlarge the interval of the numerical integration in regions with smaller second-order derivative, then a lower variational approximation error tends to be obtained.
\begin{lemma}\label{lemma_WTV}
Consider the numerical integration for $V(f_\Theta,[a,b])=\int_a^b|\frac{df_\Theta({x})}{dx}|d{x}$ with uniform partitions $(x_0,x_1,\cdots,x_n)$. Assume that $|\frac{d f_\Theta({x})}{d{x}}|$ is non-decreasing in $[a,b]$ and there exists $j$ such that $\frac{d^2f(t_1)}{dt_1^2}>\frac{d^2f(t_2)}{dt_2^2}\geq0$ for any $t_1\in[x_{j-1},x_j]$ and $t_2\in(x_j,x_{j+1}]$. Then the truncation error using non-uniform partitions $(x_0,x_1,\cdots,x_{j-1},x_j-\delta,x_{j+1},\cdots,x_n)$ is less than the truncation error using uniform partitions provided that $\delta>0$ is small enough.
\end{lemma}\par 
{The proof of Lemma \ref{lemma_WTV} can be found in Sec. \ref{proof_lemma_WTV} of the appendix.} Lemma \ref{lemma_WTV} indicates that under mild conditions, a lower approximation error can be obtained by sparing some length $\delta$ from intervals with larger second-order derivative to those intervals with smaller second-order derivative. This motivates us to assign larger scale parameters $\alpha_{\bf x}$ to elements with smaller second-order derivative:
	\begin{equation}\small\label{WTV_para2}
		\alpha_{\bf x} = \left(\left|\frac{\partial^2 f_\Theta({\bf x})}{\partial{\bf x}_{(1)}^2}\right|+\left|\frac{\partial^2 f_\Theta({\bf x})}{\partial{\bf x}_{(2)}^2}\right|+\epsilon\right)^{-1}.
	\end{equation}\noindent\noindent
Similarly, the scale parameter can be updated based on the learned continuous representation $f_\Theta(\cdot)$ in the last iteration. \par
	Next, we further discuss the selection of the directional parameters $a_{\bf x},\theta_{\bf x}$. We propose the following selection schemes for the directional parameters:
	\begin{equation}\small\label{DTV_para}
		\theta_{\bf x} = \arctan\left({\frac{\partial f_\Theta({\bf x})}{\partial{\bf x}_{(2)}}}\Big/{\frac{\partial f_\Theta({\bf x})}{\partial{\bf x}_{(1)}}}\right),
		\;a_{\bf x} = \frac{\left|(\sin\theta_{\bf x},-\cos\theta_{\bf x})	(\frac{\partial f_\Theta({\bf x})}{\partial{\bf x}_{(1)}},\frac{\partial f_\Theta({\bf x})}{\partial{\bf x}_{(2)}})^T\right|}{\left|(\cos\theta_{\bf x},\sin\theta_{\bf x})
				(\frac{\partial f_\Theta({\bf x})}{\partial{\bf x}_{(1)}},\frac{\partial f_\Theta({\bf x})}{\partial{\bf x}_{(2)}})^T\right|}.
	\end{equation}\noindent
	In \eqref{DTV_para}, the directional parameters $\theta_{\bf x}$ and $a_{\bf x}$ are determined by the local derivatives of $f_\Theta(\cdot)$. Specifically, when $\theta_{\bf x}$ admits \eqref{DTV_para}, the directional derivative of $f_\Theta(\cdot)$ along the direction $(\cos\theta_{\bf x},\sin\theta_{\bf x})$ attains the maximum intensity for all $\theta_{\bf x}\in[0,2\pi)$, and hence we assign a lower weight $a_{\bf x}<1$ to this direction to keep the sharpness. In contrast, we assign a larger weight $2-a_{\bf x}$ to its orthogonal direction, i.e., $(\sin\theta_{\bf x},-\cos\theta_{\bf x})$, to enhance the smoothness along the orthogonal direction of the gradient. The directional parameters $a_{\bf x},\theta_{\bf x}$ can be updated based on the learned continuous representation $f_\Theta(\cdot)$ in the last iteration by using \eqref{DTV_para}. \par 
{Finally, we describe the overall procedure of our space-variant NeurTV by taking the image denoising as an example. Especially, we would like to clarify that the DNN $f_\Theta(\cdot)$ is learned by solely using the observed data. Given a noisy image ${\bf O}\in{\mathbb R}^{n_1\times n_2}$, we form the following unsupervised loss function based on the space-variant NeurTV regularization: 
	\begin{equation}\small\label{loss}
		\min_{\Theta}\sum_{(i,j)}({\bf O}_{(i,j)}-f_\Theta(i,j))^2+\lambda\Psi_{{\rm NeurTV}_{\theta}^\alpha}(\Theta),
\end{equation}\noindent
where $\sum_{(i,j)}({\bf O}_{(i,j)}-f_\Theta(i,j))^2$ is the fidelity term and $\Psi_{{\rm NeurTV}_{\theta}^\alpha}(\Theta)$ is the space-variant NeurTV regularization defined in \eqref{WGV}. We use this loss function to train the DNN $f_\Theta(\cdot)$ in a purely unsupervised manner (i.e., the loss function solely uses the single observed data $\bf O$ without any other training data). Especially, the derivative-based NeurTV relies on calculating the function derivatives of the DNN output $f_\Theta({\bf x})$ w.r.t. the input coordinate $\bf x$. At each iteration of the algorithm, the function derivatives of the DNN output $f_\Theta({\bf x})$ w.r.t. the input coordinate $\bf x$ are explicitly calculated using the automatic differentiation system in Pytorch (e.g., by using the {\tt torch.autograd.grad()} command in Pytorch\footnote{\url{https://pytorch.org/docs/stable/generated/torch.autograd.grad.html}}). More details of the automatic differentiation system can refer to \cite{ADS}. The observed data $\bf O$ is used to formulate the unsupervised loss function \eqref{loss}. The denoised image can be obtained by querying the DNN $f_\Theta(\cdot)$ on meshgrid points after training. To more clearly introduce the proposed method, we provide a workflow of the space-variant NeurTV regularization-based model for image denoising in Algorithm \ref{alg}. \par }
To numerically validate the effectiveness of our new space-variant NeurTV, we show the results of image recovery using space-variant NeurTV with first-order or second-order derivatives-based scale parameters selection schemes in Fig. \ref{fig_wgtv}. It can be observed that the space-variant NeurTV obtains better performances than the original NeurTV. Meanwhile, our second-order derivatives-based scale parameters selection scheme \eqref{WTV_para2} obtains a better result than the first-order derivatives-based scheme \eqref{WTV_para1}.
\begin{figure}[t]
	\tiny
	\setlength{\tabcolsep}{0.9pt}
	\begin{center}
		\begin{tabular}{cccccc}
			\includegraphics[width=0.14\textwidth]{peppers_5c1gt.pdf}
			&\includegraphics[width=0.14\textwidth]{peppers_5c1.pdf}
			&\includegraphics[width=0.14\textwidth]{peppers_5c1TV.pdf}
			&\includegraphics[width=0.14\textwidth]{peppers_5c1SIREN-GTV.pdf}
			&\includegraphics[width=0.14\textwidth]{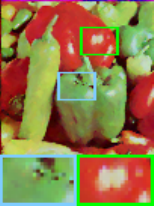}
			&\includegraphics[width=0.14\textwidth]{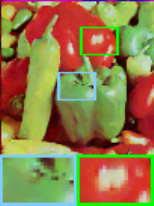}\\	
			PSNR Inf&PSNR 20.38 &
			
			PSNR 23.13 &
			PSNR 25.43 & 
			PSNR 25.77 &
			PSNR 25.93 \\
			Original&Noisy&TV&NeurTV&${{\rm NeurTV}_\theta^\alpha}$ \eqref{WTV_para1}&${{\rm NeurTV}_\theta^\alpha}$ \eqref{WTV_para2}\\
		\end{tabular}
		\vspace{-0.2cm}
	\end{center}
\caption{The results of image denoising with Gaussian noise deviation $0.1$ using classical TV, the proposed NeurTV regularization, and the proposed space-variant NeurTV regularization (denoted by ${\rm NeurTV}_\theta^\alpha$). We consider the first-order derivatives-based scale parameter selection strategy \eqref{WTV_para1} or the second-order derivatives-based scale parameter selection strategy \eqref{WTV_para2} in the space-variant NeurTV. \label{fig_wgtv}}
\vspace{-0.2cm}
\end{figure}
\section{Numerical Experiments}\label{sec_exp}
In this section, we conduct numerical experiments on different applications involving both meshgrid data and non-meshgrid data. For easy reference, we first give an intuitive summarization of different applications and the corresponding NeurTV configurations in Table \ref{tab_exp}. {By using a DNN to continuously represent data, NeurTV can capture intrinsic local correlations of data with different modalities, e.g., meshgrid data such as images and non-meshgrid data such as point clouds. Thus, the proposed NeurTV is naturally suitable for different modalities of data. Since different data exhibits different local correlations, e.g., the spatial local correlation underlying color images and the spatial-spectral local correlation underlying hyperspectral images, different NeurTV variants can be considered (e.g., the space-variant NeurTV for color images and the spatial-spectral NeurTV for hyperspectral images).}
	\subsection{Image Denoising (on Meshgrid)} {For the image denoising problem, we consider the following optimization model based on the derivative-based space-variant NeurTV:}
	\begin{equation}\small\label{denoising_model}
		\min_{\Theta}\sum_{(i,j)\in{\rm meshgrid}}({\bf O}_{(i,j)}-f_\Theta(i,j))^2+\lambda\Psi_{{\rm NeurTV}_{\theta}^\alpha}(\Theta),
	\end{equation}\noindent
where ${\bf O}\in{\mathbb R}^{n_1\times n_2}$ denotes the observed noisy image and $\Psi_{{\rm NeurTV}_{\theta}^\alpha}(\Theta)$ denotes the proposed space-variant NeurTV regularization defined in \eqref{WGV}. $\lambda$ is a trade-off parameter. The scale parameter $\alpha$ and the directional parameters $\theta,a$ in space-variant NeurTV regularization are determined through the schemes presented in \eqref{WTV_para2} \& \eqref{DTV_para}. We use the tensor factorization-based DNN \cite{TPAMI_Luo} (see Fig. \ref{fig_INR} for details) as the function $f_\Theta:{\mathbb R}^2\rightarrow{\mathbb R}$ for continuous representation. The overall algorithm of tackling the denoising model \eqref{denoising_model} is illustrated in Algorithm \ref{alg}. We use a higher-resolution meshgrid (three times larger than that of the image) to construct the space-variant NeurTV regularization $\Psi_{{\rm NeurTV}_{\theta}^\alpha}(\Theta)$.    
\begin{table}[t]
	\caption{Summary of the applications considered in experiments and the corresponding NeurTV configuration.\label{tab_exp}}\vspace{-0.2cm}
	\begin{center}
	\begin{spacing}{1.45}
		\tiny
		\setlength{\tabcolsep}{6pt}
		\begin{tabular}{ccccc}
			\toprule
		Application&Characteristic&NeurTV configuration&Model&Algorithm\\
		\midrule
		Image denoising&On meshgrid&\tabincell{c}{Space-variant NeurTV \eqref{WGV}\vspace{-0.1cm}\\(with space-variant scales and directions)}&\eqref{denoising_model}&\tabincell{c}{Iterative update+Adam\vspace{-0.1cm}\\(Algorithm \ref{alg})}\\
		Image inpainting&On meshgrid&\tabincell{c}{Space-variant NeurTV \eqref{WGV}\vspace{-0.1cm}\\(with space-variant scales and directions)}&\eqref{inpainting_model}&\tabincell{c}{Iterative update+Adam\vspace{-0.1cm}\\(Algorithm \ref{alg})}\\
		
		\tabincell{c}{Hyperspectral image\vspace{-0.1cm}\\mixed noise removal}&On meshgrid&Spatial-spectral NeurTV \eqref{SSNeurTV}&\eqref{HSI_model}&\tabincell{c}{Alternating minimization}\\
		
		Point cloud recovery&Beyond meshgrid&\tabincell{c}{Space-variant NeurTV \eqref{NeurTV_point}\vspace{-0.1cm}\\(with space-variant scales)}&\eqref{point_model}&Iterative update+Adam\\
		
		\tabincell{c}{Spatial transcriptomics\vspace{-0.1cm}\\reconstruction}&Beyond meshgrid&\tabincell{c}{Space-variant NeurTV \eqref{WGV}\vspace{-0.1cm}\\(with space-variant scales and directions)}&\eqref{ST_model}&Iterative update+Adam\\
			\bottomrule
		\end{tabular}
		\end{spacing}
	\end{center}
	\vspace{-0.6cm}
\end{table}    
By simple parameter tuning, we fix the trade-off parameter $\lambda$ as $4\times10^{-4}$ for all datasets. {The width and depth of the DNN are set to 150 and 3, respectively.} We use three baselines for image denoising, including the classical TV based on ADMM\footnote{\url{https://htmlpreview.github.io/?https://github.com/Yunhui-Gao/total-variation-image-denoising/}}, the weighted TV regularized deep image prior (DIPWTV) method \cite{DIP,DIPWTV}, and the higher-order directional TV (HDTV) method \cite{SIIMS_TDV}. {Here, the classical TV and HDTV are model-based methods. DIPWTV is a deep learning and TV-based method which uses space-variant weights to improve the flexibility of TV, and uses the observed image to train its DNN in an unsupervised manner.} We use five color images to test these methods, including ``Peppers'', ``Boat'', ``House'', ``Plane'', and ``Earth'', which are online available\footnote{\url{https://sipi.usc.edu/database/database.php}}. The resolution of images is downsampled to fit our computing resources. We add Gaussian noise with the noise standard deviation 0.1 and 0.2 to test different methods. For fair comparison, we denoise the image channel by channel for all methods. The results are quantitatively evaluated by PSNR and SSIM.\par 
\begin{figure}[t]
	\scriptsize
	\setlength{\tabcolsep}{0.9pt}
	\begin{center}
		\begin{tabular}{cccccc}
			\includegraphics[width=0.14\textwidth]{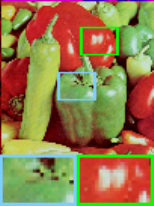}
			&\includegraphics[width=0.14\textwidth]{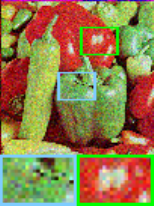}
			&\includegraphics[width=0.14\textwidth]{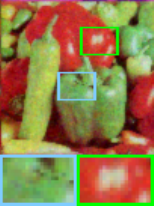}
			&\includegraphics[width=0.14\textwidth]{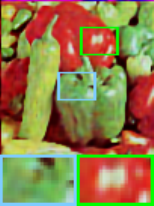}
			&\includegraphics[width=0.14\textwidth]{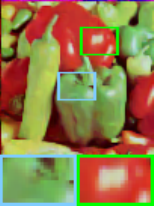}
			&\includegraphics[width=0.14\textwidth]{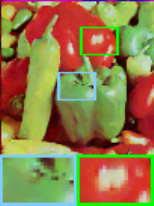}\\	
			PSNR Inf &PSNR 20.38 &
			
			PSNR 23.13 &
			PSNR 25.26&
			PSNR 25.47 &
			PSNR 25.93 \\
			\includegraphics[width=0.14\textwidth]{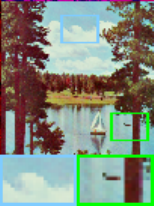}
			&\includegraphics[width=0.14\textwidth]{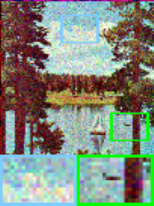}
			&\includegraphics[width=0.14\textwidth]{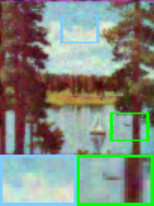}
			&\includegraphics[width=0.14\textwidth]{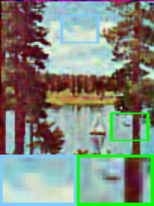}
			&\includegraphics[width=0.14\textwidth]{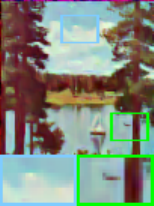}
			&\includegraphics[width=0.14\textwidth]{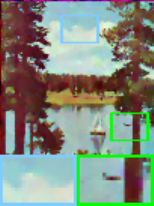}\\	
			PSNR Inf&PSNR 20.28 &			
			
			PSNR 21.92 &
			PSNR 23.77 &
			PSNR 23.39&
			PSNR 24.67\\
			\includegraphics[width=0.14\textwidth]{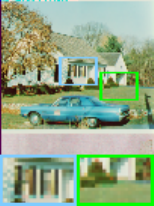}
			&\includegraphics[width=0.14\textwidth]{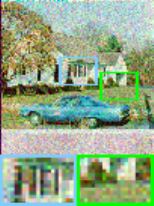}
			&\includegraphics[width=0.14\textwidth]{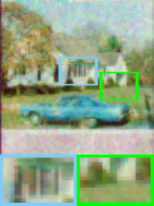}
			&\includegraphics[width=0.14\textwidth]{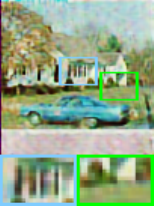}
			&\includegraphics[width=0.14\textwidth]{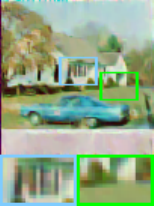}
			&\includegraphics[width=0.14\textwidth]{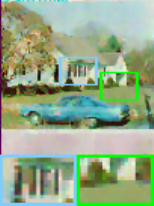}\\	
			PSNR Inf&PSNR 20.20 &
			
			PSNR 22.43 &
			PSNR 23.98&
			PSNR 23.83  &
			PSNR 24.82 \\
			Original&Noisy&TV&DIPWTV&HDTV&${\rm NeurTV}_\theta^\alpha$\\
		\end{tabular}
		\vspace{-0.2cm}
	\end{center}
	\caption{The results of image denoising by different methods on ``Peppers'', ``Boat'', and ``House'' with Gaussian noise under noise deviation 0.1. Here, DIPWTV is a weighted TV-based method with space-variant scales, and HDTV is a higher-order directional TV-based method. \label{fig_image_denoising}}
	\vspace{-0.2cm}
\end{figure}
\begin{table}[t]
	\caption{The quantitative results by different methods for image denoising.\label{tab_denoising}}\vspace{-0.2cm}
	\begin{center}
		\tiny
		\setlength{\tabcolsep}{6pt}
		\begin{tabular}{clcccccccccc}
			\toprule
			\multirow{2}*{Noise level}&\multirow{2}*{Method}&\multicolumn{2}{c}{``Peppers''}&\multicolumn{2}{c}{``Boat''}&\multicolumn{2}{c}{``House''}
			&\multicolumn{2}{c}{``Plane''}&\multicolumn{2}{c}{``Earth''}\\
			~&~&PSNR&SSIM&PSNR&SSIM&PSNR&SSIM&PSNR&SSIM&PSNR&SSIM\\
			\midrule
			\multirow{5}*{0.1}&Noisy&20.38 & 0.610 &
			20.28 & 0.642 &
			20.20 & 0.616 &
			20.12 & 0.489 &
			20.09 & 0.562  \\
			~&TV&23.13 & 0.768 &
			21.92 & 0.731 &
			22.43 & 0.710 &
			23.19 & 0.660 &
			23.97 & 0.685 \\
			~&DIPWTV&25.26 & 0.813 &
			23.77 & 0.790 &
			23.98 & 0.787 &
			25.47 & 0.804 &
			24.52 & 0.724 \\
			~&HDTV&25.47 & 0.822 &
			23.39 & 0.781 &
			23.83 & 0.792 &
			25.06 & 0.785 &
			24.57 & 0.695  \\
			~&${\rm NeurTV}_\theta^\alpha$&\bf25.93 & \bf0.826 &
			\bf24.67 & \bf0.809 &
			\bf24.82 & \bf0.823 &
			\bf26.21 & \bf0.806 &
			\bf25.12 & \bf0.731 \\
			\midrule
			\multirow{5}*{0.2}&Noisy&14.90 & 0.385 &
			14.98 & 0.431 &
			14.71 & 0.399 &
			14.92 & 0.309 &
			14.47 & 0.298 \\
			~&TV&20.48 & 0.608 &
			19.69 & 0.583 &
			20.04 & 0.559 &
			20.42 & 0.470 &
			21.01 & 0.508 \\			
			~&DIPWTV&22.21 & 0.706 &
			20.56 & 0.656 &
			20.96 & 0.651 &
			22.36 & 0.651 &
			22.23 & 0.544  \\
			~&HDTV&22.34 & 0.714 &
			20.37 & 0.630 &
			20.73 & 0.636 &
			22.12 & 0.653 &
			22.54 & 0.541 \\
			~&${\rm NeurTV}_\theta^\alpha$&\bf22.46 & \bf0.705 &
			\bf21.14 &\bf 0.664 &
			\bf21.36 &\bf 0.665 &
			\bf22.60 &\bf 0.662 &
			\bf22.60 &\bf 0.546 \\
			\bottomrule
		\end{tabular}
	\end{center}
	\vspace{-0.2cm}
\end{table}  
\begin{figure}[t]
	\scriptsize
	\setlength{\tabcolsep}{0.9pt}
	\begin{center}
		\begin{tabular}{cccccc}
			\includegraphics[width=0.14\textwidth]{plane_2p2gt.pdf}
			&\includegraphics[width=0.14\textwidth]{plane_2p2.pdf}
			&\includegraphics[width=0.14\textwidth]{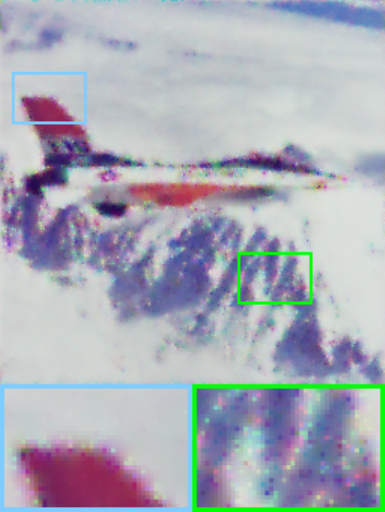}
			&\includegraphics[width=0.14\textwidth]{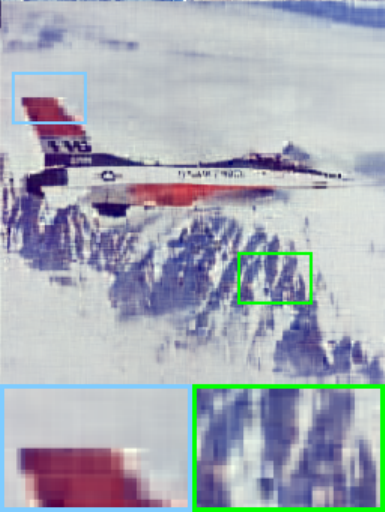}
			&\includegraphics[width=0.14\textwidth]{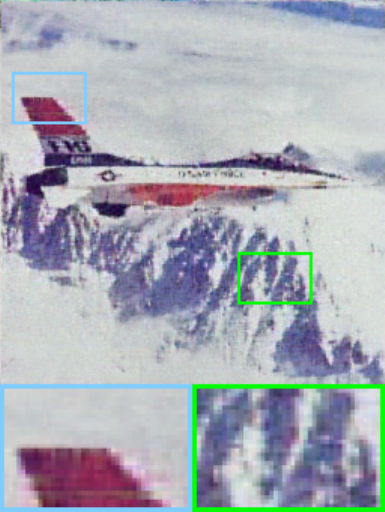}
			&\includegraphics[width=0.14\textwidth]{plane_2p2gtv.pdf}\\	
			PSNR Inf&PSNR 3.65 &
			
			PSNR 23.66&
			PSNR 25.43  &
			PSNR 25.65  &
			PSNR 26.13  \\
			\includegraphics[width=0.14\textwidth]{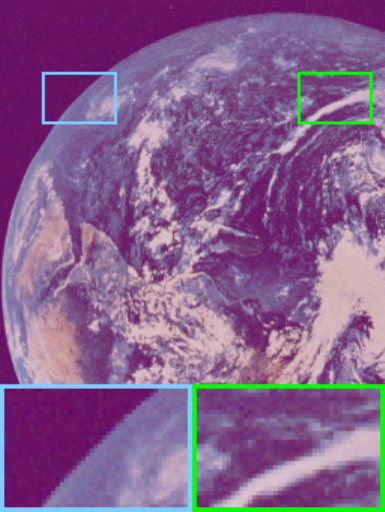}
			&\includegraphics[width=0.14\textwidth]{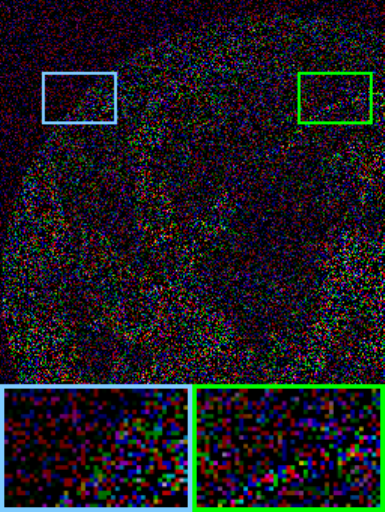}
			&\includegraphics[width=0.14\textwidth]{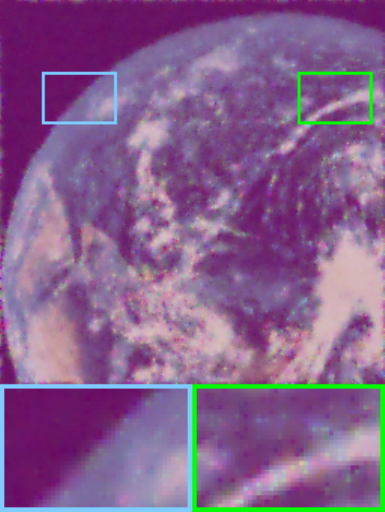}
			&\includegraphics[width=0.14\textwidth]{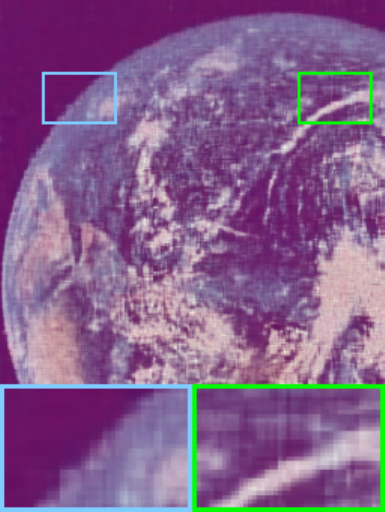}
			&\includegraphics[width=0.14\textwidth]{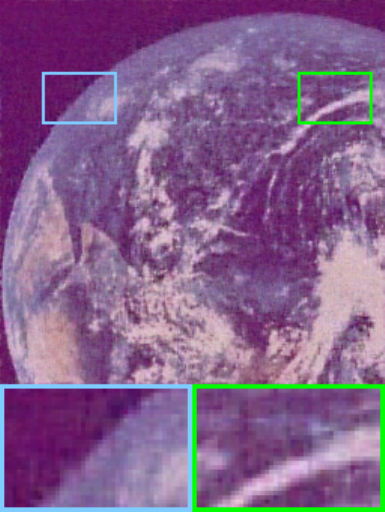}
			&\includegraphics[width=0.14\textwidth]{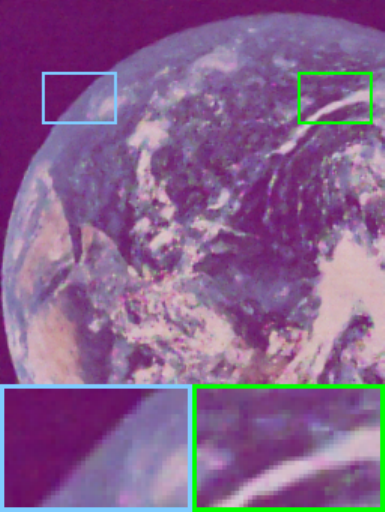}\\	
			PSNR Inf&PSNR 6.70  &  
			
			PSNR 25.31&
			PSNR 26.78  &
			PSNR 27.14 &
			PSNR 27.18\\
			Original&Noisy&STTC&HLRTF&TCTV&${\rm NeurTV}_\theta^\alpha$\\
		\end{tabular}
		\vspace{-0.2cm}
	\end{center}
	\caption{The image inpainting results by different methods on ``Plane'' and ``Earth'' with sampling rate 0.2. Here, STTC and HLRTF are TV-based methods, and TCTV is a tensor correlated TV-based method. \label{fig_image_inpainting}}
	\vspace{-0.2cm}
\end{figure}
\begin{table}[t]
	\caption{The quantitative results by different methods for image inpainting.\label{tab_inpainting}}\vspace{-0.2cm}
	\begin{center}
		\tiny
		\setlength{\tabcolsep}{6pt}
		\begin{tabular}{clcccccccccc}
			\toprule
			\multirow{2}*{Sampling rate}&\multirow{2}*{Method}&\multicolumn{2}{c}{``Peppers''}&\multicolumn{2}{c}{``Boat''}&\multicolumn{2}{c}{``House''}
			&\multicolumn{2}{c}{``Plane''}&\multicolumn{2}{c}{``Earth''}\\
			~&~&PSNR&SSIM&PSNR&SSIM&PSNR&SSIM&PSNR&SSIM&PSNR&SSIM\\
			\midrule
			\multirow{5}*{0.1}&Observed&7.02 & 0.037 &
			5.63 & 0.023 &
			4.27 & 0.014 &
			3.14 & 0.009 &
			6.18 & 0.017 \\
			~&STTC&21.75 & 0.667 &
			20.43 & 0.595 &
			20.75 & 0.630 &
			22.40 & 0.720 &
			23.71 & 0.605 \\
			~&HLRTF&22.48 & 0.642 &
			21.10 & 0.575 &
			22.01 & 0.654 &
			23.31 & 0.735 &
			24.12 & 0.604 \\
			~&TCTV&21.39 & 0.569 &
			20.70 & 0.589 &
			21.89 & 0.662 &
			23.22 & 0.725 &
			24.64 & 0.654 \\
			~&${\rm NeurTV}_\theta^\alpha$&\bf23.74 & \bf0.725 &
			\bf21.77 & \bf0.650 &
			\bf22.24 & \bf0.655 &
			\bf24.10 & \bf0.760 &
			\bf24.94 & \bf0.661 \\
			\midrule
			\multirow{5}*{0.2}&Observed&7.53 & 0.063 &
			6.12 & 0.05 &
			4.79 & 0.031 &
			3.65 & 0.019 &
			6.70 & 0.032 \\
			~&STTC&23.45 & 0.758 &
			21.86 & 0.693 &
			22.19 & 0.721 &
			23.66 & 0.783 &
			25.31 & 0.708 \\
			~&HLRTF&24.04 & 0.725 &
			22.65 & 0.675 &
			23.35 & 0.716 &
			25.43 & 0.816 &
			26.78 & 0.755 \\
			~&TCTV&23.94 & 0.698 &
			22.72 & 0.698 &
			\bf24.20 & 0.767 &
			25.65 & 0.818 &
			27.14 & \bf0.766 \\
			~&${\rm NeurTV}_\theta^\alpha$&\bf25.59 & \bf0.793 &
			\bf23.44 & \bf0.747 &
			24.11 & \bf0.770 &
			\bf26.13 & \bf0.840 &
			\bf27.18 & \bf0.766 \\
			\bottomrule
		\end{tabular}
	\end{center}
	\vspace{-0.2cm}
\end{table}    
The results of image denoising are shown in Table \ref{tab_denoising} and Fig. \ref{fig_image_denoising}. It can be observed that NeurTV considerably outperforms TV-based methods including the classical TV, DIPWTV, and HDTV. Moreover, the competing TV-based methods are specifically designed for image denoising, while our NeurTV method can be applied for diverse imaging applications on and beyond meshgrid. Hence, it is rational to say that our method has a wider applicability for different imaging applications.
\subsection{Image Inpainting (on Meshgrid)} For image inpainting, the observed image ${\cal O}\in{\mathbb R}^{n_1\times n_2\times 3}$ is incompleted and the index set of observed pixels is denoted by $\Omega$. The goal is to infer the values of the image on the complementary set of $\Omega$. {We consider the following optimization model based on the derivative-based space-variant NeurTV for image inpainting:}\begin{equation}\small\label{inpainting_model}
		\min_{\Theta}\sum_{(i,j,k)\in\Omega}({\cal O}_{(i,j,k)}-f_\Theta(i,j,k))^2+\lambda\Psi_{{\rm NeurTV}_{\theta}^\alpha}(\Theta),
	\end{equation}\noindent
	where $\Psi_{{\rm NeurTV}_{\theta}^\alpha}(\Theta)$ denotes the space-variant NeurTV defined in \eqref{WGV}. The scale parameter $\alpha$ and the directional parameters $\theta,a$ in the space-variant NeurTV regularization are determined through \eqref{WTV_para2} \& \eqref{DTV_para}. Here, we impose the space-variant NeurTV on the two spatial dimensions of the color image. For the third dimension (i.e., the channel dimension), we do not impose the NeurTV constraint. We use the tensor factorization-based DNN \cite{TPAMI_Luo} (see Fig. \ref{fig_INR} for details) as the function $f_\Theta:{\mathbb R}^3\rightarrow{\mathbb R}$ for continuous representation in model \eqref{inpainting_model}. We use a higher-resolution meshgrid (three times larger than that of the image) to define the space-variant NeurTV in model \eqref{inpainting_model}. The overall algorithm of tackling the model \eqref{inpainting_model} is similar to Algorithm \ref{alg}. We fix the trade-off parameter $\lambda$ as $3.5\times10^{-5}$ for all datasets. {The width and depth of the DNN are set to 300 and 3, respectively.} We use three baselines for image inpainting, including the tensor tree decomposition and TV-based method STTC \cite{STTC}, the deep tensor decomposition and TV-based method HLRTF \cite{HLRTF}, and the tensor correlated TV-based method TCTV \cite{TCTV}. {Here, STTC and TCTV are model-based methods. HLRTF is a DNN-based method which uses the observed data to train its DNN in an unsupervised manner.} We use five color images for testing, including ``Peppers'', ``Boat'', ``House'', ``Plane'', and ``Earth''. We consider the sampling rate 0.1 and 0.2 to produce incompleted images.\par 
	\begin{figure}[t]
	\scriptsize
	\setlength{\tabcolsep}{0.9pt}
	\begin{center}
		\begin{tabular}{cccccc}
			\includegraphics[width=0.14\textwidth]{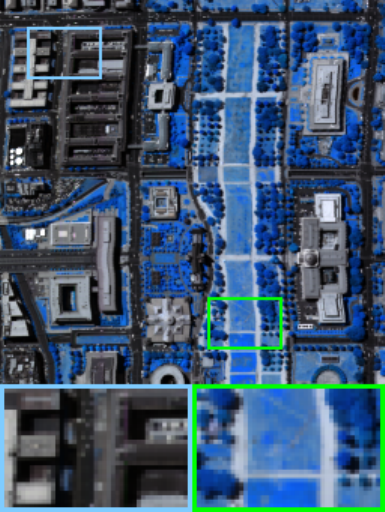}
			&\includegraphics[width=0.14\textwidth]{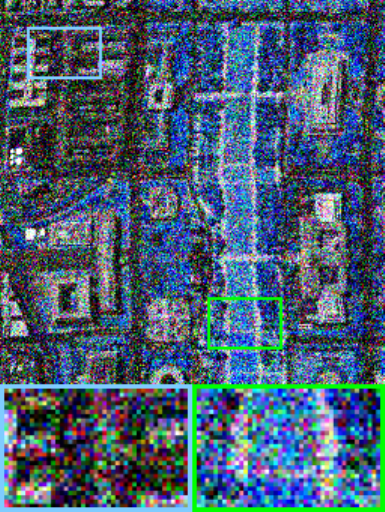}
			&\includegraphics[width=0.14\textwidth]{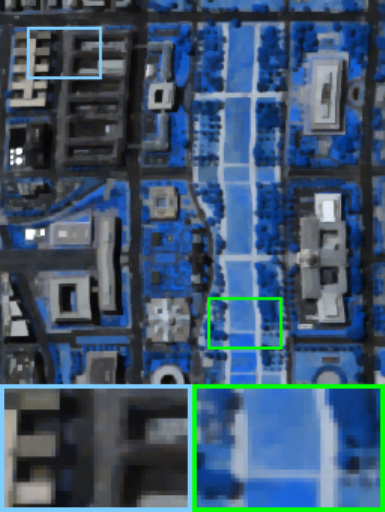}
			&\includegraphics[width=0.14\textwidth]{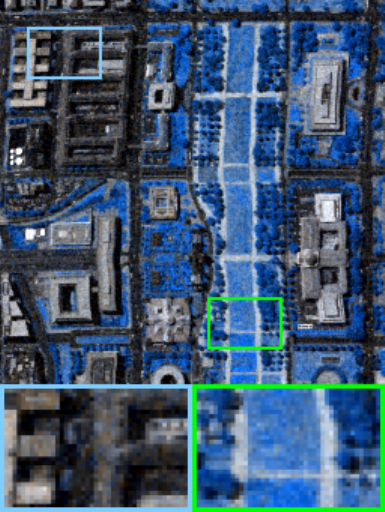}
			&\includegraphics[width=0.14\textwidth]{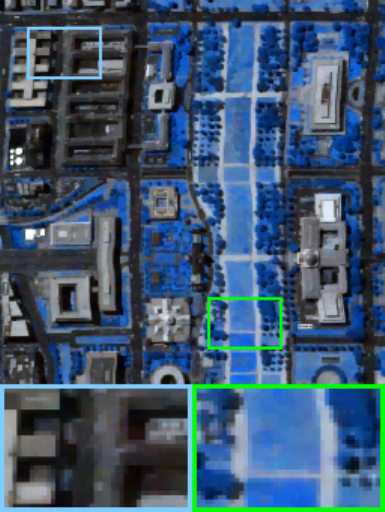}
			&\includegraphics[width=0.14\textwidth]{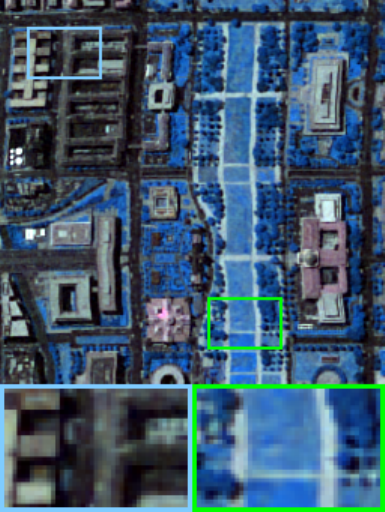}\\	
			PSNR Inf &PSNR 15.65  &
			
			PSNR 29.09 &
			PSNR 30.98  &
			PSNR 31.98  &
			PSNR 32.05  \\
			\includegraphics[width=0.14\textwidth]{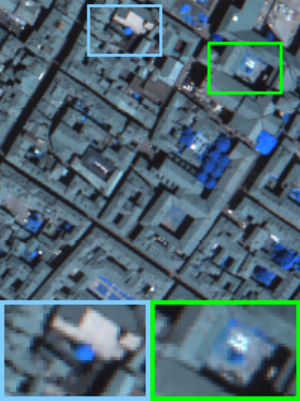}
			&\includegraphics[width=0.14\textwidth]{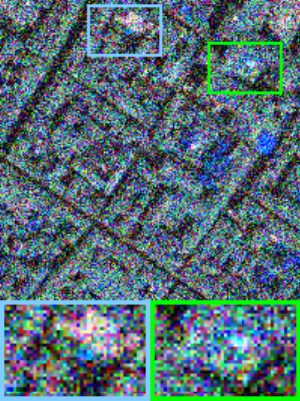}
			&\includegraphics[width=0.14\textwidth]{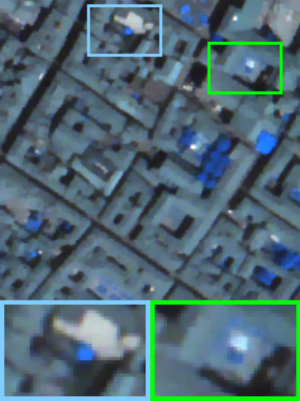}
			&\includegraphics[width=0.14\textwidth]{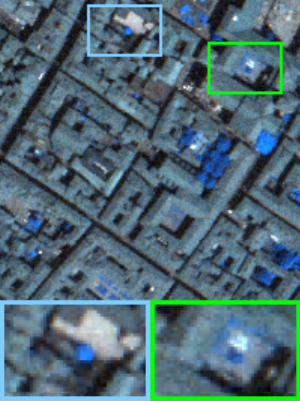}
			&\includegraphics[width=0.14\textwidth]{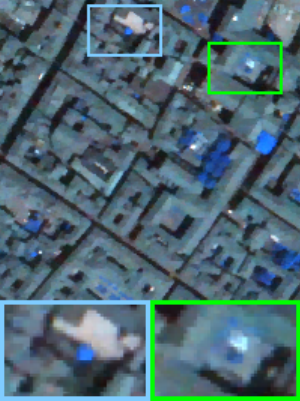}
			&\includegraphics[width=0.14\textwidth]{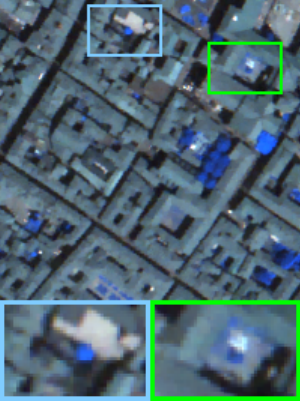}\\	
			PSNR Inf &PSNR 15.02 &
			
			PSNR 29.49  &
			PSNR 31.57 &
			PSNR 30.24 &
			PSNR 32.05  \\
			Original&Noisy&LRTDTV&LRTF-DFR&RCTV&NeurSSTV\\
		\end{tabular}
		\vspace{-0.2cm}
	\end{center}
	\caption{The results of HSI denoising by different methods on ``WDC'' and ``Pavia'' with Gaussian noise under noise deviation 0.2. Here, LRTDTV is an SSTV-based method. LRTF-DFR and RCTV are TV-based methods. \label{fig_HSI}}
	\vspace{-0.2cm}
\end{figure}
\begin{table}[t]
	\caption{The quantitative results by different methods for HSI mixed noise removal.\label{tab_HSI}}\vspace{-0.2cm}
	\begin{center}
		\tiny
		\setlength{\tabcolsep}{6pt}
		\begin{tabular}{clcccccccccc}
			\toprule
			\multirow{2}*{Noise level}&\multirow{2}*{Method}&\multicolumn{2}{c}{``WDC''}&\multicolumn{2}{c}{``Pavia''}&\multicolumn{2}{c}{``Cloth''}
			&\multicolumn{2}{c}{``Cups''}&\multicolumn{2}{c}{``Fruits''}\\
			~&~&PSNR&SSIM&PSNR&SSIM&PSNR&SSIM&PSNR&SSIM&PSNR&SSIM\\
			\midrule
			\multirow{5}*{\tabincell{c}{Gaussian 0.2}}&Observed&15.65 & 0.249 &
			15.02 & 0.223 &
			19.35 & 0.285 &
			14.77 & 0.075 &
			21.89 & 0.088 \\
			~&LRTDTV&
			29.09 & 0.826 &
			29.49 & 0.850 &
			30.36 & 0.776 &
			33.50 & 0.902 &
			38.77 & 0.843 \\
			~&LRTF-DFR&
			30.98 & 0.887 &
			31.57 & 0.910 &
			31.52 & 0.804 &
			34.30 & 0.899 &
			38.51 & 0.804 \\
			~&RCTV&
			31.98 & \bf0.917 &
			30.24 & 0.883 &
			29.73 & 0.739 &
			30.75 & 0.719 &
			36.54 & 0.718 \\
			~&NeurSSTV&
			\bf32.05 & \bf0.917 &
			\bf32.05 & \bf0.916 &
			\bf31.64 & \bf0.828 &
			\bf36.91 & \bf0.960 &
			\bf41.18 & \bf0.926 \\
			\midrule
			\multirow{5}*{\tabincell{c}{Gaussian 0.2\\Impulse 0.1}}&Observed&
			12.48 & 0.153 &
			12.57 & 0.142 &
			16.60 & 0.188 &
			12.57 & 0.049 &
			17.72 & 0.048 \\
			~&LRTDTV&
			27.92 & 0.790 &
			28.46 & 0.821 &
			29.36 & 0.733 &
			32.05 & 0.856 &
			35.92 & 0.686 \\
			~&LRTF-DFR&
			30.86 & 0.884 &
			31.12 & 0.903 &
			31.19 & 0.795 &
			33.32 & 0.855 &
			37.86 & 0.785 \\
			~&RCTV&
			30.84 & 0.897 &
			29.52 & 0.862 &
			29.13 & 0.704 &
			30.14 & 0.697 &
			34.77 & 0.631 \\
			~&NeurSSTV&
			\bf31.63 & \bf0.907 &
			\bf31.59 & \bf0.912 &
			\bf31.39 & \bf0.808 &
			\bf34.87 & \bf0.909 &
			\bf39.76 & \bf0.844 \\
			\bottomrule
		\end{tabular}
	\end{center}
\end{table}    
	The results of image inpainting are shown in Table \ref{tab_inpainting} and Fig. \ref{fig_image_inpainting}. It is seen that NeurTV generally has better performances than other TV-based methods, including STTC, HLRTF, and TCTV. Specifically, NeurTV quantitatively outperforms other competing methods in most cases, and could also obtain more pleasant visual results for image inpainting as shown in Fig. \ref{fig_image_inpainting}. These results further validate the effectiveness of our NeurTV on meshgrid.
	\subsection{Hyperspectral Image Mixed Noise Removal (on Meshgrid)}
	Then, we consider the HSI mixed noise removal \cite{LRTF-DFR,LRTDTV}. We assume that the observed HSI ${\cal O}\in{\mathbb R}^{n_1\times n_2\times n_3}$ is corrupted by Gaussian and impulse noise. {We then consider the following optimization model based on the proposed second-order derivative-based NeurSSTV \eqref{SSNeurTV} to capture spatial-spectral local smoothness for HSI mixed noise removal:}
	\begin{equation}\small\label{HSI_model}
		\min_{\Theta,{\cal S}}\sum_{(i,j,k)\in{\rm meshgrid}}({\cal O}_{(i,j,k)}-f_\Theta(i,j,k)-{\cal S}_{(i,j,k)})^2+\lambda\Psi_{\rm NeurSSTV}(\Theta)+\gamma\lVert{\cal S}\rVert_{\ell_1},
	\end{equation}\noindent
	where $\Psi_{\rm NeurSSTV}(\Theta)$ denotes the NeurSSTV defined in \eqref{SSNeurTV} and $\cal S$ denotes the sparse noise to be estimated. We use the tensor factorization-based DNN \cite{TPAMI_Luo} (see Fig. \ref{fig_INR}) as the continuous representation $f_\Theta:{\mathbb R}^3\rightarrow{\mathbb R}$. We use a higher-resolution meshgrid (three times larger than the resolution of HSI) in all three dimensions to define the NeurSSTV regularization. To tackle model \eqref{inpainting_model}, we tackle the following sub-problems alternately: 
	\begin{equation}\small
		\begin{split}
			&\min_{\Theta}\sum_{(i,j,k)\in{\rm meshgrid}}({\cal O}_{(i,j,k)}-f_\Theta(i,j,k)-{\cal S}_{(i,j,k)})^2+\lambda\Psi_{\rm NeurSSTV}(\Theta),\\
			&\min_{{\cal S}}\sum_{(i,j,k)\in{\rm meshgrid}}({\cal O}_{(i,j,k)}-f_\Theta(i,j,k)-{\cal S}_{(i,j,k)})^2+\gamma\lVert{\cal S}\rVert_{\ell_1},
		\end{split}
	\end{equation}\noindent
where the first problem is tackled by the Adam algorithm \cite{Adam} and the second one can be tackled by the soft thresholding method. 
\begin{figure}[t]
	\scriptsize
	\setlength{\tabcolsep}{0.9pt}
	\begin{center}
		\begin{tabular}{cccccc}
			\includegraphics[width=0.14\textwidth]{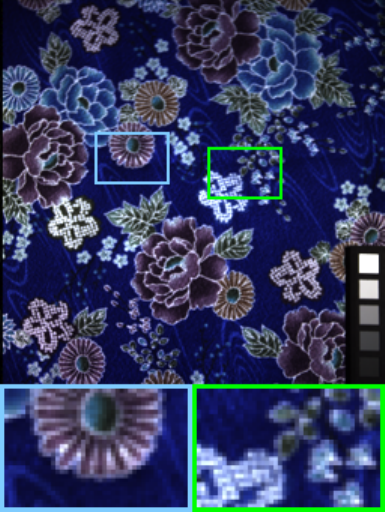}
			&\includegraphics[width=0.14\textwidth]{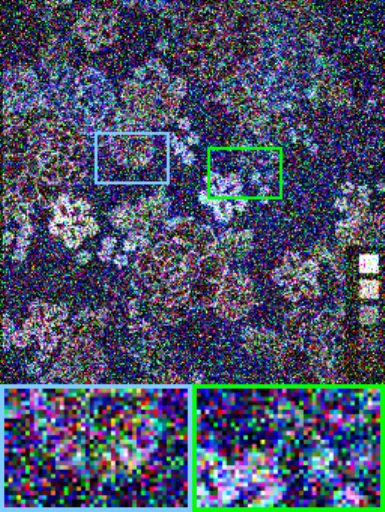}
			&\includegraphics[width=0.14\textwidth]{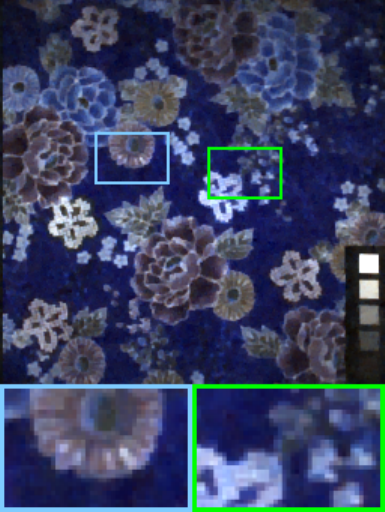}
			&\includegraphics[width=0.14\textwidth]{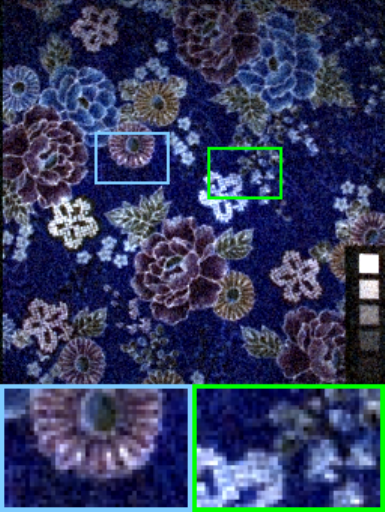}
			&\includegraphics[width=0.14\textwidth]{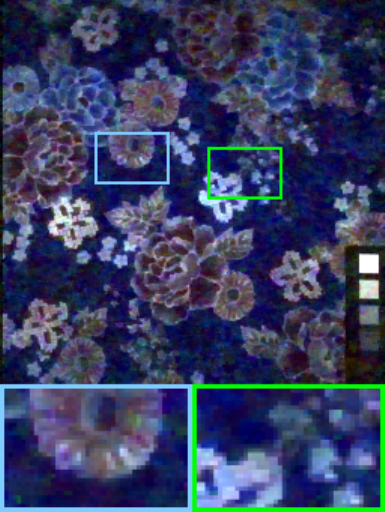}
			&\includegraphics[width=0.14\textwidth]{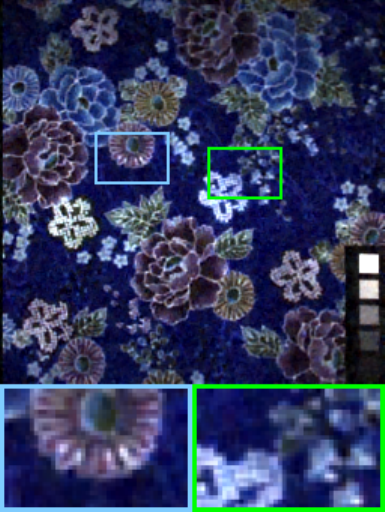}\\	
			PSNR Inf&PSNR 16.60  &
			
			PSNR 29.36  &
			PSNR 31.19 &
			PSNR 29.13  &
			PSNR 31.39 \\
			\includegraphics[width=0.14\textwidth]{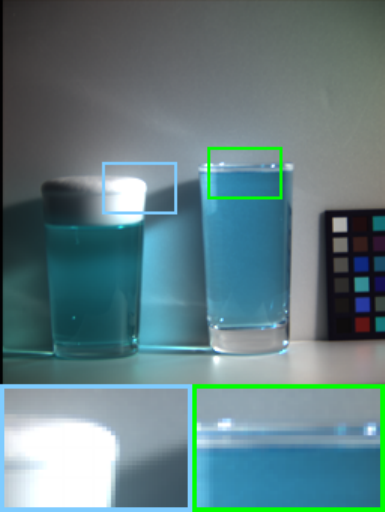}
			&\includegraphics[width=0.14\textwidth]{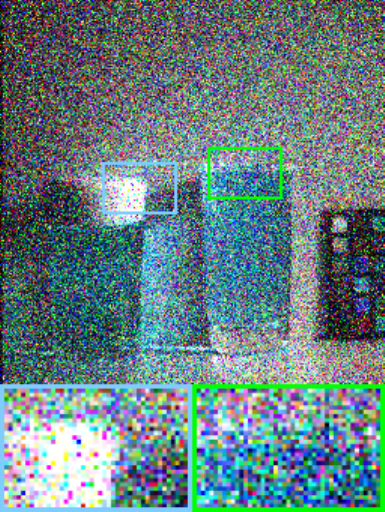}
			&\includegraphics[width=0.14\textwidth]{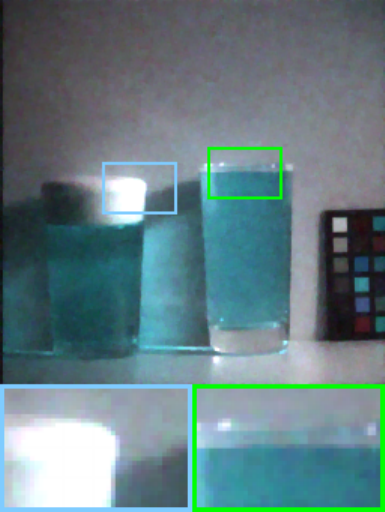}
			&\includegraphics[width=0.14\textwidth]{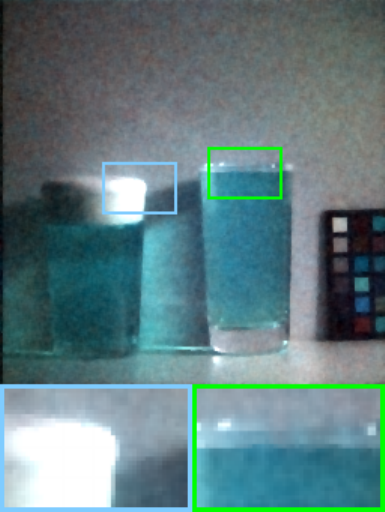}
			&\includegraphics[width=0.14\textwidth]{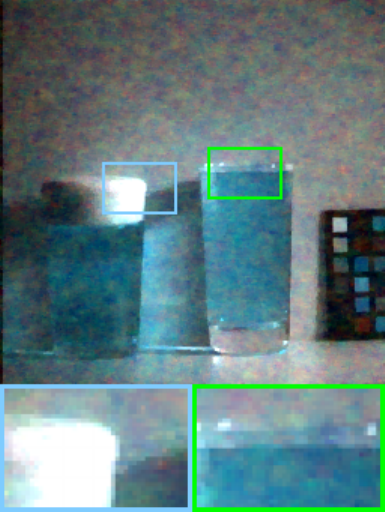}
			&\includegraphics[width=0.14\textwidth]{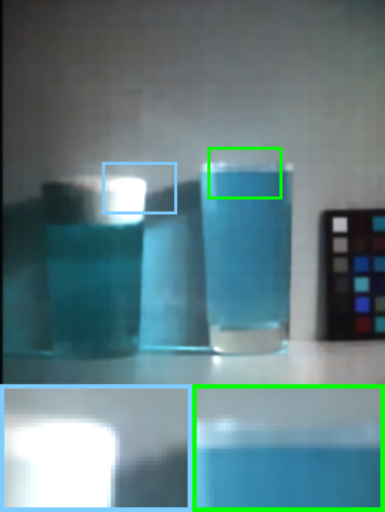}\\				
			PSNR Inf&PSNR 12.57  &
			
			PSNR 32.05 &
			PSNR 33.32 &
			PSNR 30.14 &
			PSNR 34.87  \\
			\includegraphics[width=0.14\textwidth]{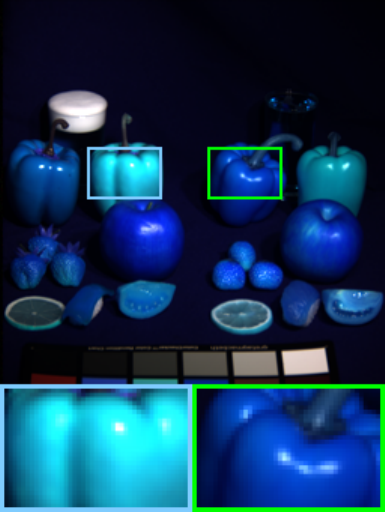}
			&\includegraphics[width=0.14\textwidth]{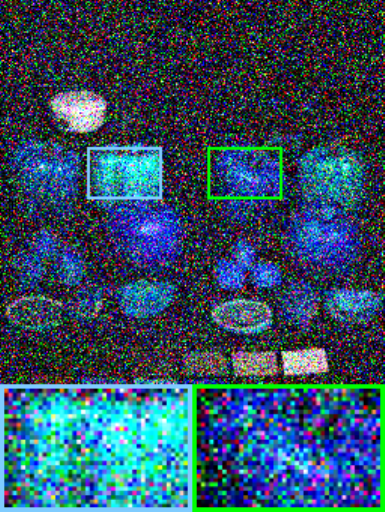}
			&\includegraphics[width=0.14\textwidth]{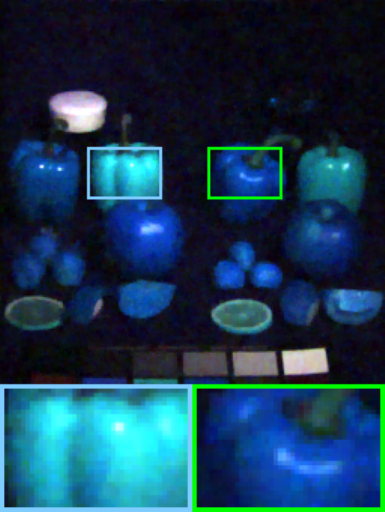}
			&\includegraphics[width=0.14\textwidth]{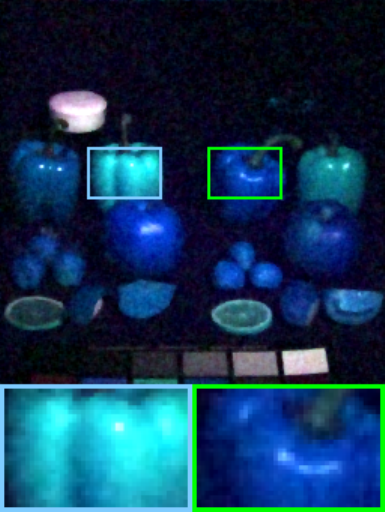}
			&\includegraphics[width=0.14\textwidth]{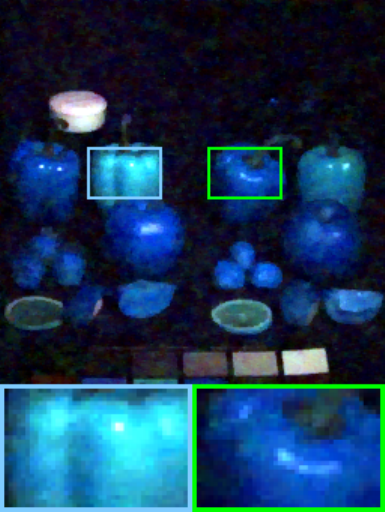}
			&\includegraphics[width=0.14\textwidth]{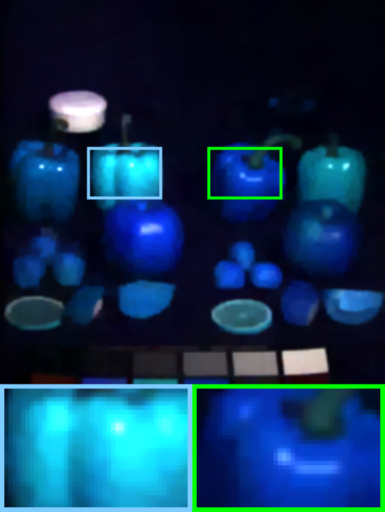}\\		
			PSNR Inf&PSNR 17.72  &
			
			PSNR 35.92  &
			PSNR 37.86 &
			PSNR 34.77 &
			PSNR 39.76  \\
			Noisy&Original&LRTDTV&LRTF-DFR&RCTV&NeurSSTV\\
		\end{tabular}
		\vspace{-0.2cm}
	\end{center}
	\caption{The results of HSI mixed noise removal by different methods on ``Cloth'', ``Cups'', and ``Fruits'' with mixed Gaussian noise (noise deviation 0.2) and impulse noise (sampling rate 0.1). Here, LRTDTV is an SSTV-based method. LRTF-DFR and RCTV are TV-based methods.\label{fig_HSI2}}
	\vspace{-0.2cm}
\end{figure}
We set the trade-off parameters $\lambda=3\times10^{-4}$ and $\gamma=0.25$ for cases with both Gaussian and impulse noise. {The width and depth of the DNN are set to 300 and 3, respectively.} For cases with only Gaussian noise, we set $\lambda=3\times10^{-4}$ and $\gamma=10$. We use three baselines for HSI mixed noise removal, including the tensor decomposition and SSTV-based method LRTDTV \cite{LRTDTV}, the double factor TV-based method LRTF-DFR \cite{LRTF-DFR}, and the representative coefficient TV-based method RCTV \cite{RCTV}. {Here, all competing methods are model-based methods. Especially, LRTDTV is an SSTV-based method. LRTF-DFR and RCTV are TV-based methods.} We use five HSIs to test these methods, including ``WDC'', ``Pavia''\footnote{\url{https://engineering.purdue.edu/~biehl/MultiSpec/hyperspectral.html}}, and three HSIs from the CAVE dataset \cite{cave}, named ``Cloth'', ``Cups'', and ``Fruits''. We consider two noisy cases. The first case contains Gaussian noise with noise deviation 0.2. The second case contains mixed Gaussian noise with deviation 0.2 and sparse impulse noise with sampling rate 0.1.\par 
The results of HSI mixed noise removal are shown in Table \ref{tab_HSI} and Figs. \ref{fig_HSI} \& \ref{fig_HSI2}. It can be observed that NeurSSTV consistently outperforms other TV-based methods LRTDTV, LRTF-DFR, and RCTV. According to Fig. \ref{fig_HSI}, NeurSSTV could obtain better visual results than other competing methods. Specifically, NeurSSTV could attenuate noise well and also preserve the details of the HSI better. 
\begin{figure}[t]
	\scriptsize
	\setlength{\tabcolsep}{0.9pt}
	\begin{center}
		\begin{tabular}{cccccccc}
			\vspace{-0.2cm}
			\includegraphics[width=0.115\textwidth]{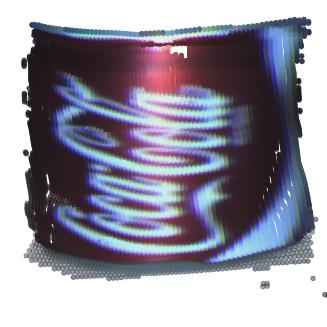}
			&\includegraphics[width=0.115\textwidth]{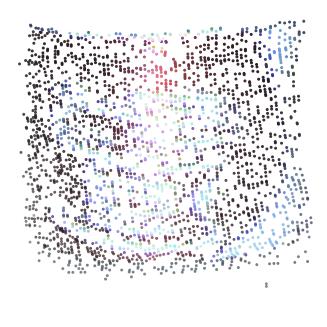}
			&\includegraphics[width=0.115\textwidth]{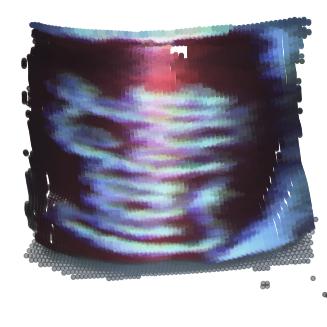}
			&\includegraphics[width=0.115\textwidth]{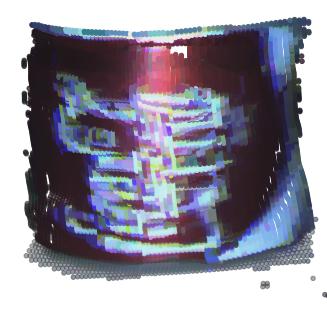}
			&\includegraphics[width=0.115\textwidth]{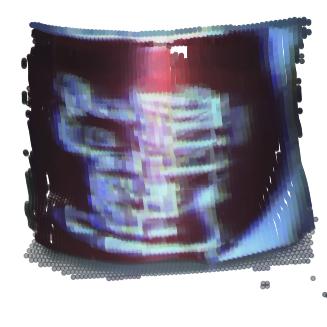}
			&\includegraphics[width=0.115\textwidth]{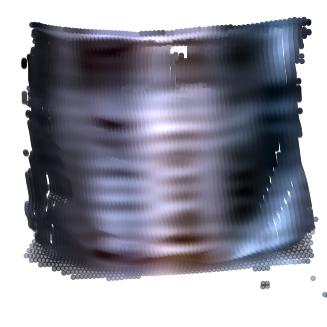}
			&\includegraphics[width=0.115\textwidth]{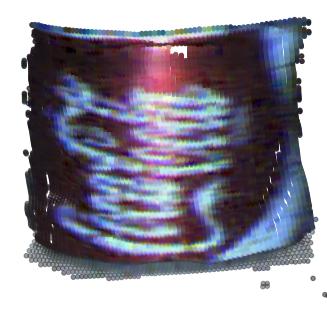}
			&\includegraphics[width=0.115\textwidth]{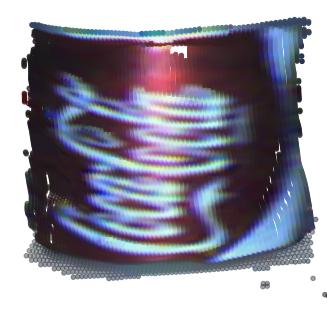}\\	
			\vspace{0.2cm}
			MSE 0&
			MSE \--\--&
			MSE 0.199 &
			MSE 0.247 &
			MSE 0.197 &
MSE 0.324 &
MSE 0.173 &
			MSE 0.118 \\
			\includegraphics[width=0.115\textwidth]{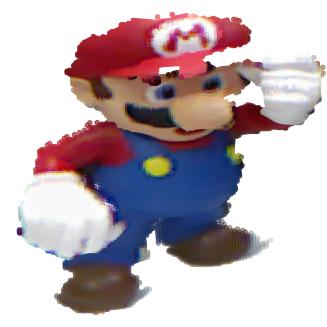}
			&\includegraphics[width=0.115\textwidth]{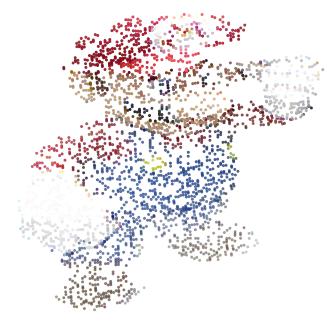}
			&\includegraphics[width=0.115\textwidth]{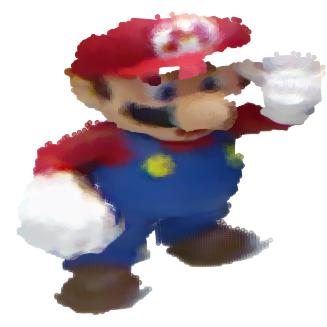}
			&\includegraphics[width=0.115\textwidth]{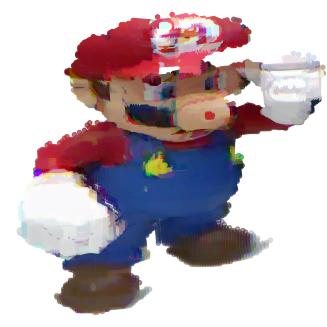}
			&\includegraphics[width=0.115\textwidth]{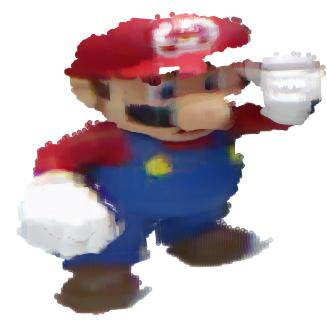}
			&\includegraphics[width=0.115\textwidth]{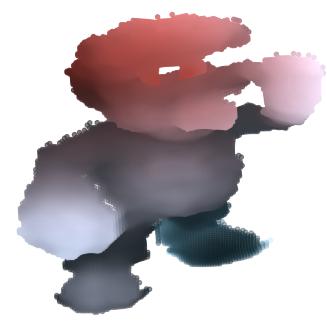}
						&\includegraphics[width=0.115\textwidth]{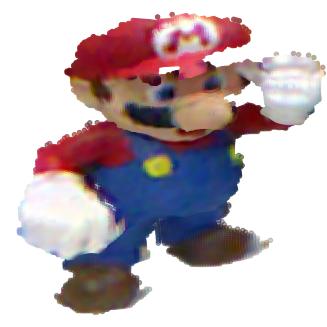}
			&\includegraphics[width=0.115\textwidth]{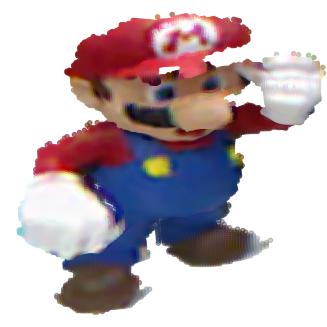}\\	
			MSE 0&MSE \--\--&
			MSE 0.123 &
			MSE 0.161 &
			MSE 0.120 &
MSE 0.373 &
MSE 0.106 &
			MSE 0.090 \\
			Original&Observed&KNR&DT&RF&FSA-HTF&SIREN&${\rm NeurTV}^\alpha$\\
		\end{tabular}
		\vspace{-0.2cm}
	\end{center}
	\caption{The results of point cloud recovery by different methods on ``Cola'' and ``Mario'' with sampling rate 0.2.  \label{fig_point}}
	\vspace{-0.2cm}
\end{figure} 
\begin{table}[t]
	\caption{The quantitative results by different methods for point cloud recovery.\label{tab_point}}\vspace{-0.2cm}
	\begin{center}
		\tiny
\setlength{\tabcolsep}{4pt}
		\begin{tabular}{clcccccccccc}
			\toprule
			\multirow{2}*{Sampling rate}&\multirow{2}*{Method}&\multicolumn{2}{c}{``Cola''}&\multicolumn{2}{c}{``Mario''}&\multicolumn{2}{c}{``Duck''}&\multicolumn{2}{c}{``Squirrel''}&\multicolumn{2}{c}{``Rabbit''}\\
			~&~&MSE&R-Square&MSE&R-Square&MSE&R-Square&MSE&R-Square&MSE&R-Square\\
			\midrule
			\multirow{6}*{0.1}
			&KNR&0.264&0.787&0.162&0.912&0.108&0.891&0.153&0.820&0.135&0.841\\
			~&DT&0.332&0.692&0.211&0.856&0.133&0.841&0.164&0.802&0.170&0.766\\
			~&RF&0.257&0.799&0.152&0.922&0.098&0.910&0.140&0.851&0.130&0.853\\
			~&FSA-HTF&0.467&0.447&0.376&0.530&0.259&0.413&0.259&0.486&0.269&0.371\\
			~&SIREN&0.209&0.867&0.138&0.936&0.097&0.911&0.129&0.874&0.119&0.878\\
			~&${\rm NeurTV}^\alpha$&\bf0.191&\bf0.889&\bf0.126&\bf0.947&\bf0.087&\bf0.929&\bf0.126&\bf0.878&\bf0.108&\bf0.899\\
			\midrule
			\multirow{6}*{0.2}
			&KNR&0.199&0.879&0.123&0.949&0.082&0.937&0.114&0.901&0.104&0.906\\
			~&DT&0.247&0.821&0.161&0.915&0.095&0.917&0.133&0.868&0.126&0.867\\
			~&RF&0.197&0.882&0.120&0.952&0.071&0.954&0.105&0.916&0.099&0.915 \\
			~&FSA-HTF&0.324&0.679&0.373&0.531&0.250&0.417&0.245&0.541&0.266&0.385\\
			~&SIREN&0.173&0.914&0.106&0.962&0.078&0.944&0.103&0.919&0.082&0.942\\
			~&${\rm NeurTV}^\alpha$&\bf0.118&\bf0.957&\bf0.090&\bf0.973&\bf0.067&\bf0.958&\bf0.085&\bf0.945&\bf0.079&\bf0.945\\
			\bottomrule
		\end{tabular}
	\end{center}
\end{table}   
\begin{figure}[t]
	\scriptsize
	\setlength{\tabcolsep}{0.9pt}
	\begin{center}
		\begin{tabular}{cccccccc}
			\vspace{-0.1cm}
			\includegraphics[width=0.115\textwidth]{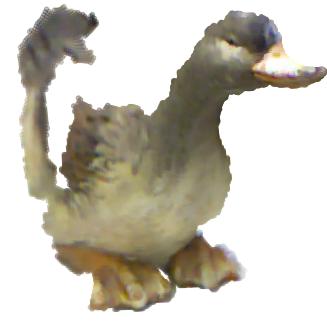}
			&\includegraphics[width=0.115\textwidth]{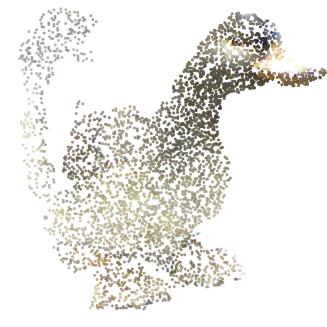}
			&\includegraphics[width=0.115\textwidth]{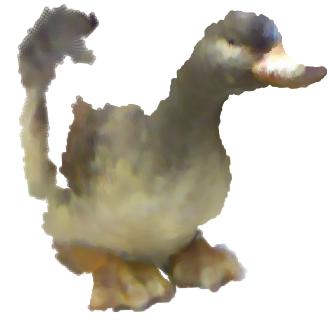}
			&\includegraphics[width=0.115\textwidth]{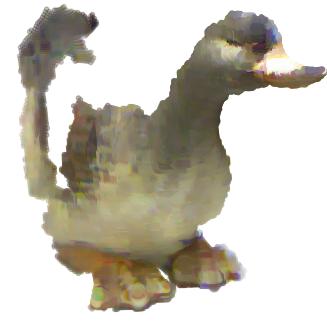}
			&\includegraphics[width=0.115\textwidth]{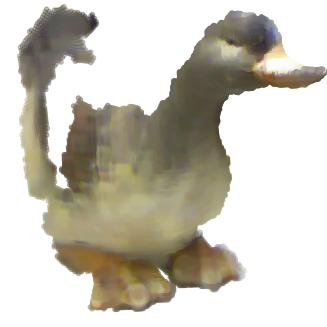}
					&\includegraphics[width=0.115\textwidth]{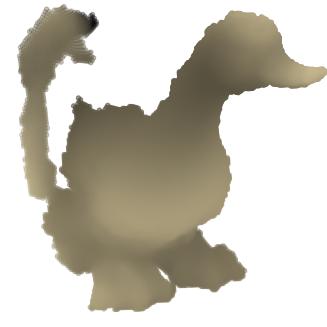}
					&\includegraphics[width=0.115\textwidth]{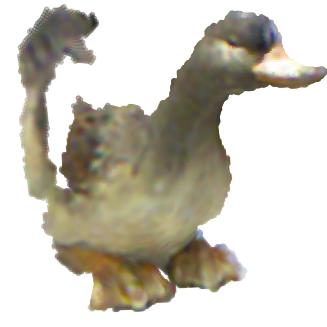}
			&\includegraphics[width=0.115\textwidth]{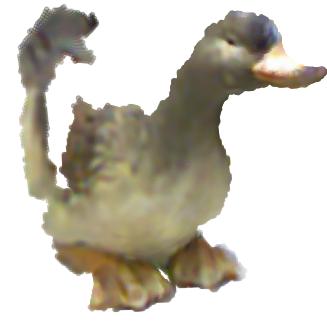}\\	
			MSE 0 &MSE \--\--&	
			
			MSE 0.082 &
			MSE 0.095 &
			MSE 0.071 &
MSE 0.250 &
MSE 0.078 &
			MSE 0.067 \\
			\vspace{-0.1cm}
			\includegraphics[width=0.115\textwidth]{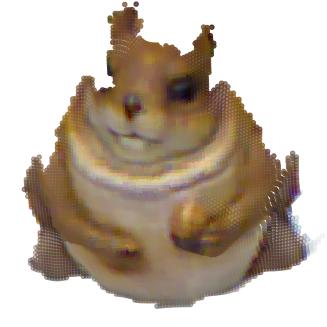}
			&\includegraphics[width=0.115\textwidth]{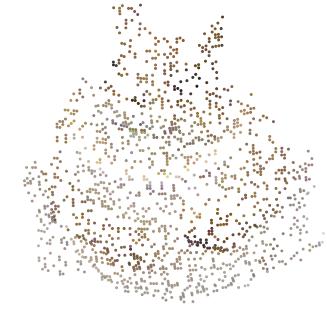}
			&\includegraphics[width=0.115\textwidth]{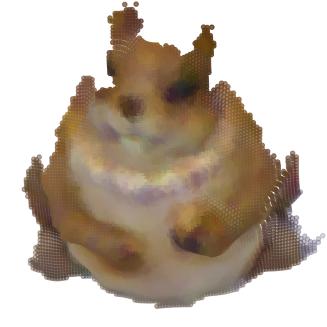}
			&\includegraphics[width=0.115\textwidth]{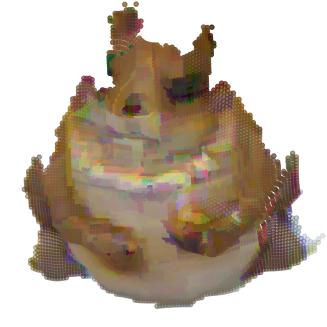}
			&\includegraphics[width=0.115\textwidth]{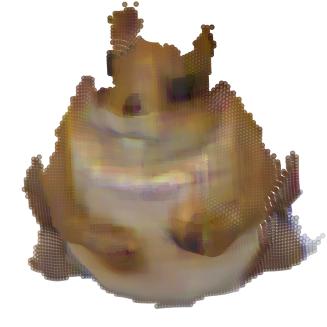}
			&\includegraphics[width=0.115\textwidth]{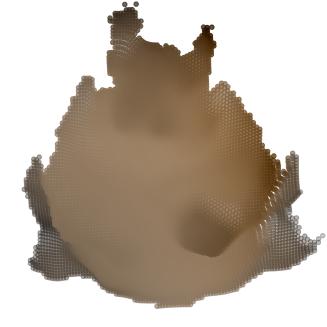}
			&\includegraphics[width=0.115\textwidth]{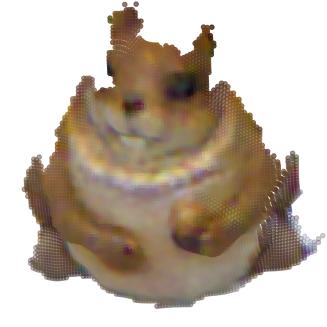}
			&\includegraphics[width=0.115\textwidth]{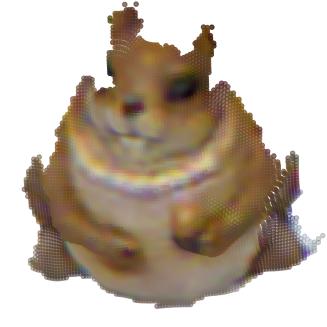}\\	
			MSE 0 &MSE \--\--&	
			
			MSE 0.114 &
			MSE 0.133 &
			MSE 0.105 &
			MSE 0.245 &
			MSE 0.103 &
			MSE 0.085 \\
			\vspace{-0.1cm}
			\includegraphics[width=0.115\textwidth]{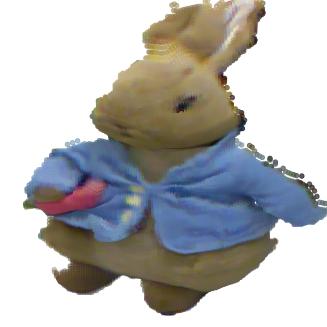}
			&\includegraphics[width=0.115\textwidth]{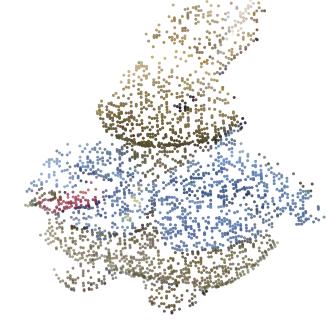}
			&\includegraphics[width=0.115\textwidth]{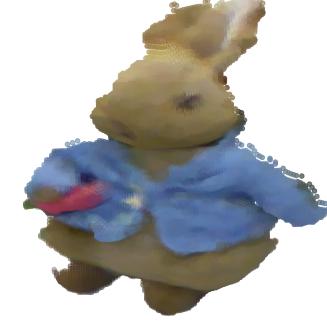}
			&\includegraphics[width=0.115\textwidth]{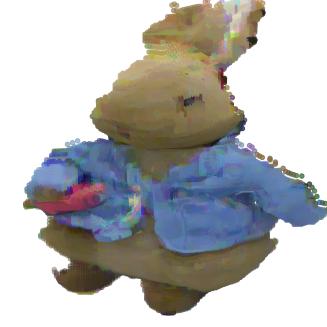}
			&\includegraphics[width=0.115\textwidth]{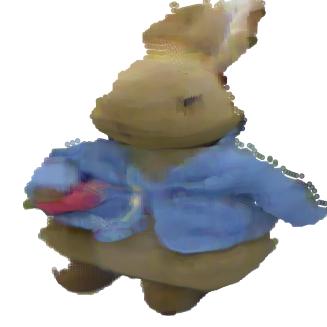}
			&\includegraphics[width=0.115\textwidth]{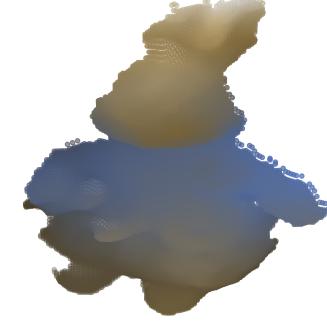}
						&\includegraphics[width=0.115\textwidth]{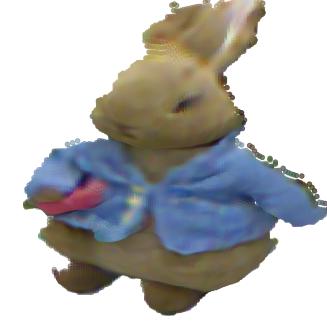}
			&\includegraphics[width=0.115\textwidth]{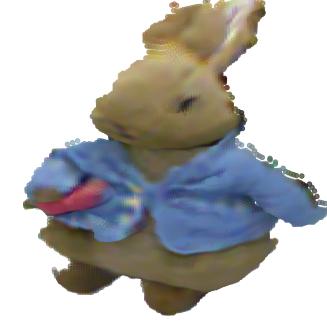}\\	
			MSE 0 &MSE \--\--&
			MSE 0.104 &
			MSE 0.126 &
			MSE 0.099 &
			MSE 0.266 &
			MSE 0.082 &
			MSE 0.079 \\
			Original&Observed&KNR&DT&RF&FSA-HTF&SIREN&${\rm NeurTV}^\alpha$\\
		\end{tabular}
		\vspace{-0.2cm}
	\end{center}
	\caption{The results of point cloud recovery by different methods on ``Duck'', ``Squirrel'', and ``Rabbit'' with sampling rate 0.2.  \label{fig_point2}}
	\vspace{-0.2cm}
\end{figure}
\subsection{Point Cloud Recovery (Beyond Meshgrid)}
NeurTV is suitable for both meshgrid and non-meshgrid data attributed to the continuous representation, while traditional TV is not suitable for non-meshgrid data (e.g., point cloud). To validate this advantage, we consider the point cloud recovery problem, which aims to estimate the color information of all points of the point cloud by giving the color information of some partially observed points. We assume that the observed point cloud with $n$ points is represented by an $n$-by-5 matrix ${\bf O}\in{\mathbb R}^{n\times 5}$, where each point ${\bf O}_{(i,:)}$ ($i=1,2,\cdots,n$) is an ($x,y,z,C,v$) format five-dimensional vector, containing coordinate information (i.e., $(x,y,z)$), channel information (i.e., $C\in\{1,2,3\}$), and the corresponding color value $v$ in the channel $C$ at the coordinate $(x,y,z)$. The continuous representation $f_\Theta:{\mathbb R}^4\rightarrow{\mathbb R}$ takes the coordinate and channel vector ($x,y,z,C$) as the input and is expected to output the corresponding color information $v$. Then, the trained model $f_\Theta(\cdot)$ can estimate the color of the point cloud at any other input location. \par
	For the point cloud recovery, we propose to impose the NeurTV regularization with space-variant scale parameters $\alpha$ in the first three dimensions $(x,y,z)$ and do not impose the constraint in the channel dimension, since the local correlations along channel dimension are not significant. We consider the following derivative-based NeurTV for point cloud data:
	\begin{equation}\small\label{NeurTV_point}
		\Psi_{{\rm NeurTV}^\alpha}(\Theta)=\sum_{i=1}^{n}\alpha_{i}
		\sum_{d=1}^3\left|
		\frac{\partial f_\Theta({\bf O}_{{(i,1:4)}})}{\partial {\bf O}_{{(i,d)}}}
		\right|,
	\end{equation}\noindent
	where each scale parameter $\alpha_{i}$ is determined through the second-order derivatives, i.e., $\alpha_{i} = (\sum_{d=1}^3|\frac{\partial^2f_\Theta({\bf O}_{{(i,1:4)}})}{\partial {\bf O}^2_{{(i,d)}}}|+\epsilon)^{-1}$. Here, we do not consider the space-variant directional parameters since extending the directional NeurTV to three dimensional point cloud is not a trivial task, and we leave it to future work. {Based on the proposed derivative-based space-variant NeurTV regularization \eqref{NeurTV_point}, the optimization model of training the continuous representation $f_\Theta(\cdot)$ for point cloud recovery is formulated as:}
	\begin{equation}\small\label{point_model}
		\min_{\Theta}\sum_{i=1}^{n}({\bf O}_{{(i,5)}}-f_\Theta({\bf O}_{{(i,1:4)}}))^2+\lambda\Psi_{{\rm NeurTV}^\alpha}(\Theta).
	\end{equation}\noindent
	We use the tensor factorization-based DNN \cite{TPAMI_Luo} (see Fig. \ref{fig_INR}) as the continuous representation $f_\Theta:{\mathbb R}^4\rightarrow{\mathbb R}$. The algorithm of tackling the point cloud recovery model \eqref{point_model} is similar to Algorithm \ref{alg} by using iterative updates of the scale parameters $\alpha_{i}$ and the DNN weights $\Theta$. We set the trade-off parameter $\lambda=1\times10^{-4}$. {The width and depth of the DNN are set to 300 and 3, respectively.} {We use five methods as baselines for point cloud recovery, including the K-neighbors regressor (KNR) \cite{sklearn}, the decision tree (DT) regressor \cite{sklearn}, the random forest (RF) regressor \cite{sklearn}, the canonical decomposition-based regression method FSA-HTF \cite{TSP_CP}, and the continuous representation-based method SIREN \cite{sine}. Here, KNR, DT, RF, and FSA-HTF are model-based methods. SIREN is a DNN-based method which uses the observed data to train its DNN in an unsupervised manner.} Five point cloud datasets are used to test these methods, including ``Cola'', ``Mario'', ``Duck'', ``Squirrel'', and ``Rabbit'', which are online available\footnote{\url{http://www.vision.deis.unibo.it/research/80-shot}}. We consider the sampling rate 0.1 and 0.2 to test different methods. The results are quantitatively evaluated by normalized root MSE and R-Square.\par 
\begin{figure}[t]
	\scriptsize
	\setlength{\tabcolsep}{0.9pt}
	\begin{center}
		\begin{tabular}{ccccccc}
			\includegraphics[width=0.125\textwidth]{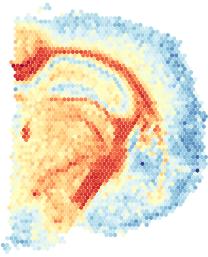}
			&\includegraphics[width=0.125\textwidth]{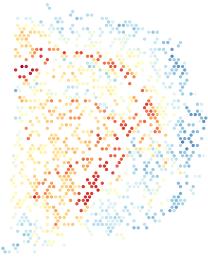}
			&\includegraphics[width=0.125\textwidth]{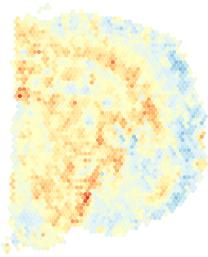}
			&\includegraphics[width=0.125\textwidth]{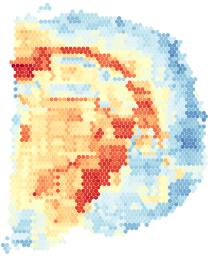}
			&\includegraphics[width=0.125\textwidth]{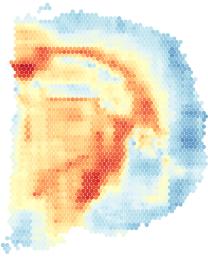}
			&\includegraphics[width=0.125\textwidth]{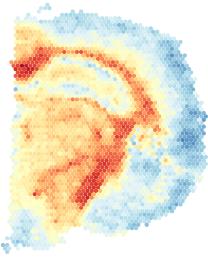}
			&\includegraphics[width=0.125\textwidth]{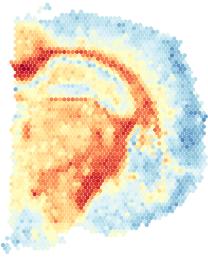}\\
			MSE 0
			&MSE \--\--&MSE 0.277
			&MSE 0.146
			&MSE 0.118
			&MSE 0.106
			&MSE 0.097
			\\
			\includegraphics[width=0.125\textwidth]{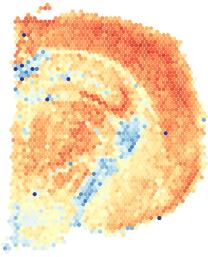}
			&\includegraphics[width=0.125\textwidth]{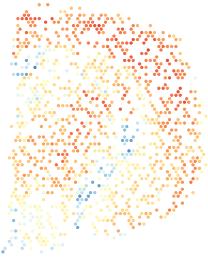}
			&\includegraphics[width=0.125\textwidth]{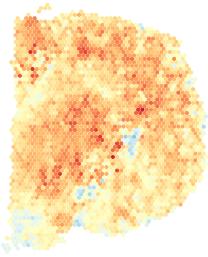}
			&\includegraphics[width=0.125\textwidth]{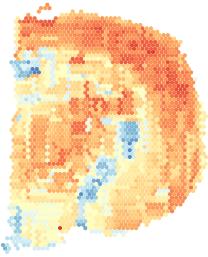}
			&\includegraphics[width=0.125\textwidth]{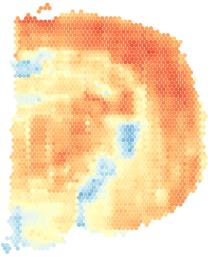}
			&\includegraphics[width=0.125\textwidth]{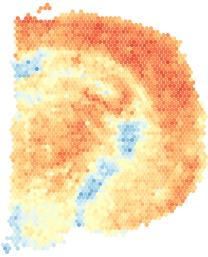}
			&\includegraphics[width=0.125\textwidth]{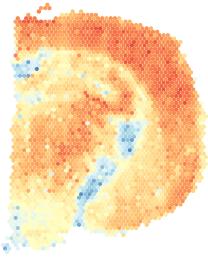}\\
			MSE 0
			&MSE \--\--&MSE 0.224
			&MSE 0.132
			&MSE 0.114
			&MSE 0.109
			&MSE 0.104
			\\
			\includegraphics[width=0.125\textwidth]{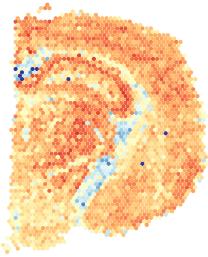}
			&\includegraphics[width=0.125\textwidth]{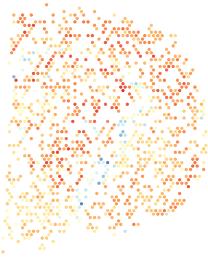}
			&\includegraphics[width=0.125\textwidth]{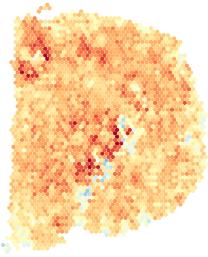}
			&\includegraphics[width=0.125\textwidth]{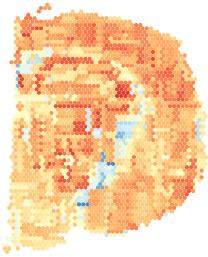}
			&\includegraphics[width=0.125\textwidth]{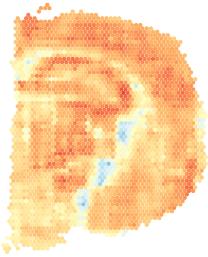}
			&\includegraphics[width=0.125\textwidth]{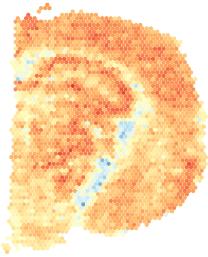}
			&\includegraphics[width=0.125\textwidth]{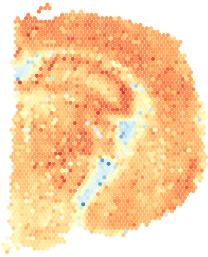}\\
			MSE 0
			&MSE \--\--
			&MSE 0.212
			&MSE 0.148
			&MSE 0.120
			&MSE 0.114
			&MSE 0.105\\
			Original&Observed&KNR&DT&RF&GNTD&${\rm NeurTV}_\theta^\alpha$\\
		\end{tabular}
		\vspace{-0.2cm}
	\end{center}
	\caption{The results of spatial transcriptomics data reconstruction by different competing methods. The three rows respectively list the results of all comparison methods on genes ``mbp'', ``snap25'', and ``atp1b1'' of the mouse brain dataset with sampling rate 0.4.\label{fig_ST}}
	\vspace{-0.2cm}
\end{figure}
\begin{table}[t]
	\caption{The quantitative results by different methods for spatial transcriptomics reconstruction on different genes of the mouse brain dataset.\label{tab_ST}}\vspace{-0.2cm}
	\begin{center}
		\tiny
		\setlength{\tabcolsep}{4pt}
		\begin{tabular}{clcccccccccc}
			\toprule
			\multirow{2}*{Sampling rate}&\multirow{2}*{Method}&\multicolumn{2}{c}{``mbp''}&\multicolumn{2}{c}{``snap25''}&\multicolumn{2}{c}{``atp1b1''}
			&\multicolumn{2}{c}{``plp1''}&\multicolumn{2}{c}{``rtn1''}\\
			~&~&MSE&R-Square&MSE&R-Square&MSE&R-Square&MSE&R-Square&MSE&R-Square\\
			\midrule
			\multirow{5}*{{0.3}}
			&KNR
			& 0.299 & 0.352
			& 0.232 & 0.125
			& 0.215 & 0.030
			& 0.417 & 0.242
			& 0.183 & 0.396\\
			~&DT
			& 0.177 & 0.737
			& 0.148 & 0.619
			& 0.171 & 0.331
			& 0.275 & 0.639
			& 0.161 & 0.553\\
			~&RF
			& 0.138 & 0.837
			& 0.118 & 0.737
			& 0.127 & 0.539
			& 0.234 & 0.728
			& 0.130 & 0.669 \\
			~&GNTD
			& 0.125 & 0.873
			& 0.117 & 0.744
			& 0.121 & 0.584
			& 0.216 & 0.768
			& 0.125 & 0.694\\
			~&${\rm NeurTV}_\theta^\alpha$
			& \bf0.113 & \bf0.890
			& \bf0.113 & \bf0.768
			& \bf0.117 & \bf0.635
			& \bf0.197 & \bf0.809
			& \bf0.120 & \bf0.735\\
			\midrule
			\multirow{5}*{{0.4}}
			&KNR 
			& 0.277 & 0.445
			& 0.224 & 0.147
			& 0.212 & 0.034
			& 0.381 & 0.382
			& 0.172 & 0.452 \\
			~&DT
			& 0.146 & 0.820
			& 0.132 & 0.686
			& 0.148 & 0.432
			& 0.242 & 0.710
			& 0.140 & 0.642\\
			~&RF
			& 0.118 & 0.880
			& 0.114 & 0.754
			& 0.120 & 0.584
			& 0.204 & 0.791
			& 0.117 & 0.724 \\
			~&GNTD
			& 0.106 & 0.903
			& 0.109 & 0.774
			& 0.114 & 0.621
			& 0.186 & 0.826
			& 0.117 & 0.726\\
			~&${\rm NeurTV}_\theta^\alpha$
			& \bf0.097 & \bf0.916
			& \bf0.104 & \bf0.796
			& \bf0.105 & \bf0.680
			& \bf0.164 & \bf0.860
			& \bf0.107 & \bf0.768\\
			\bottomrule
		\end{tabular}
	\end{center}
\end{table}    
	The results of point cloud recovery are shown in Table \ref{tab_point} and Figs. \ref{fig_point} \& \ref{fig_point2}. It can be observed that our NeurTV method quantitatively outperforms other competing methods for point cloud recovery, which validates the effectiveness of NeurTV beyond meshgrid. From the visual results in Fig. \ref{fig_point}, we can observe that our NeurTV could better recover the local structures and color information of the point clouds, which further verifies its capability to characterize the structures of point cloud data beyond meshgrid. 
	\subsection{Spatial Transcriptomics Reconstruction (Beyond Meshgrid)}
	Spatial transcriptomics data \cite{NMI_ST,GNTD} is a novel biological data that reveals informative gene expressions over a spatial area of tissues \cite{NM}. A distinct feature of spatial transcriptomics data is that it is not arranged as meshgrid data but rather placed at non-meshgrid spots in the spatial area due to the sequencing technology and the non-cubic shape of tissues, which poses challenges for classical methods like TV for characterizing it. As compared, our NeurTV benefits from the continuous representation and is suitable for the spatial transcriptomics data. Specifically, we consider the spatial transcriptomics reconstruction problem \cite{GNTD}, which aims to reconstruct the whole spatial transcriptomics from the partially sampled observation. \par
	We assume that the observed three-dimensional spatial transcriptomics data with $n$ spots is represented by an $n$-by-4 matrix ${\bf O}\in{\mathbb R}^{n\times 4}$. Each spot ${\bf O}_{(i,:)}$ ($i=1,2,\cdots,n$) is an ($x,y,g,v$) format four-dimensional vector, where $(x,y)$ denotes the spatial coordinate, $g$ corresponds to a gene type, and $v$ is the gene expression of the gene $g$ at the spot with coordinate $(x,y)$. The continuous representation $f_\Theta:{\mathbb R}^3\rightarrow{\mathbb R}$ takes the coordinate and gene vector ($x,y,g$) as input and is expected to output the corresponding gene expression value $v$. {We then consider the following optimization model for training the continuous representation $f_\Theta(\cdot)$ based on the proposed derivative-based space-variant NeurTV regularization:}
	\begin{equation}\small\label{ST_model}
		\min_{\Theta}\sum_{i=1}^{n}({\bf O}_{{(i,4)}}-f_\Theta({\bf O}_{{(i,1:3)}}))^2+\lambda\Psi_{{\rm NeurTV}_\theta^\alpha}(\Theta),
	\end{equation}\noindent
	where $\Psi_{{\rm NeurTV}_{\theta}^\alpha}(\Theta)$ denotes the space-variant NeurTV defined in \eqref{WGV}. We impose the NeurTV constraint in the spatial dimensions $(x,y)$ and do not impose the constraint in the gene dimension $g$ since the local correlations in the gene dimension are not significant. We use the tensor factorization-based DNN \cite{TPAMI_Luo} (see Fig. \ref{fig_INR}) as the continuous representation $f_\Theta:{\mathbb R}^3\rightarrow{\mathbb R}$. The algorithm of tackling the model \eqref{ST_model} is similar to Algorithm \ref{alg} by using iterative updates of the scale and directional parameters $\alpha,\theta,a$ and the DNN weights $\Theta$. We set the trade-off parameter $\lambda$ as $2.5\times10^{-4}$. {The width and depth of the DNN are set to 300 and 3, respectively.} {We use four baselines, including the three regression methods KNR \cite{sklearn}, DT \cite{sklearn}, RF \cite{sklearn}, and a graph tensor decomposition-based method GNTD \cite{GNTD}. Here, KNR, DT, and RF are model-based methods. GNTD is a DNN-based method which uses the observed spatial transcriptomics data to train its DNN in an unsupervised manner.} We use the spatial transcriptomics data of a mouse brain dataset\footnote{\url{https://www.10xgenomics.com/datasets}} to test these methods. We use five gene types ``mbp'', ``snap25', ``atp1b1'', ``plp1'', and ``rtn1'' for testing. We consider the sampling rate 0.3 and 0.4 to generate observed data.\par
\begin{figure}[t]
	\scriptsize
	
	\setlength{\tabcolsep}{0.9pt}
	\begin{center}
		\includegraphics[width=0.7\textwidth]{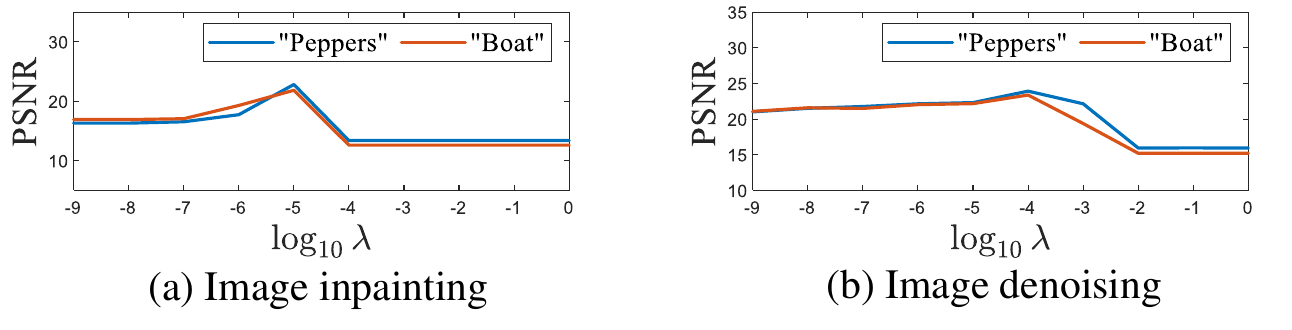}
		\vspace{-0.3cm}
	\end{center}
	\caption{The PSNR w.r.t. the value of the trade-off parameter $\lambda$ by using our space-variant NeurTV for (a) image inpainting with sampling rate 0.1 and (b) image denoising with Gaussian noise deviation 0.1.\label{fig_large}}\vspace{-0.2cm}
\end{figure}
\begin{figure}[t]
	\scriptsize
	\setlength{\tabcolsep}{0.9pt}
	\begin{center}
		\includegraphics[width=0.7\textwidth]{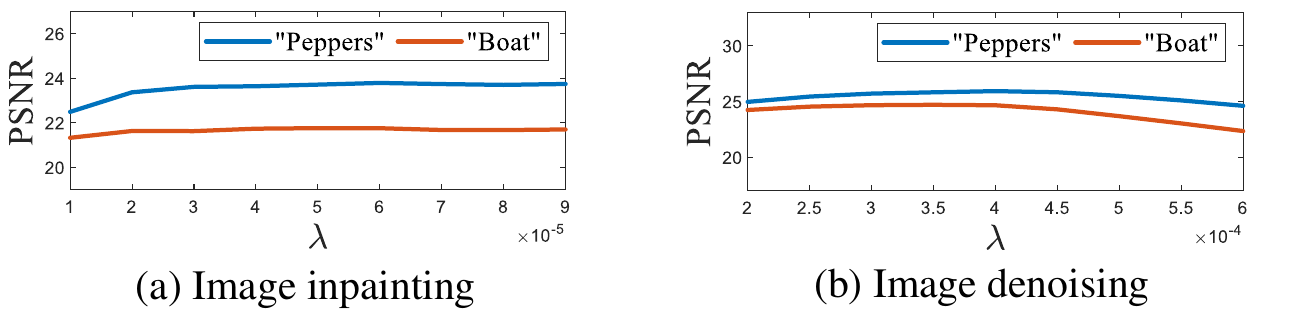}
		\vspace{-0.3cm}
	\end{center}
	\caption{The PSNR w.r.t. the value of the trade-off parameter $\lambda$ by using our space-variant NeurTV for (a) image inpainting with sampling rate 0.1 and (b) image denoising with Gaussian noise deviation 0.1.\label{fig_hyper}}\vspace{-0.2cm}
\end{figure}
\begin{figure}[t]
	\scriptsize
	
	\setlength{\tabcolsep}{0.9pt}
	\begin{center}
		\includegraphics[width=0.7\textwidth]{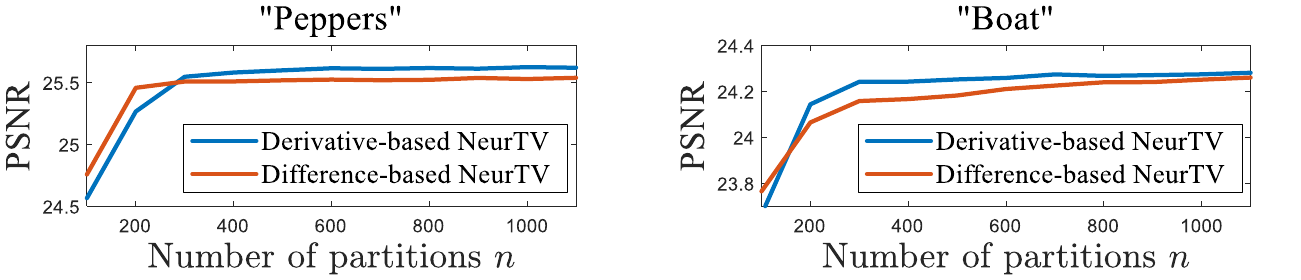}
		\vspace{-0.3cm}
	\end{center}
	\caption{The PSNR w.r.t. the number of partitions $n$ for image denoising with noise deviation 0.1 by using the derivative-based NeurTV model \eqref{model_gtv} and the difference-based NeurTV model \eqref{model_tv}.\label{fig_n}}\vspace{-0.2cm}
\end{figure}
\begin{figure}[t]
	
	\scriptsize
	\setlength{\tabcolsep}{0.9pt}
	\begin{center}
		\includegraphics[width=0.7\textwidth]{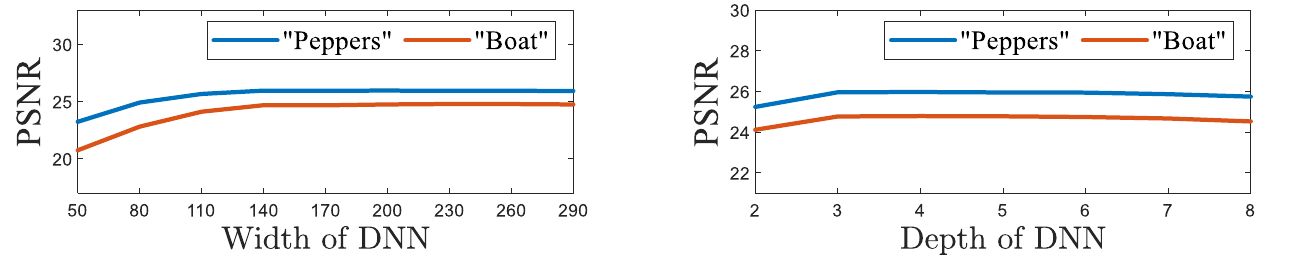}
		\vspace{-0.3cm}
	\end{center}
	\caption{The PSNR w.r.t. the width and depth of the DNN by using our space-variant NeurTV for image denoising with Gaussian noise deviation 0.1.\label{fig_width}}\vspace{-0.2cm}
\end{figure}
\begin{figure}[!h]
	\tiny
	\setlength{\tabcolsep}{0.9pt}
	\begin{center}\vspace{0.3cm}
		\begin{tabular}{cccccc}
			\includegraphics[width=0.13\textwidth]{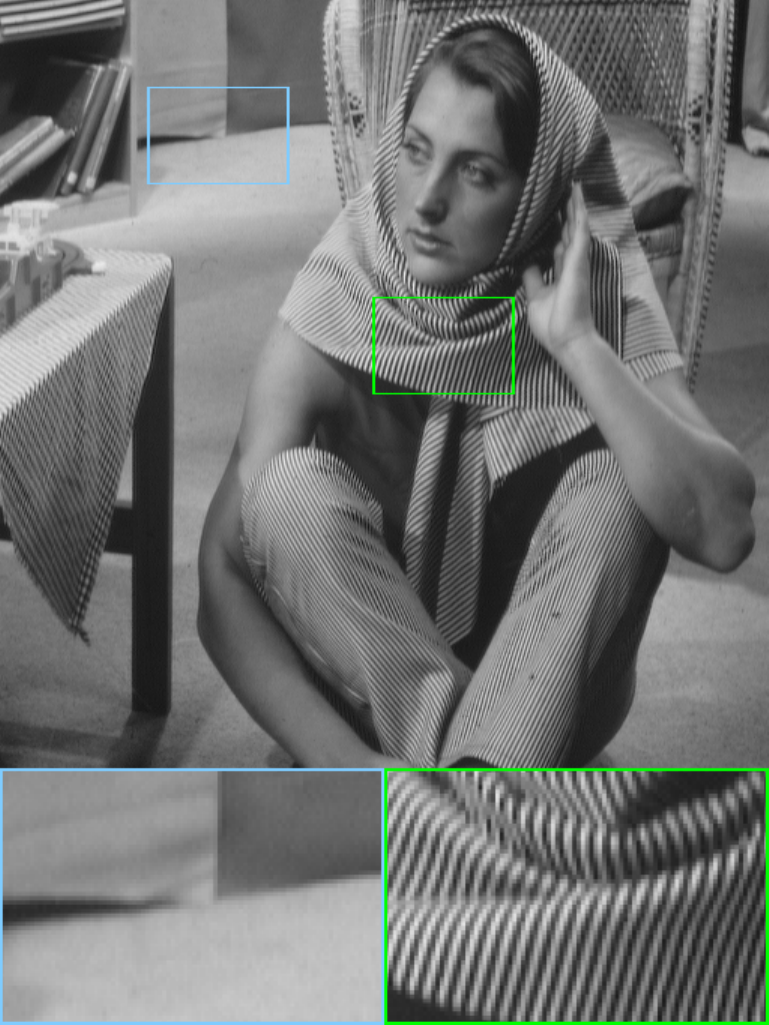}
			&\includegraphics[width=0.13\textwidth]{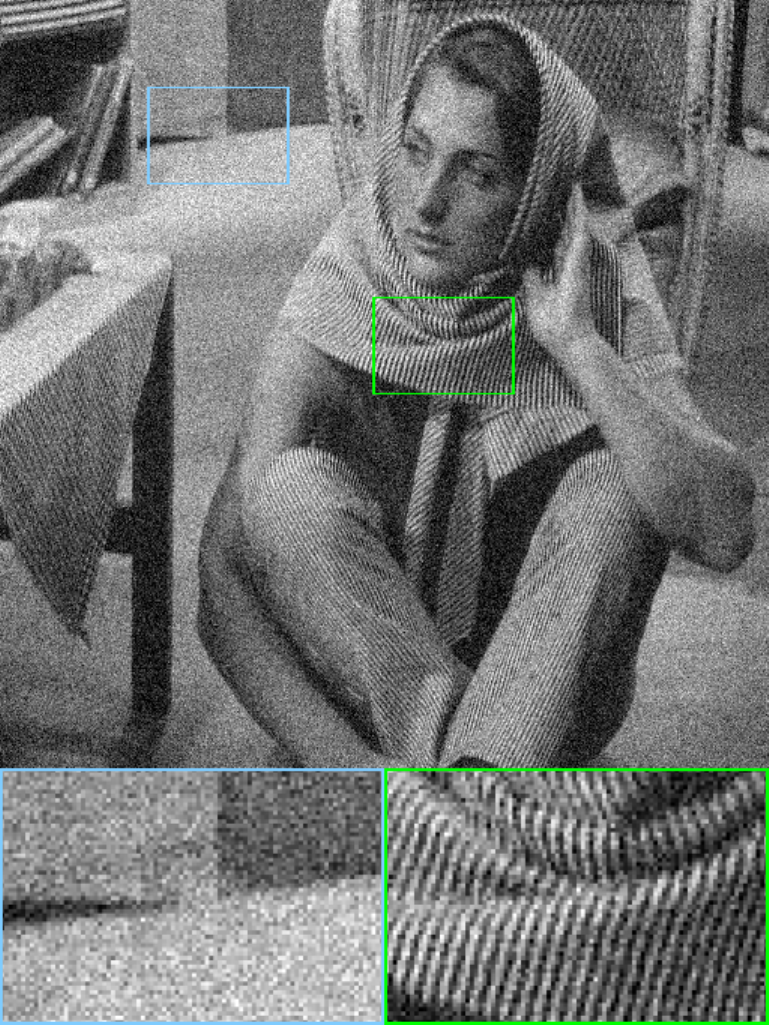}
			&\includegraphics[width=0.13\textwidth]{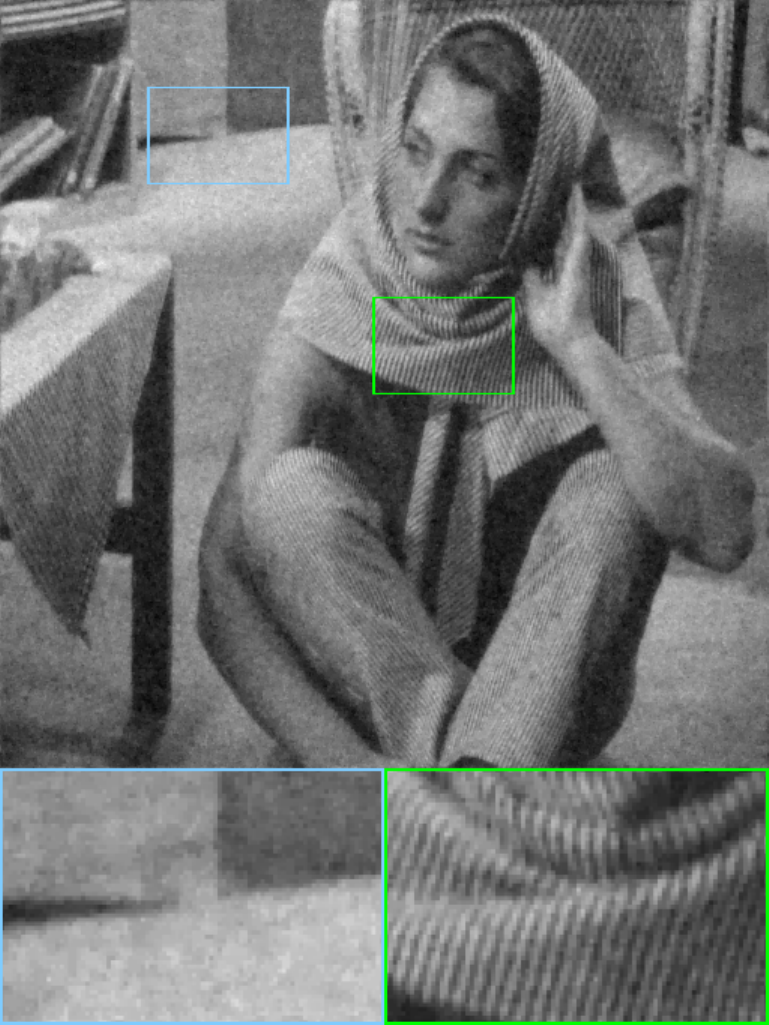}
			&\includegraphics[width=0.13\textwidth]{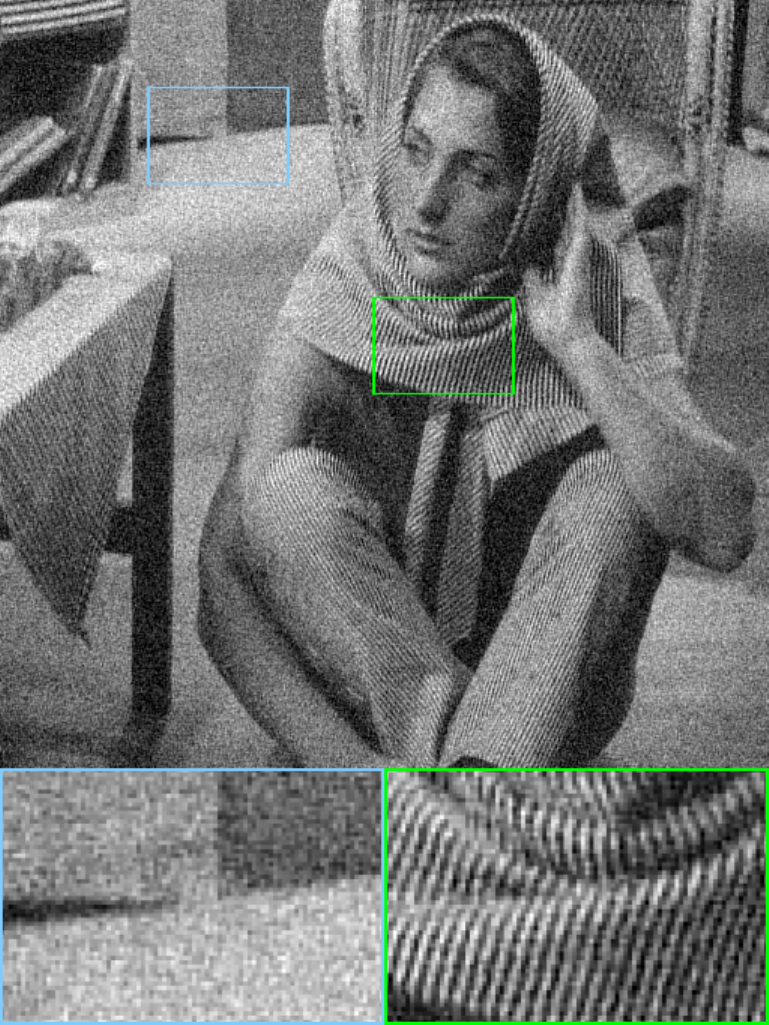}
			&\includegraphics[width=0.13\textwidth]{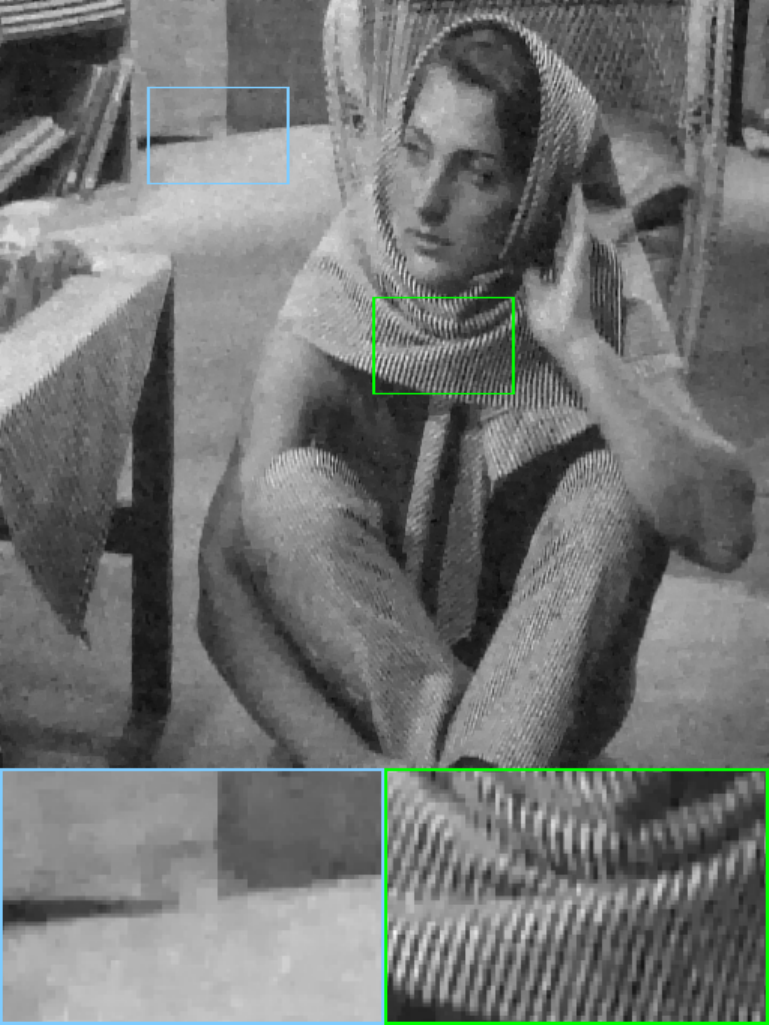}
			&\includegraphics[width=0.13\textwidth]{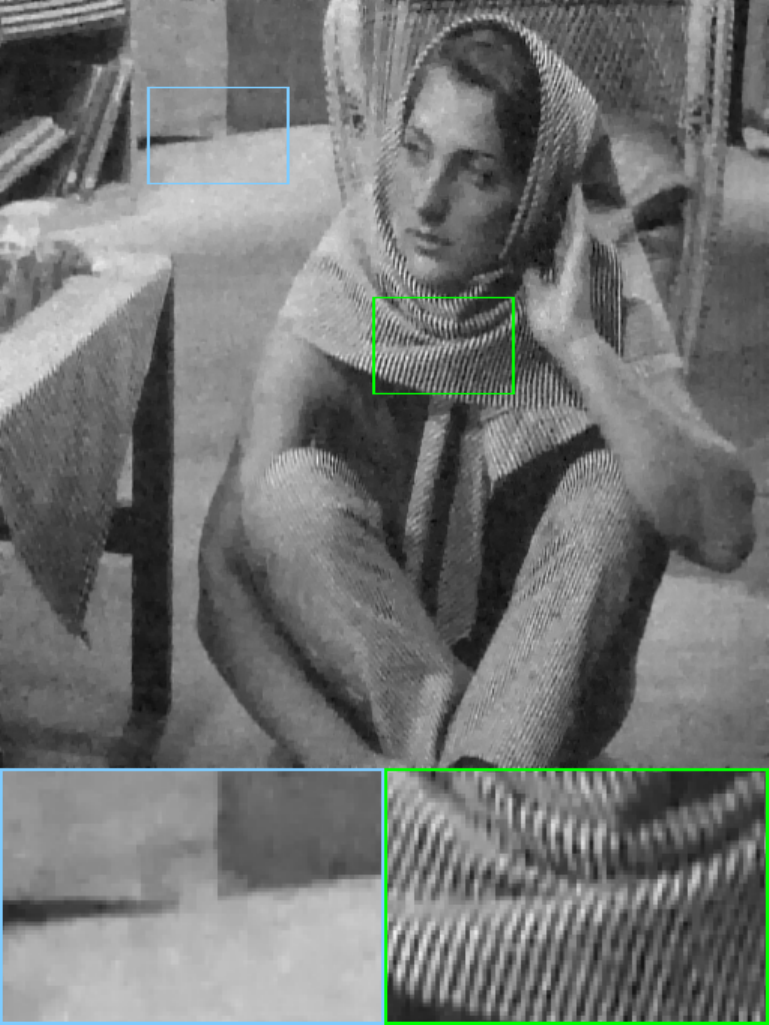}\\	
			PSNR Inf &
			PSNR 20.10 &
			PSNR 25.39 &
			PSNR 21.90 &
			PSNR 26.27  &
			PSNR 26.29   \\
			Original&Noisy&TV&Loss w/o reg.&\tabincell{c}{Difference-based\\NeurTV}&\tabincell{c}{Derivative-based\\NeurTV}\\
		\end{tabular}
		\vspace{-0.3cm}
	\end{center}
	\caption{The results of image denoising under noise deviation 0.1 by different methods (the classical model-based TV (denoted by TV), the continuous representation method without any regularization (denoted by Loss w/o reg.), the proposed difference-based NeurTV method, and the proposed derivative-based NeurTV method) on the ``Barbara'' image.  \label{fig_Barbara}\vspace{-0.4cm}}
\end{figure} 
	The results of spatial transcriptomics reconstruction are shown in Table \ref{tab_ST} and Fig. \ref{fig_ST}. It can be observed that our NeurTV quantitatively outperforms the other methods. Especially, our method outperforms GNTD \cite{GNTD}, which is a state-of-the-art method for spatial transcriptomics reconstruction. The visual results in Fig. \ref{fig_ST} show that our NeurTV can better identify local structures of the spatial transcriptomics and recover the spatial gene expressions more accurately than other comparison methods. 
	\subsection{Influence of Hyperparameters and Ablation Study}\label{sec_ablation} 
In this section, we test the influence of hyperparameters and perform the ablation study. First, we evaluate the sensitivity of NeurTV models (e.g., models \eqref{denoising_model} \& \eqref{inpainting_model}) w.r.t. the trade-off hyperparameter $\lambda$. {We first use a larger range of $\lambda$ (i.e., from $10^{-9}$ to $1$) to test the proposed method; see Fig. \ref{fig_large}. It can be observed that the trade-off parameter $\lambda$ needs to be set to a proper value for optimal performance. Based on this observation, we further set a smaller range of $\lambda$ within a suitable order of magnitude; see Fig. \ref{fig_hyper}.} It can be observed that our method could attain satisfactory PSNR values for a relatively wide range of $\lambda$ under a suitable order of magnitude, which validates the robustness of our method w.r.t. the trade-off hyperparameter. \par 
Meanwhile, the number of partitions $n$ is another important hyperparameter in our NeurTV-based models, e.g., models \eqref{model_gtv} \& \eqref{model_tv}. Here, we change the number of partitions and test our NeurTV-based models \eqref{model_gtv} \& \eqref{model_tv} for image denoising. {We have implemented normalization by dividing by the number of partitions in each test to ensure the consistency between NeurTV regularizations with different numbers of partitions. For fairness, we have separately tuned the trade-off parameter $\lambda$ for each number of partitions to achieve the optimal performance.} The results are shown in Fig. \ref{fig_n}. It is clear that larger $n$ inclines to obtain better performances for both derivative-based and difference-based NeurTV models, which coincides with the variational approximation analysis in Lemmas \ref{lemma_error_gv}-\ref{lemma_error_tv}. \par 
{Here, we additionally test the influence of width and depth of the DNN. Specifically, we have changed the number of neurons (i.e., width) and the number of layers (i.e., depth) of the tensor factorization-based DNN and tested the space-variant NeurTV for image denoising. The results are shown in Fig. \ref{fig_width}. It can be observed that our method can obtain satisfactory performances across a wide range of DNN width and depth settings.}\par 
{Finally, we perform an additional ablation study for the proposed method with and without the regularization $\Psi(\Theta)$. Especially, the benefit of the regularization term $\Psi(\Theta)$ is to encode prior information (i.e., the local correlations of data) into the model to alleviate the ill-posedness of the data recovery problem. Without the regularization term, the DNN $f_\Theta(\cdot)$ would overfit to the observed data, resulting in unsatisfactory performances for data recovery. Here, we conduct an ablation study for the proposed method with and without the NeurTV regularization on the ``Barbara'' image, which has abundant textures. The results are shown in Fig. \ref{fig_Barbara}. It can be observed that both the difference-based and derivative-based NeurTV regularizations work well for the ``Barbara'' image. As compared, the continuous representation method without regularization (i.e., Loss w/o reg. in Fig. \ref{fig_Barbara}) does not attenuate noise well, which validates the benefit of the NeurTV regularization.}
	\section{Conclusion}\label{sec_con}
We have proposed the NeurTV regularization to capture local correlations of data based on continuous representation. As compared with classical discrete meshgrid-based TV, our NeurTV is free of discretization error induced by the difference operator, and is suitable for both meshgrid and non-meshgrid data. By virtue of the continuous and differentiable nature of DNN, NeurTV can be readily extended to capture local correlations of data for any direction and any order of derivatives. We have reinterpreted NeurTV from the variational approximation perspective, which allows us to draw the connection between NeurTV and classical TV regularizations and motivates us to develop variants such as arbitrary resolution and space-variant NeurTV regularizations. Numerical experiments on different inverse imaging problems with meshgrid and non-meshgrid data validated the effectiveness and superiority of our NeurTV methods. {In this work, we have considered selecting the trade-off parameter of the NeurTV regularization by manual tuning, which is consistent with standard TV methods in the literature. It is also interesting to consider training the trade-off parameter through the DNN. We will investigate this promising idea in future work.}
	\appendix
	\section{Proof of Lemmas}
	\setcounter{theorem}{0} 
	\renewcommand{\thetheorem}{3.6}
\subsection{Proof of Lemma \ref{lemma_TV_DF}}\label{proof_lemma_TV_DF}
\begin{lemma}[Total variation of function \cite{TV}]
Given a function $f_\Theta({\bf x}):\Omega\rightarrow{\mathbb R}$ parameterized by $\Theta$, where $\Omega\subset{\mathbb R}^N$, the TV of $f_\Theta(\cdot)$ in $\Omega$ is defined as
		\begin{equation}\small
			V(f_\Theta;\Omega) := \sup\{\int_\Omega f_\Theta({\bf x})\;{\rm div}\phi({\bf x}) d{\bf x}: \phi\in C_c^\infty(\Omega,{\mathbb R}^N),\lVert\phi({\bf x})\rVert_{\infty}\leq1\},
		\end{equation}\noindent
		where $\rm div$ is the divergence operator and $\phi(\cdot)$ is a differentiable vector field with compact support. If the function $f_\Theta({\bf x}):\Omega\rightarrow{\mathbb R}$ is differentiable w.r.t. the input $\bf x$, then its functional TV \eqref{FunTV} is equal to $
			V(f_\Theta;\Omega) = \sum_{d=1}^N\int_\Omega \left|\frac{\partial f_\Theta({\bf x})}{\partial {\bf x}_{(d)}}\right| d{\bf x}.$
	\end{lemma}
	\begin{proof}
		By integration by parts, we have 
		\begin{equation}\small
			\begin{split}
				\int_\Omega f_\Theta({\bf x})\;{\rm div}\phi({\bf x}) d{\bf x} &= \sum_{d=1}^N\int_\Omega f_\Theta({\bf x})\;\frac{\partial\phi_d({\bf x})}{\partial {\bf x}_{(d)}} d{\bf x}=\sum_{d=1}^N(-\int_\Omega \frac{\partial f_\Theta({\bf x})}{\partial {\bf x}_{(d)}}\;\phi_d({\bf x})d{\bf x}),
			\end{split}
		\end{equation}\noindent
		where $\phi_d$ denotes the $d$-th component of the vector field $\phi(\cdot)$. Here, the first term of integration by parts vanishes since $\phi(\cdot)$ has compact support and hence $\phi_d({\bf x})=0$ for ${\bf x}\in\partial\Omega$. Since 
		\begin{equation}\small		
-\int_\Omega \frac{\partial f_\Theta({\bf x})}{\partial {\bf x}_{(d)}}\phi_d({\bf x}) d{\bf x}\leq\int_\Omega \left|\frac{\partial f_\Theta({\bf x})}{\partial {\bf x}_{(d)}}\phi_d({\bf x})\right| d{\bf x}\leq\int_\Omega \left|\frac{\partial f_\Theta({\bf x})}{\partial {\bf x}_{(d)}}\right|d{\bf x}
	\end{equation}\noindent
and the equality attains when $\phi_d({\bf x})=-\frac{\partial f_\Theta({\bf x})}{\partial {\bf x}_{(d)}}\big{/}|\frac{\partial f_\Theta({\bf x})}{\partial {\bf x}_{(d)}}|$, we have
		\begin{equation}\small
			\begin{split}
				V(f_\Theta;\Omega) &=\sup\{\int_\Omega f_\Theta({\bf x})\;{\rm div}\phi({\bf x}) d{\bf x}: \lVert\phi({\bf x})\rVert_{\infty}\leq1\}\\&
				=\sup\{\sum_{d=1}^N(-\int_\Omega \frac{\partial f_\Theta({\bf x})}{\partial {\bf x}_{(d)}}\;\phi_d({\bf x})d{\bf x}): \lVert\phi({\bf x})\rVert_{\infty}\leq1\}
				=\sum_{d=1}^N \int_\Omega \left|\frac{\partial f_\Theta({\bf x})}{\partial {\bf x}_{(d)}}\right|d{\bf x}.
			\end{split}
		\end{equation}\noindent
	\end{proof}
	\setcounter{theorem}{0} 
	\renewcommand{\thetheorem}{3.8}
\subsection{Proof of Lemma \ref{lemma_uniform}}\label{proof_lemma_uniform}
	\begin{lemma}
For a function $f_\Theta(x):[a,b]\rightarrow{\mathbb R}$ that is differentiable w.r.t. the input $x$, we consider the uniform partitions over $[a,b]$, i.e., $x_i = a+\frac{i(b-a)}{n}$ $(i=0,1,\cdots,n)$. Then we have $V(f_\Theta;[a,b])=\lim_{n\rightarrow \infty}\sum_{i=1}^n{|f_\Theta(x_i)-f_\Theta(x_{i-1})|}$.
	\end{lemma}
	\begin{proof}
		Given any partitions ${\Gamma}=(x_1,\cdots,x_n)$, denote $V_{\Gamma}(f_\Theta;[a,b]) := \sum_{i=1}^n|f_\Theta(x_i)-f_\Theta(x_{i-1})|$. Let $\epsilon>0$, according to Lemma \ref{lemma_P} (at the end of the appendix) there exist some partitions ${\Gamma}'=(x_0',x_1',\cdots,x_m')$ such that $V(f_\Theta;[a,b])-\frac{\epsilon}{2}<V_{{\Gamma}'}(f_\Theta;[a,b])\leq V(f_\Theta;[a,b])$. Since $f_\Theta(\cdot)$ is continuous in the interval $[a,b]$, there exists $\delta>0$
		such that $|f_\Theta(x)-f_\Theta(y)|<\epsilon/(4(m+1))$ for any $|x-y|<\delta$. Consider an uniform partition ${\Gamma}''=(x''_0,\cdots,x''_n)$ such that $\lVert{\Gamma}''\rVert:=\max_{i=1,2,\cdots,n}|x_i''-x_{i-1}''|<\delta$. Let ${\Gamma}={\Gamma}'\cup{\Gamma}''$. Then we have $\lVert{\Gamma}\rVert<\delta$ and
		\begin{equation}\small
			V_{{\Gamma}'}(f_\Theta;[a,b])<V_{{\Gamma}}(f_\Theta;[a,b]),\;V_{{\Gamma}''}(f_\Theta;[a,b])<V_{{\Gamma}}(f_\Theta;[a,b]).
		\end{equation}\noindent
		Note that adding every point $x'_i$ ($i=0,1,\cdots,m$) in ${\Gamma}'$ to the partitions ${\Gamma}''$ leads to an increase of function variation at most $2{\epsilon}/{4(m+1)}$ since $\lVert {\Gamma}'\cup {\Gamma}''\rVert=\lVert{\Gamma}\rVert<\delta$. Hence the total increase of function variation of $V_{{\Gamma}}(f_\Theta;[a,b])$ against $V_{{\Gamma}''}(f_\Theta;[a,b])$ admits $
			V_{{\Gamma}}(f_\Theta;[a,b])-\frac{\epsilon}{2}\leq V_{{\Gamma}''}(f_\Theta;[a,b])$. Hence we have 
		\begin{equation}\small
			\begin{split}
				V(f_\Theta;[a,b])-{\epsilon}\leq V_{{\Gamma}'}(f_\Theta;[a,b])-\frac{\epsilon}{2}\leq V_{{\Gamma}}(f_\Theta;[a,b])-\frac{\epsilon}{2}\leq V_{{\Gamma}''}(f_\Theta;[a,b])\leq V(f_\Theta;[a,b]).	
			\end{split}
		\end{equation}\noindent
		Let $\epsilon\rightarrow0$ we have $V(f_\Theta;[a,b])=V_{{\Gamma}''}(f_\Theta;[a,b])$ with $\delta\rightarrow0$, i.e., with $n\rightarrow\infty$. Hence $V(f_\Theta;[a,b])=\lim_{n\rightarrow\infty}V_{{\Gamma}''}(f_\Theta;[a,b])$, which completes the proof.
	\end{proof}
	\renewcommand{\thetheorem}{3.9}
\subsection{Proof of Lemma \ref{lemma_error_gv}}\label{proof_lemma_error_gv}
	\begin{lemma}
		Suppose that $f_\Theta(x):[a,b]\rightarrow{\mathbb R}$ is differentiable w.r.t. the input $x$. The truncation error $R$ of numerical integration by using uniform partitioned quadratures for approximating the functional TV (i.e., approximating $V(f_\Theta;[a,b])$) satisfies
		\begin{equation}\small
			R:=\left|\left(V(f_\Theta;[a,b])-\frac{b-a}{n}{\sum_{i=1}^n\left|\frac{df_\Theta(x_i)}{dx_i}\right|}\right)\right|\leq
			\frac{(b-a)^2}{2n}\left|\frac{df^2_\Theta(\eta)}{d\eta^2}\right|,
		\end{equation}\noindent
		where $\eta\in(a,b)$ and $\frac{df^2_\Theta(\eta)}{d\eta^2}$ denotes the second-order derivative.
	\end{lemma}
	\begin{proof}
		For any $i=1,2,\cdots,n$, we have
		\begin{equation}\small
			\begin{split}
				\int_{x_{i-1}}^{x_i}\left|\frac{d f_\Theta({x})}{d{x}}\right|dx-&\left|\frac{df_\Theta(x_i)}{dx_i}\right|(x_{i}-x_{i-1})=\int_{x_{i-1}}^{x_i}\left(\left|\frac{d f_\Theta({x})}{d{x}}\right|-\left|\frac{df_\Theta(x_i)}{dx_i}\right|\right)dx\\
				&\leq\int_{x_{i-1}}^{x_i}\left|\frac{d f_\Theta({x})}{d{x}}-\frac{df_\Theta(x_i)}{dx_i}\right|dx=\int_{x_{i-1}}^{x_i}\left|\frac{d^2 f_\Theta(\xi_i)}{d \xi_i^2}(x_i-x)\right|dx\\
				&=\int_{x_{i-1}}^{x_i}\left|\frac{d^2 f_\Theta(\xi_i)}{d \xi_i^2}\right|(x_i-x)dx=\left|\frac{d^2 f_\Theta(\eta_i)}{d \eta_i^2}\right|\int_{x_{i-1}}^{x_i}(x_i-x)dx,
			\end{split}
		\end{equation}\noindent
		where $\xi_i\in(x_{i-1},x)$ is a function of $x$ and $\eta_i\in[x_{i-1},x_{i}]$. The last equality follows from the mean value theorem for integral provided that $(x_i-x)\geq0$ for $x\in[x_{i-1},x_i]$ and $|\frac{d^2 f_\Theta(\xi_i)}{d \xi_i^2}|$ is a continuous function w.r.t. $x$. Summing the error from $i=1,2,\cdots,n$ we have
		\begin{equation}\small
			\begin{split}
				R:=\left|\sum_{i=1}^n\left(\int_{x_{i-1}}^{x_i}\left|\frac{d f_\Theta({x})}{d{x}}\right|dx-\left|\frac{df_\Theta(x_i)}{dx_i}\right|(x_{i}-x_{i-1})\right)\right|&\leq
				\sum_{i=1}^n \left|\frac{d^2 f_\Theta(\eta_i)}{d \eta_i^2}\right|\int_{x_{i-1}}^{x_i}(x_i-x)dx
				\\&=\frac{n}{2}\left|\frac{df^2_\Theta(\eta)}{d\eta^2}\right|{(x_i-x_{i-1})^2},
			\end{split}
		\end{equation}\noindent
		where $\eta\in(a,b)$ and the last equality follows from the intermediate value theorem provided that $|\frac{d^2 f_\Theta(x)}{d x^2}|$ is continuous w.r.t. $x$ in $[a,b]$. Further calculate yields $
				R\leq\frac{(b-a)^2}{2n}\left|\frac{df^2_\Theta(\eta)}{d\eta^2}\right|.$
	\end{proof}
	\renewcommand{\thetheorem}{3.10}
\subsection{Proof of Lemma \ref{lemma_error_tv}}\label{proof_lemma_error_tv}
	\begin{lemma}
		Given a function $f_\Theta(x):[a,b]\rightarrow{\mathbb R}$ that is differentiable w.r.t. the input $x$, we consider the uniform partitions ${\Gamma}_n := (x_0,x_1,\cdots,x_n)$ of $[a,b]\subset{\mathbb R}$, i.e., $x_i = a+\frac{i(b-a)}{n}$ $(i=0,1,\cdots,n)$. Then we have $
			\sum_{{\Gamma}_n}|f_\Theta(x_{i})-f_\Theta(x_{i-1})|\leq\sum_{{\Gamma}_{2n}}|f_\Theta(x_{i})-f_\Theta(x_{i-1})|\leq V(f_\Theta;[a,b]).$
	\end{lemma}
	\begin{proof}
		The first inequality follows from the triangle inequality, i.e., 
		\begin{equation}\small
			|f_\Theta(x_{i})-f_\Theta(x_{i-1})|\leq|f_\Theta((x_{i-1}+x_{i})/2)-f_\Theta(x_{i-1})|+|f_\Theta(x_{i})-f_\Theta((x_{i-1}+x_{i})/2)|,
		\end{equation}\noindent
		and the second inequality follows from Lemma \ref{lemma_P}.
	\end{proof}
	\renewcommand{\thetheorem}{3.12}
\subsection{Proof of Lemma \ref{lemma_WTV}}\label{proof_lemma_WTV}
	\begin{lemma}
Consider the numerical integration for $V(f_\Theta,[a,b])=\int_a^b|\frac{df_\Theta({x})}{dx}|d{x}$ with uniform partitions $(x_0,x_1,\cdots,x_n)$. Assume that $|\frac{d f_\Theta({x})}{d{x}}|$ is non-decreasing in $[a,b]$ and there exists $j$ such that $\frac{d^2f(t_1)}{dt_1^2}>\frac{d^2f(t_2)}{dt_2^2}\geq0$ for any $t_1\in[x_{j-1},x_j]$ and $t_2\in(x_j,x_{j+1}]$. Then the truncation error using non-uniform partitions $(x_0,x_1,\cdots,x_{j-1},x_j-\delta,x_{j+1},\cdots,x_n)$ is less than the truncation error using uniform partitions provided that $\delta>0$ is small enough.
	\end{lemma}
	\begin{proof}
		Since $|\frac{d f_\Theta({x})}{d{x}}|$ is non-decreasing we have $|\frac{d f_\Theta({x})}{d{x}}|-|\frac{df_\Theta(x_i)}{dx_i}|\leq0$ for $x\in[x_{i-1},x_i]$, and hence $\int_{x_{i-1}}^{x_i}(|\frac{d f_\Theta({x})}{d{x}}|-|\frac{df_\Theta(x_i)}{dx_i}|)dx\leq0$ for $i=1,\cdots,n$. Denote $R_1$ the truncation error of numerical integration using uniform partitions $(x_0,x_1,\cdots,x_n)$. We have
		\begin{equation}\small
			\begin{split}
				R_1=\left\lvert\sum_{i=1}^n\int_{x_{i-1}}^{x_i}\left(\left|\frac{d f_\Theta({x})}{d{x}}\right|-\left|\frac{df_\Theta(x_i)}{dx_i}\right|\right)dx\right\rvert=\sum_{i=1}^n\int_{x_{i-1}}^{x_i}\left(\left|\frac{df_\Theta(x_i)}{dx_i}\right|-\left|\frac{d f_\Theta({x})}{d{x}}\right|\right)dx.
			\end{split}
		\end{equation}\noindent
		Let $\delta\in(0,x_{j}-x_{j-1})$. Denote $R_2$ the truncation error of numerical integration using non-uniform partitions $(x_0,x_1,\cdots,x_{j-1},x_j-\delta,x_{j+1},\cdots,x_n)$ and note that $|\frac{d f_\Theta({x})}{d{x}}|$ is non-decreasing in $[a,b]$. We have
		\begin{equation}\small
			\begin{split}
				R_2=&\sum_{i\neq j,\;j+1}\int_{x_{i-1}}^{x_i}\left(\left|\frac{df_\Theta(x_i)}{dx_i}\right|-\left|\frac{d f_\Theta({x})}{d{x}}\right|\right)dx+\int_{x_{j-1}}^{x_j-\delta}\left(\left|\frac{d f_\Theta(x_j-\delta)}{d (x_j-\delta)}\right|-\left|\frac{d f_\Theta({x})}{d{x}}\right|\right)dx\\&+\int_{x_{j}-\delta}^{x_j}\left(\left|\frac{d f_\Theta(x_{j+1})}{d (x_{j+1})}\right|-\left|\frac{d f_\Theta({x})}{d{x}}\right|\right)dx+\int_{x_{j}}^{x_{j+1}}\left(\left|\frac{d f_\Theta(x_{j+1})}{d (x_{j+1})}\right|-\left|\frac{d f_\Theta({x})}{d{x}}\right|\right)dx.
			\end{split}
		\end{equation}\noindent
		Hence we have
		\begin{equation}\small
			\begin{split}
				R_1-R_2=& \int_{x_{j-1}}^{x_j-\delta}\left(\left|\frac{d f_\Theta(x_{j})}{d (x_{j})}\right|-\left|\frac{d f_\Theta({x})}{d{x}}\right|\right)-\left(\left|\frac{d f_\Theta(x_j-\delta)}{d (x_j-\delta)}\right|-\left|\frac{d f_\Theta({x})}{d{x}}\right|\right)dx\\
				&+\int_{x_{j}-\delta}^{x_j}\left(\left|\frac{d f_\Theta(x_{j})}{d (x_{j})}\right|-\left|\frac{d f_\Theta({x})}{d{x}}\right|\right)-\left(\left|\frac{d f_\Theta(x_{j+1})}{d (x_{j+1})}\right|-\left|\frac{d f_\Theta({x})}{d{x}}\right|\right)dx\\
				=&\int_{x_{j-1}}^{x_j-\delta}\left(\left|\frac{d f_\Theta(x_{j})}{d (x_{j})}\right|-\left|\frac{d f_\Theta(x_j-\delta)}{d (x_j-\delta)}\right|\right)dx+\int_{x_{j}-\delta}^{x_j}\left(\left|\frac{d f_\Theta(x_{j})}{d (x_{j})}\right|-\left|\frac{d f_\Theta(x_{j+1})}{d (x_{j+1})}\right|\right)dx.
			\end{split}
		\end{equation}\noindent
		Since $|\frac{d f_\Theta({x})}{d{x}}|$ is non-decreasing and $\frac{d^2 f_\Theta(x)}{dx^2}>0$ for $x\in[x_{j-1},x_{j}]$, we have $\frac{d f_\Theta({x})}{d{x}}\geq0$ for $x\in[x_{j-1},x_{j}]$. Hence we have
		\begin{equation}\small
			\begin{split}
				&R_1-R_2=\int_{x_{j-1}}^{x_j-\delta}\left(\left|\frac{d f_\Theta(x_{j})}{d (x_{j})}\right|-\left|\frac{d f_\Theta(x_j-\delta)}{d (x_j-\delta)}\right|\right)dx+\int_{x_{j}-\delta}^{x_j}\left(\left|\frac{d f_\Theta(x_{j})}{d (x_{j})}\right|-\left|\frac{d f_\Theta(x_{j+1})}{d (x_{j+1})}\right|\right)dx\\
				&=\left(\frac{d f_\Theta(x_{j})}{d (x_{j})}-\frac{d f_\Theta(x_j-\delta)}{d (x_j-\delta)}\right)(h-\delta)+\left(\frac{d f_\Theta(x_{j})}{d (x_{j})}-\frac{d f_\Theta(x_{j+1})}{d (x_{j+1})}\right)\delta
				=\frac{d^2 f_\Theta(\eta_1)}{d\eta_1^2}\delta(h-\delta)-\frac{d^2 f_\Theta(\eta_2)}{d\eta_2^2}h\delta,
			\end{split}
		\end{equation}\noindent
		where $h=x_{j+1}-x_{j}$, $\eta_1\in(x_j-\delta,x_j)$, and $\eta_2\in(x_j,x_{j+1})$. Note that $\frac{d^2 f_\Theta(t_1)}{d t_1^2}>\frac{d^2 f_\Theta(t_2)}{d t_2^2}\geq0$ for any $t_1\in[x_{j-1},x_j]$ and $t_2\in(x_j,x_{j+1}]$. Hence when $\delta< {h\left(\frac{d^2 f_\Theta(\eta_1)}{d\eta_1^2}-\frac{d^2 f_\Theta(\eta_2)}{d\eta_2^2}\right)}\Big{/}\left({\frac{d^2 f_\Theta(\eta_1)}{d\eta_1^2}}\right)$ we have $R_1-R_2>0$, i.e., the truncation error of numerical integration using non-uniform partitions is less than the truncation error using uniform partitions. 
	\end{proof}
	\setcounter{theorem}{0} 
\renewcommand{\thetheorem}{3.13}
\subsection{Proof of Lemma \ref{lemma_P}}\label{proof_lemma_eq}
\begin{lemma}\label{lemma_P}
	For a function $f_\Theta(x):[a,b]\rightarrow{\mathbb R}$ that is differentiable w.r.t. the input $x$, its functional TV \eqref{FunTV} is equal to
	\begin{equation}\small
		V(f_\Theta;[a,b]) = \int_a^b \left|\frac{df_\Theta({x})}{dx}\right| d{x}=\sup_{\Gamma}\{\sum_{i=1}^n|f_\Theta(x_i)-f_\Theta(x_{i-1})|:n\in{\mathbb N},{\Gamma}=(x_0,x_1,\cdots,x_n)\},
	\end{equation}\noindent
	where $\Gamma$ are some partitions over $[a,b]$ with $x_0=a$ and $x_n=b$.
\end{lemma}
\begin{proof}
	The first equality follows from Lemma \ref{lemma_TV_DF} and we show the second equality as follows. For any partitions ${\Gamma}=(x_0,x_1,\cdots,x_n)$ we have
	\begin{equation}\small
		\begin{split}
			\sum_{i=1}^n|f_\Theta(x_i)-f_\Theta(x_{i-1})| = \sum_{i=1}^n\left|\int_{x_{i-1}}^{x_i}\frac{d f_\Theta({x})}{d{x}}dx\right|\leq
			\sum_{i=1}^n\int_{x_{i-1}}^{x_i}\left|\frac{d f_\Theta({x})}{d{x}}\right|dx
			=\int_a^b \left|\frac{df_\Theta({x})}{dx}\right| d{x}.
		\end{split}
	\end{equation}\noindent
	Hence we have $\sup_{\Gamma}\{\sum_{i=1}^n|f_\Theta(x_i)-f_\Theta(x_{i-1})|\}\leq\int_a^b |\frac{df_\Theta({x})}{dx}| d{x}$. On the other hand, using the intermediate value theorem we have $\sum_{i=1}^n|f_\Theta(x_i)-f_\Theta(x_{i-1})| = \sum_{i=1}^n \left|\frac{\partial f_\Theta(\xi_i)}{\partial \xi_i}\right|(x_{i}-x_{i-1})$, where $\xi_i\in(x_{i-1},x_i)$. Let $\eta_i:=\inf_{x\in(x_{i-1},x_i)}|\frac{d f_\Theta({x})}{d{x}}|$, then we have
	\begin{equation}\small
		\sum_{i=1}^n|f_\Theta(x_i)-f_\Theta(x_{i-1})| = \sum_{i=1}^n \left|\frac{\partial f_\Theta(\xi_i)}{\partial \xi_i}\right|(x_{i}-x_{i-1})\geq \sum_{i=1}^n \eta_i(x_{i}-x_{i-1}).
	\end{equation}\noindent
By taking supremum we have $\sup_{\Gamma}\{\sum_{i=1}^n|f_\Theta(x_i)-f_\Theta(x_{i-1})|\}\geq \sup_{\Gamma}\{\sum_{i=1}^n \eta_i(x_{i}-x_{i-1})\}=\int_a^b \left|\frac{df_\Theta({x})}{dx}\right| d{x}$, where the last equality follows from the definition of integral.
\end{proof}
	\bibliographystyle{siamplain}
\bibliography{references}
\end{document}